\documentclass[10pt,a4paper]{article}

\usepackage[utf8]{inputenc}
\usepackage[T1]{fontenc}
\usepackage{amsmath,amssymb,amsfonts,amsthm,bm}
\usepackage{graphicx}

\newtheorem{theorem}{Theorem}[section]

\newtheorem{lemma}[theorem]{Lemma}

\theoremstyle{remark}
\newtheorem{remark}[theorem]{Remark}
\usepackage{booktabs}
\usepackage{hyperref}
\usepackage[margin=1in]{geometry}
\usepackage{xcolor}
\usepackage{enumitem}
\usepackage{natbib}
\usepackage[capitalise,noabbrev]{cleveref}
\usepackage{multirow}
\usepackage[ruled,lined,noend]{algorithm2e}
\usepackage{soul}
\usepackage{wrapfig}
\usepackage{subcaption}
\usepackage{placeins}
\usepackage{morefloats}
\usepackage{alphalph}
\title{\Large\bfseries Priming: Hybrid State Space Models From Pre-trained Transformers}

\author{
Aditya Chattopadhyay\textsuperscript{*} \and 
Elvis Nunez\textsuperscript{*} \and 
Prannay Kaul\textsuperscript{*} \and 
Benjamin Bowman\textsuperscript{*} \and 
Evan Becker\textsuperscript{*} \and 
Luca Zancato\textsuperscript{*,\textdagger} \and 
David Thomas \and 
Wei Xia \and 
Stefano Soatto}

\footnotetext[1]{Equal contribution}
\footnotetext[2]{Corresponding author: \texttt{zancato@amazon.com}}

\date{
AWS Agentic AI\\
}

\begin{document}
\maketitle

\begin{abstract}

Hybrid State-Space models combine Attention with recurrent State-Space Model (SSM) layers, balancing eidetic memory from Attention with compressed fading
memory from SSMs. This yields smaller Key-Value caches and faster decoding than Transformers, along with a richer architectural design space. Exploring that design space at scale has so far required training each configuration from scratch, a barrier that has kept most large-model Hybrid research within a narrow range of architectures.  We introduce \emph{Priming}, a method that initializes a Hybrid State-Space model from a pre-trained Transformer and, through short alignment and post-training phases, recovers downstream quality using less than $0.5\%$ of the source model's pre-training token budget. 
Priming starts by selecting which Attention layers to replace with SSMs and transferring the source weights accordingly. A short distillation phase then aligns each new SSM layer to the Attention layer it replaces, followed by end-to-end adaptation with next-token prediction loss. Priming is agnostic to the source Transformer family (e.g., Qwen, Llama, Mistral), model class (dense or Mixture-of-Experts), and model scale.
Priming turns Hybrid architecture design from a pre-training problem into a knowledge transfer problem, enabling us to run the first controlled comparison of SSM layer types at scale under identical conditions. We evaluate three SSM families: Gated KalmaNet (GKA), Gated DeltaNet (GDN), and Mamba-2. We show that the theoretical expressiveness hierarchy among them (GKA~$>$~GDN~$>$~Mamba-2) translates directly into downstream performance: GKA outperforms GDN, which outperforms Mamba-2 on long-context reasoning tasks. SSM-based Hybrid models further surpass Hybrid models that use Sliding Window Attention in place of SSMs, even when the sliding window is sized such that its KV cache is up to $4\times$ larger than the SSM state, demonstrating that a compressed fading memory is more useful than a larger window of verbatim tokens when complementing Attention.
We scale these findings to 8B and 32B reasoning models with native 128K-token contexts. Our GKA-Primed-HQwen3-32B-Reasoner achieves a $+3.8$ average percentage-point gain across reasoning benchmarks over its source Qwen3-32B model, while remaining within $1\%$ of the same Transformer post-trained with our recipe. Because SSM layers maintain only a fixed-size state instead of a growing Key-Value cache, the primed model's inference-time memory footprint is roughly halved, enabling ${\sim}2{\times}$ concurrent sequences on the same hardware and up to $2.3{\times}$ decode throughput. GKA additionally offers a runtime compute-quality knob: the iteration count of its internal solver can be adjusted at serving time without retraining, letting practitioners trade latency for accuracy.  To foster further research on Hybrid architectures, we release a \textit{model zoo of primed Hybrid models}\footnote{Training and inference code: \url{https://github.com/awslabs/hybrid-model-factory}} for long-context reasoning and instruction following, together with the Priming training and inference code (Sequence Parallelism algorithms for long-context training, optimized GKA kernels, and vLLM serving plugin), all under Apache~2.0 License.\footnote{Model zoo: \url{https://huggingface.co/collections/amazon/primed-hybrid-models-collection}}

{\small ~\\
\noindent  {\bf Keywords:} Generative AI, Large Language Models, Long Context Reasoning, Test-Time Scaling; Knowledge Distillation, Cross-Architecture Transfer; Dynamical Systems, Realization Theory, Recurrent Neural Networks, Hybrid State Space Models.}

\begin{figure*}[t]
  \centering
  \begin{minipage}[c]{0.63\textwidth}
    \centering
    \includegraphics[width=\linewidth, height=0.92\textheight, keepaspectratio]{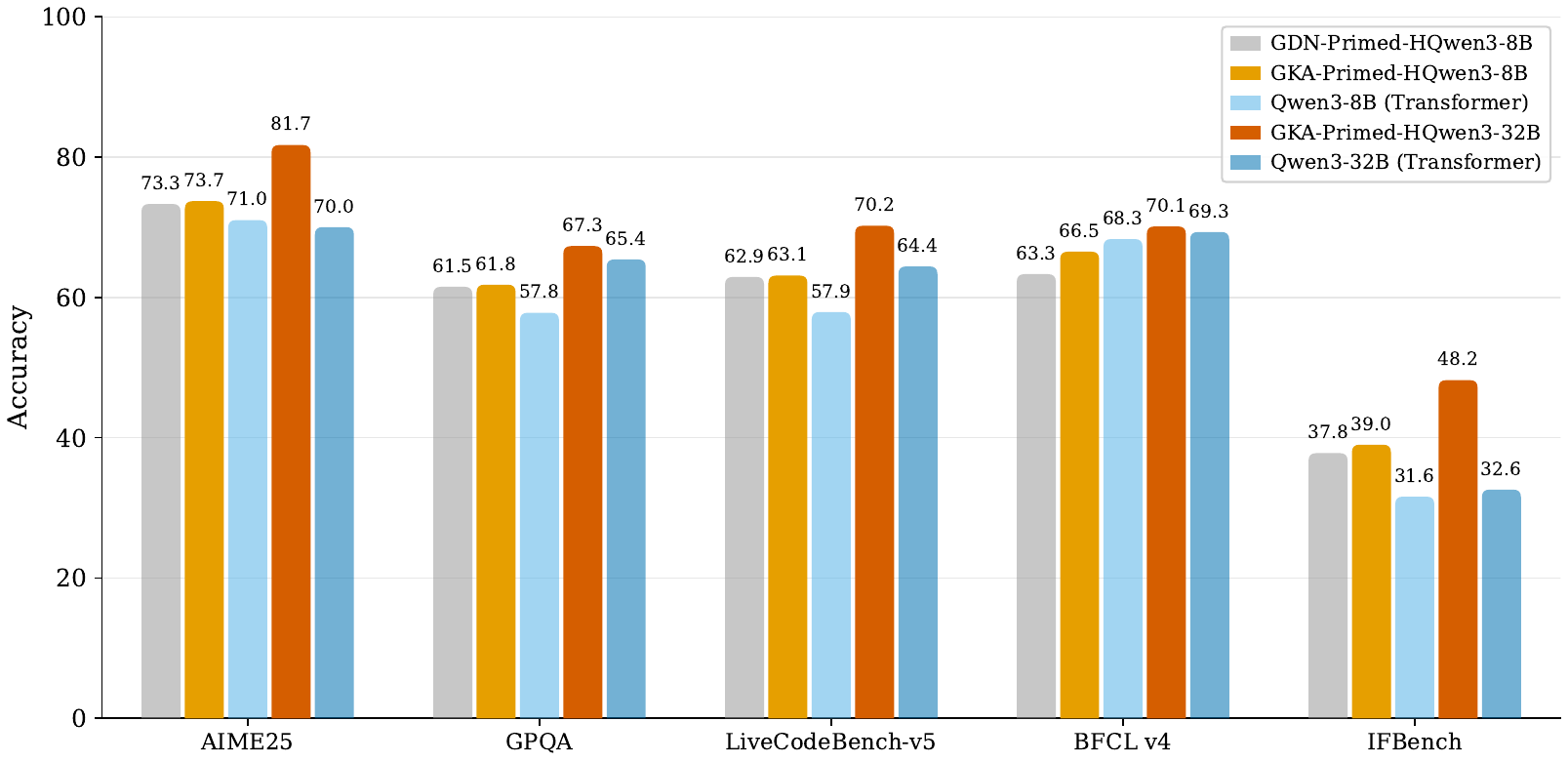}
  \end{minipage}%
  \hfill
  \begin{minipage}[c]{0.32\textwidth}
    \centering
    \includegraphics[width=\linewidth]{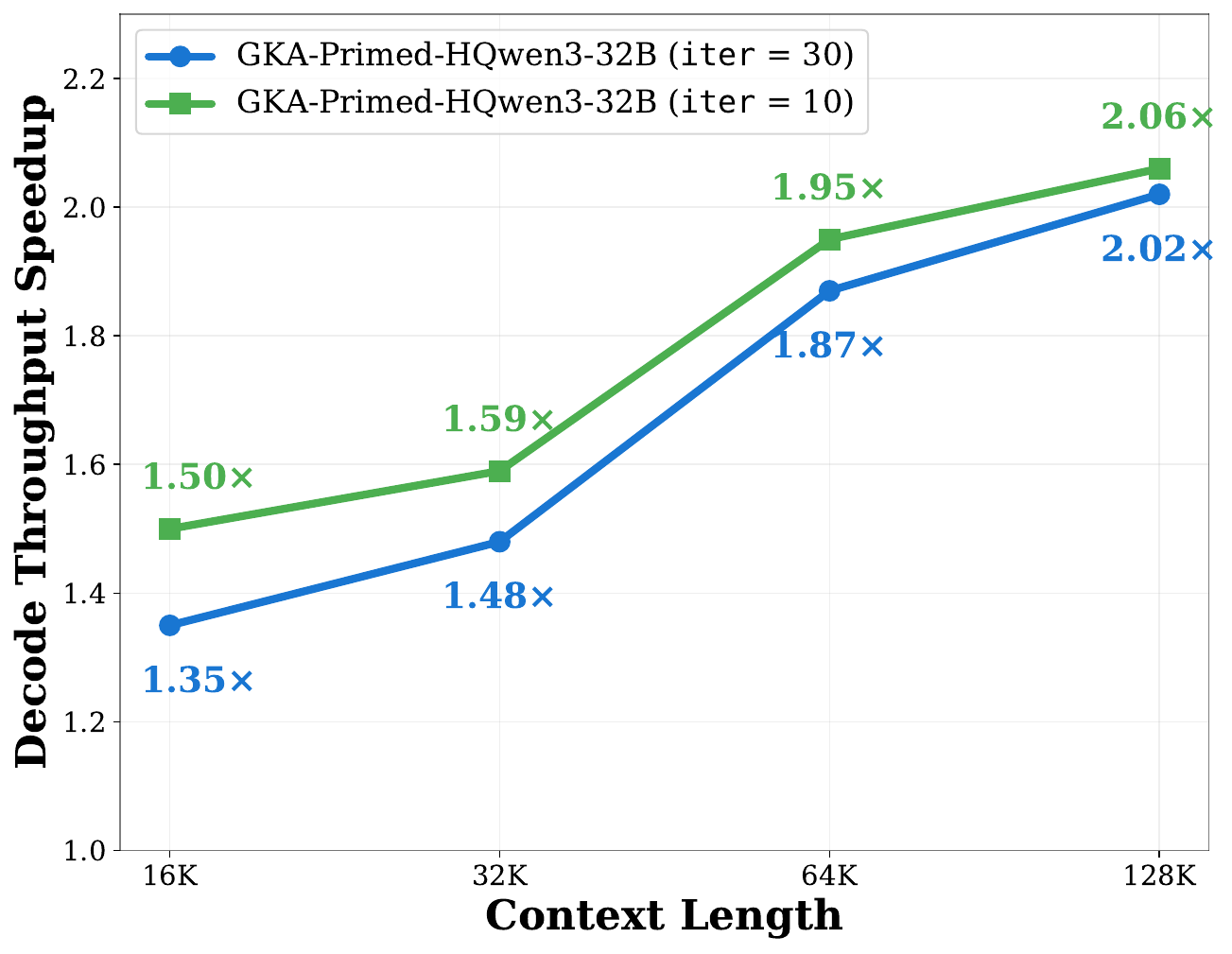}

    \vspace{0.5em}

    \includegraphics[width=\linewidth]{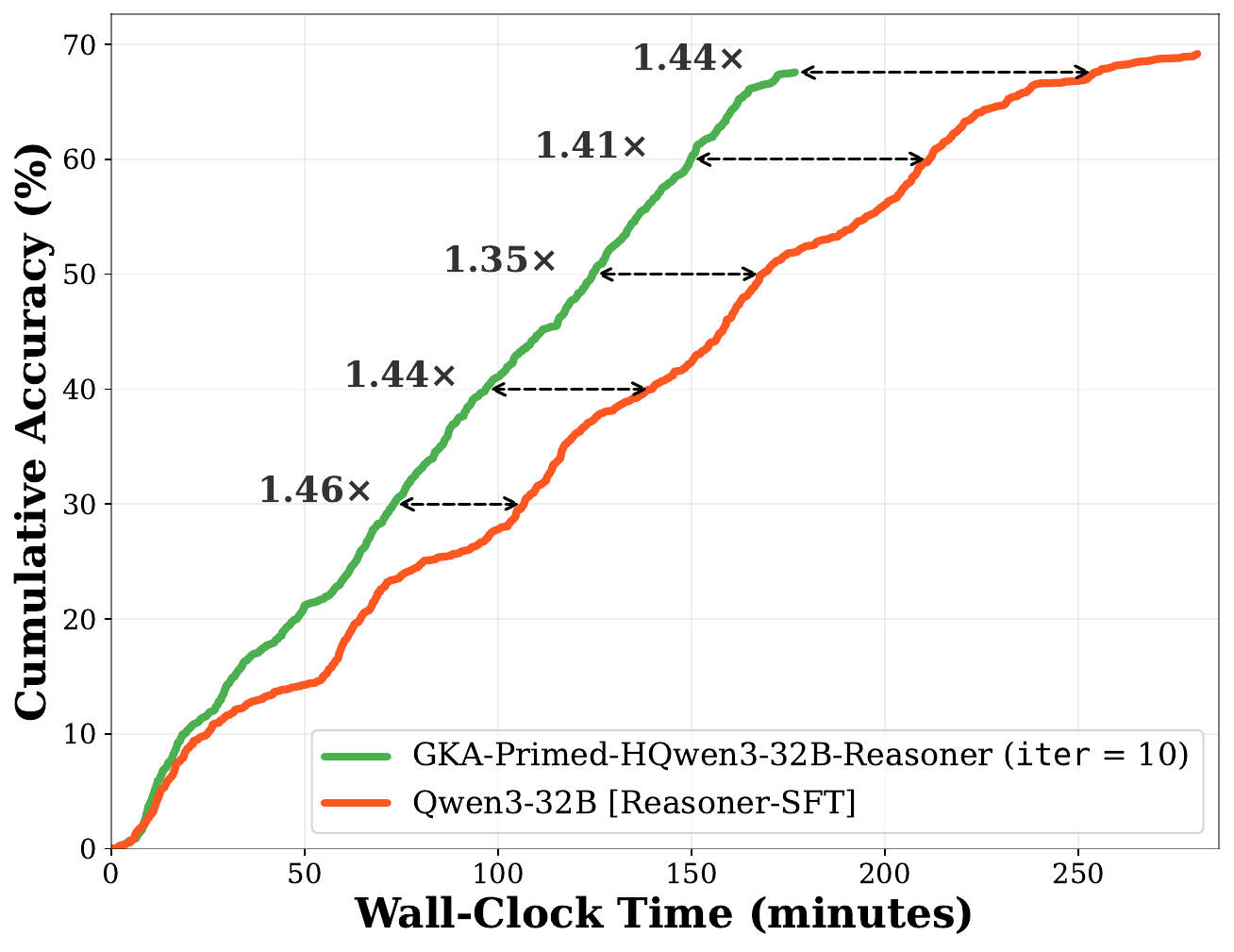}
  \end{minipage}

    \caption{\textbf{Primed Hybrid models match Transformer quality at half the memory and decode faster.}
      \textbf{(Left)}~Reasoning accuracy of Primed Hybrid reasoning models vs.~the source Qwen3 Transformer across five benchmarks at 8B and 32B scale. At 8B, the GKA Hybrid uniformly outperforms its GDN counterpart, consistent with the expressiveness hierarchy GKA~$>$~GDN (\cref{sec:ssm-comparison}).
      \textbf{(Top right)}~Decode throughput speedup of GKA-Primed-HQwen3-32B over the source Qwen3-32B Transformer as a function of context length. Replacing half the Attention layers with fixed-size SSMs roughly halves the KV cache, yielding up to $2.06{\times}$ throughput at 128K. Lowering the GKA solver iterations from 30 to 10 at serving time (without retraining) trades a small accuracy margin for additional throughput.
      \textbf{(Bottom right)}~Test-time scaling on the AIME~2025 hard subset: cumulative accuracy vs.\ wall-clock time. Cumulative accuracy at time $t$ is the fraction of rollouts for which a correct final answer has been produced by time $t$. 
      GKA-Primed-HQwen3-32B-Reasoner (10 iters) reaches each accuracy threshold faster than the source Transformer trained on the same Priming data (Qwen3-32B [Reasoner-SFT]) while converging to similar final accuracy. GKA-Primed-HQwen3-32B-Reasoner translates its per-token throughput advantage into faster time-to-answer and higher accuracy for given inference budget.
      All models are evaluated at 64K generation length with thinking mode enabled.}
  \label{fig:splash}
\end{figure*}
\FloatBarrier

\end{abstract}

\section{Introduction}
\label{sec:intro}

Hybrid State-Space models~\citep{bmojo, nemotronsuper, qwen3codernext2026} combine Attention layers with recurrent State-Space Model (SSM) layers. While Attention stores every past token verbatim in a Key-Value (KV) cache that grows linearly with context length, each SSM layer maintains a fixed-dimensional latent state that contains a compressed summary of the entire past. The resulting hybrid architecture balances two complementary forms of memory: \emph{eidetic} memory from Attention, which retains tokens exactly at quadratic cost, and \emph{fading} memory from SSMs, which compresses history into a fixed-size state at linear cost. This trade-off yields a smaller KV cache, faster decoding, and a richer design space than pure Transformers. 

Yet exploring this design space is prohibitively expensive. Pre-training a single large model consumes trillions of tokens, thousands of GPU-hours, and substantial energy costs. Because most large-scale checkpoints today are Transformers, the knowledge they encode is, in practice, locked to their architecture. If one wishes to study how the eidetic-to-fading memory ratio affects quality, or which SSM variant works best, the conventional path is to train each configuration from scratch, a barrier that has kept the Hybrid design space largely unexplored at scale.

We introduce \emph{Priming}, a method that circumvents this training cost: it initializes a Hybrid State-Space model from a pre-trained Transformer’s weights and adapts it through two short training phases using less than $0.5\%$ of the source model’s pre-training token budget. Priming is possible because a Transformer is itself a special case of a state-space realization~\citep{bmojo, ssd}: Attention and SSM layers share a common algebraic structure via the matrix mixer abstraction. This implies that a Transformer's weights already encode directions along which the model reads, writes, and routes information that are meaningful for SSM layers as well. Given a target Hybrid ratio (the ratio of SSM layers to Attention layers in the Hybrid model), Priming first performs \emph{layer selection} to identify the Attention layers in the source Transformer that are most amenable to SSM realization and marks the rest to be retained as Attention in the final Hybrid model. The source weights are then transferred into the corresponding layers. The structural correspondence between Attention and SSM projections seeds the SSM layers, while embeddings, feedforward networks, layer norms, lm-head, and the retained Attention layers are inherited directly. Priming then proceeds through two additional stages: \emph{Stage 1}, an alignment phase that trains the SSM layers to match the behavior of the Attention layers they replace, and \emph{Stage 2}, an end-to-end post-training phase with next-token prediction loss. Priming is agnostic to the Transformer family (Qwen, Llama, Mistral), model class (dense or Mixture-of-Experts), and model size; in this work we demonstrate it on Qwen3 models from 8B to 32B parameters.

Priming turns Hybrid architecture design from a pre-training problem into a
knowledge transfer problem, enabling cheap, systematic exploration of the design space: SSM layer type, Hybrid ratio, and layer assignment.\footnote{Given a target Hybrid ratio, `layer assignment' refers to the process of deciding which specific layers of the Hybrid model are realized as SSMs and which are kept as Attention.} We exploit this to conduct, to the best of our knowledge, the first controlled comparison at scale of SSM layer types under identical conditions: same source model, same data, same hyper-parameters, and same Hybrid ratio. We evaluate three SSM families: Gated KalmaNet (GKA)~\citep{gatedkalmanet}, Gated DeltaNet (GDN)~\citep{yang2024gated}, and Mamba-2~\citep{ssd}. We also consider the B'MOJO-F layer~\citep{bmojo}, which fuses eidetic and fading memory within a single sublayer. As can be demonstrated theoretically, GKA admits a more expressive state update than GDN, which in turn is more expressive than Mamba-2's scalar decay. We show that this hierarchy carries over to downstream accuracy: GKA-based Hybrid models are our strongest Primed models, outperforming their GDN-based counterparts on long-context reasoning tasks, with Mamba-2-based Hybrid models trailing behind. Furthermore, all SSM-based Hybrid models outperform a comparable Hybrid that uses Sliding Window Attention (SWA) in place of SSMs by at least 5.5\% (relative) on long-context benchmarks, even when the SWA window is sized to give a 2-4$\times$ larger KV cache footprint than the SSM state. Evidence that fading memory, which compresses the full history, is more useful than a sliding window of verbatim tokens~\citep{bmojo} when complementing Attention layers.

Among the SSM families we study, GKA occupies a unique position. Whereas Mamba-2 and GDN perform a fixed computation per layer, GKA solves an online ridge-regression objective via an iterative Chebyshev solver whose iteration count can be adjusted at serving time without retraining. This exposes a runtime quality-compute knob: reducing iterations lowers accuracy but improves latency, enabling practitioners to select an operating point along the accuracy-cost tradeoff curve.

We scale these findings to production-grade models. Through post-training and context extension, we build long-context reasoning Hybrid models at 8B and 32B scales. Our most performant Hybrid model supports 128K context tokens and achieves a $+3.8$ average percentage-point gain on reasoning benchmarks over its source Qwen3-32B, while remaining within ${\sim}1\%$ of the same source Qwen3-32B Transformer post-trained with our recipe. Because SSM layers require no KV cache, our model's memory footprint is roughly halved, enabling ${\sim}2{\times}$ the concurrent sequences on the same hardware and up to $2.3{\times}$ decode throughput at 128K context.
This throughput advantage is particularly consequential for inference and reinforcement learning post-training, where the training throughput is gated by how quickly the policy generates correct trajectories: on AIME~2025, GKA-Primed-HQwen3-32B-Reasoner produces correct rollouts up to $1.6{\times}$ faster than the equivalent Transformer across a range of concurrency settings.

Achieving these results at 128K context and 32B scale requires new infrastructure. In addition to optimized implementations and kernels for Priming, we develop a Point-to-Point (P2P) Sequence Parallelism algorithm that enables efficient long-context training of SSM layers. P2P shards the activations across GPUs with sequence-length-independent communication, making it up to $2.3\times$ faster for multi-node training than the AllToAll-based SP approach for SSMs used in large-scale training frameworks like Megatron-LM~\citep{shoeybi2019megatron}.

\subsection*{Contributions}
\begin{enumerate}[nosep]
  \item \textbf{Priming}: a general method for building Hybrid State-Space models by initializing from pre-trained Transformers at less than $0.5\%$ of the pre-training token budget. Priming is applicable across model families, scales, and SSM layer types (\cref{sec:why-priming}).

  \item \textbf{Controlled SSM comparison}: head-to-head evaluation of SSM variants under identical Priming conditions, establishing the empirical hierarchy GKA~$>$~GDN~$>$~Mamba-2 and demonstrating that fading-memory layers consistently outperform windowed eidetic memory (SWA layers) in Hybrid settings (\cref{sec:ssm-comparison}).

  \item \textbf{Hybrid reasoning models at scale}: GKA-primed models at 8B and 32B that match or exceed their source Transformers on reasoning benchmarks while offering up to $2.3{\times}$ higher decode throughput. We also introduce the first training-free context extension method for Hybrid architectures, extending the native context of our Primed models from 128K to 256K tokens (\cref{sec:experiments}, \cref{subsec:tts}, \cref{subsec:state-composition}).

  \item \textbf{Training infrastructure}: two Sequence Parallelism algorithms for SSM layers (Point-to-Point and Universal SP), a fused-architecture implementation of Priming that reduces training memory overhead (\cref{sec:scalable-impl}), and optimized GKA kernels (\cref{sec:gka_main_ideas}).

  \item \textbf{Open release}: We open-source the optimized Priming codebase, a vLLM serving plugin for Hybrid architectures, and a model zoo of Primed Hybrid models for long-context reasoning and instruction following, all under the Apache 2.0 License.
  
\end{enumerate}

\tableofcontents

\section{Hybrid Architectures}
\label{sec:hybrid-arch}

\subsection{Preliminaries: Attention and State Space Models}
\label{subsec:prelim}

An auto-regressive language model is a stack of \emph{decoder layers}, each composed of a \emph{sequence-mixing} sublayer followed by a feed-forward network (also referred to as a Multilayer Perceptron or MLP).
The sequence-mixing sublayer is the component that aggregates information across token positions, the two families of sequence-mixing layers relevant to this work, Attention and State Space Models, differ fundamentally in \emph{how} they store and access past context.

\paragraph{Attention.}
Given an input sequence of $T$ token representations $\mathbf{X} = \{\textbf{x}_1, \textbf{x}_2, \dots, \textbf{x}_T\}$, a single causal Attention head computes query, key, and value vectors $\bm{q}_t, \bm{k}_t \in \mathbb{R}^{d_k}$ and $\bm{v}_t \in \mathbb{R}^{d_v}$ at each position $t$ via learned non-linear functions of $\mathbf{x}_t$\footnote{For example, Qwen3 employs linear projection followed by unit normalization of the query and key vectors.}.  The output at position $t$ is then
\begin{equation}
  \bm{y}_t
  = \sum_{i=1}^{t}
    \frac{\exp(\bm{q}_t^\top \bm{k}_i)}
         {\sum_{j=1}^{t}\exp(\bm{q}_t^\top \bm{k}_j)}\;
    \bm{v}_i\,,
  \label{eq:attention}
\end{equation}
a softmax-weighted retrieval over all past key--value pairs.
At inference time, the pairs $(\bm{k}_i, \bm{v}_i)$ for $i = 1, \dots, t$ are stored verbatim in a \emph{KV cache} that grows
linearly with context length. This cache is an \emph{eidetic memory}: every token representation is retained exactly, providing precise access to the full context at a memory and compute cost that scales with the number of tokens stored.

\paragraph{State Space Models (SSMs).}
An SSM replaces the growing KV cache with a fixed-size \emph{recurrent state} $\mathbf{S}_t \in \mathbb{R}^{d_v \times d_k}$ that is updated at each time step and read out to produce the layer's output $\bm{y}_t$:
\begin{equation}
  \mathbf{S}_t
  = \mathbf{S}_{t-1}\,\mathbf{A}_t + \bm{v}_t \, \bm{B}_t ,
  \qquad
  \bm{y}_t = \mathbf{S}_t\,\bm{q}_t\ + \bm{v}_t\,D_t,
  \label{eq:ssm-generic}
\end{equation}
where $\mathbf{A}_t \in \mathbb{R}^{d_k \times d_k}$ is a \emph{transition operator} that controls how the previous state is transformed, and $\bm{B}_t \in \mathbb{R}^{1 \times d_k}$ is a \emph{gating operator} that controls how the incoming how the incoming key-value pair is written into the state. 
$D_t \in \mathbb{R}$ is a \emph{feedthrough} scalar that controls how much of the current input is passed directly to the output without being mediated by the state. The state $\mathbf{S}_t$ acts as a \emph{fading memory}: it compresses the entire history into a representation of bounded size, yielding $O(1)$ memory regardless of context length, but inherently losing information as the sequence grows.

Different SSM families are distinguished by their choice of $\mathbf{A}_t$ and $\bm{B}_t$; we describe the three families used in this work in \cref{subsec:ssm-zoo}. Note that when $\mathbf{A}_t = \mathbf{I}$ (the identity) and $ \bm{B}_t = \bm{k}_t^T$, the recurrence reduces to vanilla Linear Attention~\cite{linearattn}, making the SSM framework a strict generalisation of Linear Attention.

\subsection{Hybrid Models and Priming}
\label{subsec:hybrid-def}

A \emph{Hybrid model}, in our usage, is an auto-regressive language model that employs \emph{both} Attention layers and SSM layers as sequence-mixing sublayers within the same network \cite{bmojo}.
Standard Transformers store all tokens verbatim in a KV cache that grows linearly with context, making long-context training expensive and inference demanding in both memory and latency. Hybrid models address this by replacing a subset of Attention layers with recurrent SSM layers, each of which maintains a fixed-dimensional latent state that acts as a compressed summary of the remote past. Because only the remaining Attention layers require a KV cache, the overall cache footprint shrinks and decoding becomes faster. The resulting architecture balances two complementary memory regimes: eidetic memory from Attention (quadratic cost) and fading memory from SSMs (linear cost) \cite{bmojo}.

Several recent efforts have developed Hybrid models trained from scratch. Jamba~\cite{jamba}, Griffin~\cite{griffin}, Samba~\cite{samba}, the Nemotron family~\cite{nemotron3, nemotron_nano2, nemotronsuper} and B'MOJO \cite{bmojo} each propose architectures that interleave Attention and SSM layers in a regular pattern or that combine Attention and SSMs within the same layer like in B'MOJO \cite{bmojo}. 
These approaches demonstrate strong results but require the full cost of pre-training from a random initialization.

\paragraph{Priming.}
Our work takes a different approach.
\emph{Priming} instantiates a Hybrid model from the weights of an existing pre-trained Transformer, using both its architecture and its learned parameters as a starting point.
This avoids the cost of training from scratch and enables any Transformer with public weights to serve as a \textit{source model}.
For context, the pre-training of Qwen3-8B consumed 36 trillion tokens; our Primed Hybrid 8B models require fewer than 150 billion training tokens, less than $0.5\%$ of the pre-training budget.

\subsection{Priming Design Space}
\label{subsec:design-space}

Three axes parameterize the space of Primed Hybrid architectures:
\begin{enumerate}[label=(\arabic*)]
  \item The \textbf{Hybrid ratio}: the fraction of sequence-mixing
        layers that are SSM layers. A 75\% ratio means three out of every four sequence-mixing layers are SSMs; higher ratios yield greater inference efficiency at the cost of reduced eidetic capacity.
  \item The \textbf{layer pattern}: the specific assignment of each position in the layer stack to either Attention or SSM. As Priming instantiates a Hybrid model from the weights of a pre-trained Transformer, the choice of \emph{which} layers are SSM layers matters: not all layers in the source Transformer are equally amenable to SSM instantiation. (\cref{subsec:from_theory_to_practice}). We address this via importance-based layer selection in~\cref{subsec:layer_selection}.
  \item The \textbf{SSM layer type}: which SSM family is used for the
        fading-memory layers. We evaluate three representatives: Mamba-2, Gated DeltaNet (GDN), and Gated KalmaNet (GKA). Each can optionally be paired with Sliding Window Attention within a single Hybrid sublayer (the B'MOJO-F configuration described in~\cref{subsec:bmojof}).
\end{enumerate}

\subsection{The Fading--Eidetic Tradeoff}
\label{subsec:tradeoff}

The design of a Hybrid model is fundamentally a question of how to balance two kinds of memory.
An Attention layer stores every key--value pair $(\bm{k}_i, \bm{v}_i)$ exactly in its KV cache, providing precise retrieval at a cost that grows with context length. An SSM layer compresses the entire history into a single state matrix $\mathbf{S}_t \in \mathbb{R}^{d_v \times d_k}$, bounding memory regardless of context length but inevitably losing information.

Given a target Hybrid architecture with a desired ratio of eidetic to fading memory layers (e.g.\ 50\%), the key design question is which layers should have eidetic and which should have fading memory. We use the source Transformer to answer this: its layer-wise behavior reveals where eidetic memory is structurally necessary and where fading memory suffices, guiding the placement of each layer type in the Hybrid. Identifying this assignment is the subject of our layer selection procedure (\cref{subsec:layer_selection}), and the formal conditions under which a fading-memory SSM can approximate an Attention layer's function are developed next in~\cref{sec:why-priming}.

\section{What Is Priming and Why Does It Work?}
\label{sec:why-priming}

Training a Hybrid LLM from scratch demands the full cost of pre-training. \emph{Priming} sidesteps this by initializing a Hybrid model directly from a pre-trained Transformer's weights, a cross-architecture knowledge transfer that consumes fewer than $0.5\%$ of the source Transformer's pre-training token budget. \emph{Why should such a transfer work at all?}
Attention and SSM layers compute different input-output mappings: one produces a softmax-weighted retrieval over an eidetic cache, the other evolves a compressed recurrent state through learned transitions.
On what grounds can an SSM inherit the representations induced by Attention weights? In this section we develop a principled answer, grounded in the key ideas of classical realization theory of dynamical systems.%

\subsection{Realization-theoretic Foundations for Attention-to-SSM Priming}

\emph{When and why can an SSM layer recover the representations of an Attention layer in a pre-trained Transformer?}
Classical realization theory from system identification~\citep{lindquist1979stochastic,DewildeVanderVeen1998,ljung1995system} gives us a useful language to design and understand our Priming pipeline. 

\subsubsection{Attention as an Input-Output Map}

Consider a single Attention head in a layer of the source Transformer. Given an input sequence $\mathbf{X} \in \mathbb{R}^{T \times d}$, causal (auto-regressive) self-attention computes 
\begin{equation} 
\mathbf{Y} = \mathbf{M}(\mathbf{X}) \mathbf{V}(\mathbf{X}), 
\end{equation} 
where $\mathbf{V}(\mathbf{X}) \in \mathbb{R}^{T \times d}$ is the matrix of value projections and $\mathbf{M}(\mathbf{X}) \in \mathbb{R}^{T \times T}$ is the \emph{mixing matrix}, lower-triangular (causal), row-stochastic (due to softmax), and input-dependent (function of keys and queries). The entry $M_{ij}$ encodes how much token $j$ contributes to the representation of token~$i$.

Following the \emph{matrix mixer} abstraction introduced by~\citet{ssd} and operationalized for distillation by~\citet{mohawk}, both Attention and SSMs can be viewed as applying a $T \times T$ mixing matrix to a token sequence. Attention produces $\mathbf{M}$ via softmax over query-key inner products while an SSM with state dimension $n$ produces a lower-triangular semi-separable matrix $\mathbf{T}$ of order $n$ \cite{lindquist1979stochastic}, determined by the inputs and the recurrent dynamics. This shared abstraction is the foundation for Priming: if the SSM's mixing matrix can represent the Transformer's, that is, if $\mathbf{T}(\mathbf{X}) \mathbf{V}(\mathbf{X}) \approx \mathbf{M}(\mathbf{X}) \mathbf{V}(\mathbf{X})$, then the SSM inherits the Attention layer's input-output behavior and the value projections and downstream MLPs can be transferred directly. This mixing matrix $\mathbf{M}(\mathbf{X})$ is therefore the central object of interest for Priming Hybrid architectures. In the next section we develop theoretical connections that characterize when $\mathbf{M}(\mathbf{X})$ can be realized by an SSM.

\subsubsection{The Realization Theory Problem}
\label{subsec. realization theory}
An SSM with state dimension $n$ (e.g., Mamba-2, GDN, GKA) yields a \emph{finite-horizon input-output matrix}\footnote{We refer to the SSM's mixing matrix as its \emph{finite-horizon input-output matrix} to distinguish it from Attention's mixing matrix $\mathbf{M}(\mathbf{X})$.} $\mathbf{T}(\mathbf{X}) \in \mathbb{R}^{T \times T}$ that relates inputs to outputs over a length-$T$ sequence (see \cref{sec: mixing matrix and finite horizon input-output matrix} for the explicit form for a general linear time-varying SSM). This matrix is lower-triangular and \emph{semi-separable of order $n$}: every submatrix lying strictly below the main diagonal has rank at most $n$, because all information from the past must pass through the $n$-dimensional state bottleneck.

Classical realization theory~\citep{DewildeVanderVeen1998, lindquist1979stochastic} studies the following question: given a causal input-output map, here represented by the lower-triangular Attention mixing matrix $\mathbf{M}(\mathbf{X})$, what is the smallest state dimension $n$ of a state-space model that can realize it exactly?

The key tool used in the literature of linear input-varying SSMs is the rank of the \emph{Hankel submatrices}, which measure how much information must be carried from the past into the future. A Hankel submatrix is obtained by cutting the mixing matrix at time $k$ and retaining only the block that maps past inputs to future outputs. Its rank tells us how many degrees of freedom must pass through the latent state. 
For each split point $k$, define $\mathbf{H}(\mathbf{X})_k = \mathbf{M}(\mathbf{X})_{k+1:T,1:k},$ the block that maps inputs from the first $k$ timesteps (the ``past'') to outputs from timesteps $k+1,\dots,T$ (the ``future''). Intuitively, $\mathbf{H}(\mathbf{X})_k$ captures the past-to-future influence across the time cut at instant $k$.
The minimal state dimension to realize the input-output causal map (given by $\mathbf{M}(\mathbf{X}$)) is then
\begin{equation}\label{eq:max_hankel_rank}\tag{Hankel rank}
n_{\min}(\mathbf{X}) = \max_{1\leq k < T} \operatorname{rank}(\mathbf{H}(\mathbf{X})_k).
\end{equation}

At one extreme, any $T \times T$ lower-triangular matrix can always be realized by an SSM with state dimension $T$, by storing the entire history in the state. This is a trivial bound that defeats the purpose of using a compact recurrent state. The key result from realization theory is that there exists an SSM with state $n \ll T$ that can realize $\mathbf{M}(\mathbf{X})$ exactly whenever the Hankel rank is bounded by $n$. We make this notion precise in the following theorem that adapts the classical time-varying realization theorem \citep[Theorem~3.7]{DewildeVanderVeen1998} to the causal Attention setting of our work, we refer the reader to our proof in \cref{app:realization-theory}. \footnote{The actual SSM layers in this work operate on matrix-valued states 
$\mathbf{S}_t \in \mathbb{R}^{d_v \times d_k}$ with vector-valued inputs and 
outputs. We show how to adapt the classical proof for the scalar input-output case in \cref{app:realization-theory}. The key insight is that the \emph{finite-horizon input-output matrix} acts identically and independently on each of the $d_v$ columns, so the analysis applies element-wise.}

\begin{theorem}[Finite-Horizon Per-Instance Realization of the Attention Map]
\label{thm:realization_of_mixing_matrices}
Fix a sequence $\mathbf X\in\mathbb R^{T\times d}$, let $\mathbf V(\mathbf X)\in\mathbb R^{T\times d_v}$ be the sequence of values, and let $\mathbf M(\mathbf X)\in\mathbb R^{T\times T}$ be the causal Attention mixing matrix of a single head. 
Let $\mathbf X_{\le t}$ denote the input sequence up to time~$t$. Then there exists a linear time-varying state-space realization
\[
\mathbf S_{t}
  = \mathbf S_{t-1}\,\mathbf A_t(\mathbf X_{\le t})
    + \bm{v}_t\,\mathbf B_t(\mathbf X_{\le t}),
\qquad
\bm{y}_t = \mathbf S_t\,\mathbf C_t(\mathbf X_{\le t})
           + \bm{v}_t\,D_t(\mathbf X_{\le t}),
\]
with $\bm{v}_t\in\mathbb R^{d_v}$, $\bm{y}_t\in\mathbb R^{d_v}$, zero initial state, and state dimension $n_{\min}(\mathbf X)$, whose \textit{finite-horizon input-output matrix} $\mathbf T(\mathbf X)$ is exactly the \textit{mixing matrix} $\mathbf M(\mathbf X)$. Moreover, no realization of this finite-horizon input-output map with smaller state dimension exists.
\end{theorem}

\Cref{thm:realization_of_mixing_matrices} is encouraging but presents two practical challenges for initializing Hybrid models from a pre-trained
Transformer via Priming. \textbf{(i) Per-instance existence.} The guarantee
holds for a fixed $\mathbf{X}$: different input sequences generally induce
different minimal realizations, with different state dimensions and different
parameter maps $(\mathbf{A}_t, \mathbf{B}_t, \mathbf{C}_t, D_t)$. An SSM layer in a Primed Hybrid model, however,  is expected to operate across multiple input sequences with a single shared parametrization. \textbf{(ii) Arbitrarily large Hankel rank.} Since Attention is eidetic,
nothing constrains its mixing matrix to be low-rank across time cuts, that is,
$n_{\min}(\mathbf{X})$ can be as large as $T$. Approximating $\mathbf{M}(\mathbf{X})$ with an SSM of state size $n \ll T$ is therefore prone to approximation errors.

Priming addresses these challenges on two fronts. For (i), rather than realizing a per-instance SSM, we optimize a \emph{single} SSM parametrization
over the training distribution, so the realization is amortized across inputs rather than constructed per-sequence. For (ii), we exploit the empirical
observation that \ref{eq:max_hankel_rank} varies substantially across layers: some layers perform local computations that an SSM can
capture with modest state size, while others perform complex, long-range,
input-dependent retrieval that requires a large state dimension to realize with an SSM. Our Hybrid architecture is therefore designed around this heterogeneity: high-Hankel-rank
layers are kept as Attention, and only layers amenable to low-rank realization are replaced with SSMs. \Cref{subsec:layer_selection} describes our selection
procedure.

\subsubsection{From Theory to Practice} \label{subsec:from_theory_to_practice}

The difficulty of transferring Attention layers into SSM layers is that, \emph{a priori}, we do not know how the \ref{eq:max_hankel_rank} of $\mathbf{M}(\mathbf{X})$ varies across the input distribution. The learned SSM parameterization must therefore handle the full distribution, including the high-rank Attention patterns that strain the SSM's capacity. This is what makes naive one-shot distillation insufficient and motivates the multi-stage design of Priming:
\begin{enumerate}[nosep] 
\item \emph{Avoid the hardest cases.} Layer selection (\cref{subsec:layer_selection}) identifies Attention layers in the source Transformer whose mixing matrices are empirically most complex (high \ref{eq:max_hankel_rank}) and keeps them as Attention in our Hybrid model. Thus, rather than forcing an approximation that the theory predicts will be poor, or that requires a prohibitively large state size $n$, we keep those layers as Attention. 
\item \emph{Start from a good initialization.} Following layer selection, Stage 0 (\cref{subsec:stage0-theory}) uses the Attention-SSM correspondence to initialize the SSM layers in our Hybrid model near a known (yet imperfect) realization, rather than from random weights. 
\item \emph{Provide enough capacity to cover the input distribution.} State expansion (\cref{subsec:stateexp-theory}) increases the SSM's degrees of freedom, through learned projections, so that the parameterization has more room to represent complex temporal patterns.
\item \emph{Align per-layer structure.} Stage 1 (\cref{subsec:stage1-theory}) trains the SSM layers in the Hybrid model to match the Attention layer's actual output (in the source Transformer) across a distribution of inputs, this is the practical \textit{amortized} version of the per-instance realization algorithm. 
\item \emph{Align to post-training data distribution.} Stage~2 (\cref{subsec:stage2-theory}) optimizes the full Hybrid model with cross-entropy loss end-to-end to obtain strong models on downstream domains of interest for the users (e.g. long-context reasoning and instruction following).
\end{enumerate}
The next sections describe the key ideas in each stage of our Priming pipeline in detail.

\subsection{The Priming Pipeline as Approximation of the Realization-Theoretic Ideal}
\label{subsec: priming pipeline}
In principle, one could compute the \ref{eq:max_hankel_rank} of each Attention layer's mixing matrix per input and amortize these realizations into a single SSM parameterization. However, this is not practical since, the entire input distribution cannot be enumerated at training time, $\mathbf{M}(\mathbf{X})$ is never materialized on the GPU \cite{dao2022flashattention}, and no closed-form amortization of the per-input realizations into a single SSM parameterization exists. Priming addresses each of these obstacles through a sequence of pragmatic approximations.

\subsubsection{Layer Selection: Avoiding the Hardest Cases}
\label{subsec. layer select}
Realization theory predicts that Attention layers with high \ref{eq:max_hankel_rank} are the hardest to approximate with a fixed-size SSM, so we keep those layers as Attention. Computing the \ref{eq:max_hankel_rank} directly is impractical, so we use a cheap empirical proxy: replace each Attention layer individually with Sliding Window Attention (SWA) of window $w$ and measure the resulting degradation on a set of target benchmarks. A SWA mixer has banded structure, which caps its Hankel rank at $w$ uniformly across inputs; a layer that tolerates this substitution therefore has low effective Hankel rank and is a natural candidate for SSM instantiation. Conversely, layers that degrade sharply under SWA rely on high-rank structure (long-range retrieval, global aggregation, sharp input-dependent routing) that a fixed-size SSM would struggle to realize, and we retain these as Attention in the Primed Hybrid. We refer to the per-layer SWA-induced degradation as the layer's \emph{importance}; layers with high importance are those most sensitive to the SWA substitution and hence predicted by the theory to have high Hankel rank.

Empirically, this proxy works well. Importance-based layer selection consistently outperforms uniform (evenly spaced) placement across all model scales we tested (\cref{subsec:pattern_ablations}). Alternative heuristics, such as concentrating Attention layers in the early or late portions of the network, degrade quality; in our experiments the high-importance layers do not follow a regular or predictable pattern across the source Transformer (see \cref{fig:importance_scores}).

\paragraph{Importance-based layer selection algorithm.}
\label{subsec:layer_selection}

For each Attention layer $i$ in the source Transformer, we create a variant in which only layer $i$ is replaced with SWA of window size $w$\footnote{We use $w = 2048$ in all our experiments.} and evaluate this variant on a set of tasks representative of long-context abilities, such as many-shot in-context learning, retrieval-augmented generation, and passage re-ranking. The \emph{importance} of layer $i$ is the performance drop relative to the unmodified baseline. Given a target of $M$ SSM layers (determined by the desired Hybrid ratio), we rank all layers by their importance scores and instantiate the $M$ least important layers as SSM layers. The full procedure is summarized in Algorithm~\ref{alg:layer_selection}. The resulting selection pattern defines the Hybrid architecture produced by Stage~0 of Priming (\cref{subsec:stage0-theory}).

In \cref{subsec:pattern_ablations}, we apply this procedure to Qwen3-8B and Qwen3-32B and find that importance is concentrated in a small set of middle-to-late layers (Figure~\ref{fig:importance_scores}). Using the resulting selection pattern during Priming outperforms a uniform (evenly spaced) pattern by 30\% (relative) on long-context aggregate performance.

\subsubsection{Stage~0: Transferring Knowledge from Source Transformer}\label{subsec:stage0-theory}

Stage~0 provides a good initialization for the SSM layers by exploiting a structural correspondence between Attention and the SSM recurrence. The correspondence motivates a direct transfer of the source Transformer's projection matrices into the Hybrid's SSM layers, without training. This lets Priming bootstrap the alignment between Hybrid SSM layers and their source Attention counterparts by re-using the Transformer's pre-trained weights rather than starting from random initialization.

\paragraph{Query, key, and value projections.} As described in \cref{sec:hybrid-arch}, both Attention and SSM layers perform their computations on a common set of query $\mathbf{q}_t$, key $\mathbf{k}_t$, and value $\mathbf{v}_t$ vectors (\cref{eq:attention} and \cref{eq:ssm-generic}), obtained by linear projections $\mathbf{W}_Q, \mathbf{W}_K, \mathbf{W}_V$ applied to the input $\mathbf{x}_t$. Priming exploits this correspondence by initializing the SSM's projection matrices directly from the corresponding Attention layer in the source Transformer.

\paragraph{Output and gating projections.} Beyond these parameters, both Attention and SSM layers employ an output projection $\mathbf{W}_O$ that maps per-head outputs back to the model dimension, which we transfer directly from the source Attention layer. The SSM additionally has a gating projection $\mathbf{W}_G$ that modulates the recurrent readout $\mathbf{y}_t$ before it is mapped by $\mathbf{W}_O$. Although $\mathbf{W}_G$ has no Attention counterpart, it operates in the same representational space as $\mathbf{W}_O$ and $\mathbf{W}_V$, so we initialize it from these two matrices (details and ablation in \cref{sec:priming_gate_init}) to reuse the knowledge stored in the source Transformer's weights rather than starting from random initialization as in prior work~\citep{mambainllama}. The remaining SSM-specific parameters (transition dynamics, convolutions, and other layer-type-specific components) have no Attention counterpart and are initialized randomly.

\paragraph{Other Hybrid components.} All remaining components of the Hybrid model, the MLPs, layer norms, embeddings, LM head, and the Attention layers retained by layer selection (\cref{subsec:layer_selection}), are copied directly from the source Transformer. These components are unchanged by the Hybrid architecture, so their pre-trained weights transfer without any modification.\\

This initialization is not the exact parameter setting for the SSM that realizes the Attention layer's mixing matrix: while realization theory guarantees that such a setting exists, copying the Attention projections gives us an approximate surrogate rather than the true solution. The transferred projections already encode the subspaces in which the Attention layer operates, giving the SSM a well-aligned starting point. Stage~1 (\cref{subsec:stage1-theory}) closes the remaining gap by learning parameters that amortize the realization across the input data distribution.

\subsubsection{State Expansion via Adaptive GQA}\label{subsec:stateexp-theory}

Realization theory tells us that the SSM state must carry enough degrees of freedom to approximate $\mathbf{M}(\mathbf{X})$ across the input distribution (see \cref{subsec. realization theory}). A grouped KV structure (inherited from the source Transformer during Stage~0) reduces the SSM's effective state capacity and thus its ability to meet this requirement. We unpack this mismatch below, then describe our State Expansion approach as a remedy to restore the SSM's capacity.

Modern Transformers use Grouped Query Attention (GQA)~\citep{gqa}, where a small number of key-value heads $H_{\mathrm{KV}}$ are shared across a larger number of query heads $H_Q$. Since Stage~0 reuses the Transformer's projection weights to initialize the SSM layers, this grouped structure carries over to the Hybrid model. In Attention, sharing KV heads does not substantially degrade performance, since the full token history is retained in the reduced KV cache.

The SSM state, by contrast, is already a compressed summary of the entire history. As defined in \cref{eq:ssm-generic}, the transition operator $\mathbf{A}_t$ and the gating operator $\bm{B}_t$ are functions of the current key $\bm{k}_t$, with value $\bm{v}_t$ being written to the state through $\bm{v}_t\,\bm{B}_t$ (see \cref{tab:ssm-recurrences-main} for the specific forms in each SSM we consider). When $\bm{k}_t$ and $\bm{v}_t$ are shared across the $m = H_Q / H_{\mathrm{KV}}$ query heads in a group, the $\bm{k}_t$-dependent components of $\mathbf{A}_t$ and $\bm{B}_t$ are shared too, and the state updates for heads in the same group can differ only through per-head gating/decay parameters ($\gamma_t, \beta_t$, etc.). This reduces the SSM's effective state capacity relative to having $m$ independent key-value heads.  The per-head query $\bm{q}_t$ controls only what is read from the resulting state, and cannot recover information that was never written or was already forgotten. Attention does not suffer from this restriction: the full history is retained verbatim in the KV cache (\cref{eq:attention}), and each query head can retrospectively select different information via softmax-weighted routing, even when keys and values are shared within a group.

\paragraph{Adaptive GQA (AGQA).}
To recover per-head diversity without abandoning the GQA initialization, we introduce Adaptive GQA (AGQA). Instead of replicating each KV head $m$ times identically, we learn a lightweight, low-rank projection that maps the $H_{\mathrm{KV}}$ shared heads to $H_Q$ potentially distinct ones. A residual connection ensures the projection starts as standard GQA replication (by initializing the learned branch to zero) and gradually deviates during training, introducing per-head variation only as needed (see \cref{sec:priming_agqa} for details).

AGQA adds less than $1\%$ of the total model parameters for Primed Hybrid 32B models. A related method, M1~\citep{m1}, uses a full linear projection requiring ${\sim}1.8\times$ more parameters than AGQA and no residual connection, so its projection must learn the GQA mapping from scratch. By recovering per-head diversity with minimal overhead, AGQA allows the SSM layers of our Primed Hybrid model to better utilize their state capacity despite inheriting a grouped KV structure from the source Transformer.

\subsubsection{Stage~1: Layerwise Alignment}\label{subsec:stage1-theory}

Stage~0 initializes the SSM via a correspondence with the Attention layer's projection matrices. However, this yields only an approximate realization of the Attention layer. Stage~1 closes this gap by \emph{amortizing} the per-instance realization argument across the data distribution: we train the SSM parameters on generic web text data (\cref{training_recipes:stage1}). Concretely, we use the source Transformer as a teacher and train the SSM layers of the Hybrid to replicate the teacher's outputs. This is a \textit{practical proxy} for aligning the SSM's finite-horizon input-output matrix $\mathbf{T}(\mathbf{X})$ with the Attention mixing matrix $\mathbf{M}(\mathbf{X})$, without having to materialize the full $T\times T$ matrices. Two alignment objectives naturally arise.

\paragraph{Layerwise Alignment Objectives.}
We define a \emph{decoder layer} as a block consisting of a sequence mixer (Attention or SSM) and an MLP. Let $\ell$ denote the total number of decoder layers, $F_i$ the $i$-th decoder layer, and $\mathbf{X}_i = F_i(\mathbf{X}_{i-1})$ the hidden states after layer~$i$, with $\mathbf{X}_0$ denoting the sequence of input token embeddings. The alignment objectives operate on hidden states via a Mean-Squared Error (MSE) loss. The following are the MSE objectives considered in this work.

$\;$ %

\noindent \emph{End-to-end MSE.} The loss is applied once, on the final hidden states before the LM head:
\begin{equation}
\label{eq:e2e-loss}
\mathcal{L}_{\mathrm{e2e}}
  = \bigl\| \mathbf{X}_\ell^{\mathrm{Hybrid}} - \mathbf{X}_\ell^{\mathrm{Transformer}} \bigr\|^2,
\end{equation}
where $\mathbf{X}_\ell^{\mathrm{Hybrid}}$ and $\mathbf{X}_\ell^{\mathrm{Transformer}}$ denote the final-layer hidden states of the Hybrid model and the source Transformer respectively, each passed through the final layer norm (pre-LM Head representation) before comparison. This objective gives the SSM layers freedom to adjust their own internal representations, they need only match the teacher at the final output, not at every intermediate layer.

$\;$ %

\noindent \emph{Layerwise MSE.} The loss is applied at each decoder layer independently:

\begin{equation}
\label{eq:lw-loss}
\mathcal{L}_{\mathrm{LW}}
  = \frac{1}{\ell} \sum_{i=1}^{\ell}
    \bigl\| F_i^{\mathrm{Transformer}}(\mathbf{X}_{i-1}^{\mathrm{Transformer}})
         - F_i^{\mathrm{Hybrid}}(\mathbf{X}_{i-1}^{\mathrm{Hybrid}}) \bigr\|^2.
\end{equation}

\noindent This provides denser supervision by constraining each SSM layer to align with the Attention layer's output at the same depth. Note that $F_i$ includes both the sequence mixer and the MLP so the alignment operates at the full decoder-layer level, giving the MLP room to compensate for approximation errors in the SSM.\\

We ablate both objectives in \cref{subsubsec:e2e-vs-lw_ablations} and find that end-to-end MSE outperforms the layerwise variant. Consequently, Primed Hybrid models use end-to-end MSE for Stage 1 alignment unless otherwise specified.

\subsubsection{Stage~2: Task Adaptation}\label{subsec:stage2-theory}

After Stage~1, the Hybrid model retains the source Transformer's high \ref{eq:max_hankel_rank} Attention layers unchanged and replaces the remaining layers with SSMs trained to \textit{realize} their low-rank Attention counterparts. Stage~2 takes this standalone model and fine-tunes it end-to-end on the downstream task of interest using standard cross-entropy next-token prediction. In this work we consider two recipes, both starting from the Stage~1 Hybrid model.

\paragraph{Long-context Instruct model.}
We first extend the model's effective context length to 128K via continued pretraining on long-context data. We then perform supervised fine-tuning (SFT) on instruction-following data to produce a general-purpose long-context instruct model. The full training pipeline for our instruct-tuned models is detailed in \cref{training_recipes:it}.

\paragraph{Long-context Reasoning model.}
We first perform reasoning SFT at 32K context on chain-of-thought traces. We then extend the context to 128K using a mixture of reasoning traces and long documents, and finally align the model to the instruction format via a long-context instruction-alignment phase. The full training pipeline for our reasoning models is detailed in \cref{subsec:long-context-reasoning}.
\\

\section{SSM Layers for Fading Memory in Hybrids}
\label{sec:ssm-comparison}

We now describe the specific sequence-mixing layers we use to realize Attention in our Primed Hybrid models. We consider four candidates, the pure-SSM families Mamba-2 \citep{ssd}, GDN \citep{yang2024gated}, and GKA \citep{gatedkalmanet}, along with B'MOJO-F \citep{bmojo}, which is a Hybrid layer that couples an SSM with Sliding Window Attention inside a single sublayer. We describe the three pure-SSM families next through a common lens that highlights their differences in expressivity and state-update dynamics in \cref{subsec:ssm-zoo}. GKA receives additional treatment in \cref{sec:gka_main_ideas}, owing to its distinctive property of exposing a runtime compute-quality knob that can be adjusted at inference time without retraining. B'MOJO-F, being a Hybrid layer, is discussed separately in \cref{subsec:bmojof};

\subsection{SSM Layer Types}
\label{subsec:ssm-zoo}

The three SSM layer types we consider share the generic recurrence \cref{eq:ssm-generic}, differing only in the structure of the transition operator $\mathbf{A}_t$ and the gating operator $\mathbf{B}_t$. These structural differences determine how each layer manages its finite memory. We describe the three choices below and summarize their transition structures in \cref{tab:ssm-recurrences-main}.

\paragraph{Mamba-2~\citep{ssd}.}
Mamba-2 can be represented as the simplest state update (up to specific parametrization): an input-dependent scalar decay applied
uniformly to the entire state, combined with a direct key write,
\begin{equation}
\mathbf{A}_t = \gamma_t\,\mathbf{I},
\qquad
\bm{B}_t = \bm{k}_t^\top
\qquad\Longrightarrow\qquad
\mathbf{S}_t = \gamma_t\,\mathbf{S}_{t-1} + \bm{v}_t\,\bm{k}_t^\top,
\label{eq:mamba-2}
\end{equation}
where $\gamma_t \in (0,1)$ is a scalar decay computed from the current input token $\bm{x}_t$ through a learned projection and nonlinearity. Because $\mathbf{A}_t$ is a scalar multiple of the identity, every past key-value association is progressively forgotten as more associations are observed regardless of their content.

\paragraph{Gated DeltaNet (GDN)~\citep{yang2024gated}.}
GDN augments Mamba-2's uniform decay with a \emph{single-key erase} based on the \emph{delta update rule}~\citep{widrow1960adaptive}: before writing the new value, it subtracts from the state a rank-1 term aligned with the current key, selectively removing the value previously associated with that key,
\begin{equation}
\mathbf{A}_t = \gamma_t\bigl(\mathbf{I} - \beta_t\,\bm{k}_t\,\bm{k}_t^\top\bigr),
\qquad
\bm{B}_t = \beta_t\,\bm{k}_t^\top
\qquad\Longrightarrow\qquad
\mathbf{S}_t = \mathbf{S}_{t-1}\,\gamma_t\bigl(\mathbf{I} - \beta_t\,\bm{k}_t\,\bm{k}_t^\top\bigr) + \beta_t\,\bm{v}_t\,\bm{k}_t^\top,
\label{eq:gdn}
\end{equation}
where $\gamma_t, \beta_t \in (0,1)$ are scalar factors computed from the current input token $\bm{x}_t$ through learned projections and nonlinearities, with $\gamma_t$ the scalar decay inherited from Mamba-2 and $\beta_t$ the strength of the key-specific erase. Whereas Mamba-2's $\mathbf{A}_t$ is a pure scalar multiple of the identity, GDN's $\mathbf{A}_t$ has the structure \emph{scalar-times-identity plus a rank-1 correction}, so GDN can erase along a specific key direction without applying equally strong decay elsewhere. In the limit $\gamma_t \to 1$, global decay vanishes and the update becomes a pure key-specific revision; in general, GDN interpolates between uniform decay and key-specific revision at each step. This targeted update makes GDN strictly more expressive than Mamba-2.

\paragraph{Gated KalmaNet (GKA)~\citep{gatedkalmanet}.}
GKA (pronounced "gee-ka") retains GDN's scalar-times-identity-plus-rank-1 structure, but replaces GDN's \emph{single-key erase} with a \emph{history-aware erase} that accounts for every key seen so far. Concretely, the rank-1 erase direction is no longer aligned with the current key alone; instead, it is
an input-varying vector $\mathbf{g}_t^{\top} \in \mathbb{R}^{1 \times d_k}$ that
depends on all past keys $(\bm{k}_1, \ldots, \bm{k}_{t-1})$,
\begin{equation}
\mathbf{A}_t = \mathbf{I} - \bm{k}_t\,\mathbf{g}_t^{\top},
\qquad
\bm{B}_t = \mathbf{g}_t^{\top}
\qquad\Longrightarrow\qquad
\mathbf{S}_t = \mathbf{S}_{t-1}\bigl(\mathbf{I} - \bm{k}_t\,\mathbf{g}_t^{\top}\bigr) + \bm{v}_t\,\mathbf{g}_t^{\top}.
\label{eq:gka}
\end{equation}
The vector $\mathbf{g}_t$ is obtained as a function of all the past keys as follows\footnote{Note, the state update form in Equation (KF) from \cite{gatedkalmanet} is in terms of $\bm\Phi_{t-1}$ which can be obtained using the Sherman-Morrison identity that $\bm{\Phi}_t = \bm{\Phi}_{t-1} - \dfrac{\beta_t\,\bm{\Phi}_{t-1}\,\bm{k}_t\,\bm{k}_t^\top\,\bm{\Phi}_{t-1}}{1 + \beta_t\,\bm{k}_t^\top\,\bm{\Phi}_{t-1}\,\bm{k}_t}$.}:
\begin{equation}
\mathbf{g}_t = \beta_t\,\bm\Phi_{t}\,\bm{k}_t,
\qquad
\bm\Phi_{t} \;:=\; \Big(\sum_{i \leq t} \beta_i \big(\prod_{j>i} \gamma_j\big)\,\bm{k}_i\bm{k}_i^\top + \lambda_t\mathbf{I}\Big)^{-1},
\label{eq:gka-gain-closed-form}
\end{equation}
where $\lambda_t$ is a time-dependent regularizer for proper conditioning of the inverse matrix $\bm{\Phi}_t$ (see \cite[Equation 4]{gatedkalmanet}). We note that the input-selectivity parameter $\beta_t$ in \cref{eq:gka-gain-closed-form} is not part of the original GKA formulation in~\citet{gatedkalmanet}; we introduce it here, and develop the motivation and ablations in \cref{subsec:beta_gating}. The recurrence in \cref{eq:gka}
originates from the Kalman Filter formulation of an online
ridge-regression problem, hence the name \emph{Gated KalmaNet}.

From \cref{eq:gka-gain-closed-form}, the erase direction $\bm{g}_t$ is the current key $\bm{k}_t$ pre-multiplied by the inverse of a regularized,
input-weighted cumulative sum of past outer products
$\sum_{i \leq t} \beta_i \big(\prod_{j>i} \gamma_j\big)\,\bm{k}_i\bm{k}_i^\top$.This is precisely the sense in which $\bm{g}_t$ depends on all past keys: past tokens that wrote weakly (small $\beta_i$) or long ago (small $\prod_{j>i} \gamma_j$) contribute little to $\bm\Phi_t^{-1}$ and, through the inverse\footnote{To the extent that their key directions are not already well-represented by other past keys.}, leave $\bm\Phi_t$ large along their key directions; conversely, tokens that wrote strongly or are recent, shrink $\bm\Phi_t$ along
their directions. Since $\bm{g}_t = \beta_t\,\bm\Phi_t\,\bm{k}_t$, the erase direction
inherits this structure from $\bm\Phi_t$, with the contribution of each past
direction further modulated by its alignment with the current key $\bm{k}_t$. In contrast, Mamba-2 decays all state directions uniformly by $\gamma_t$,
and GDN combines the same uniform decay with a rank-1 erase along the current key $\bm{k}_t$; GKA's erase direction is instead shaped by the entire key history.

\paragraph{Expressiveness hierarchy.}
The three layers form a progression in how the rank-1 correction in $\mathbf{A}_t$ uses the state: Mamba-2 has no correction and simply applies a uniform scalar decay ($\mathbf{A}_t = \gamma_t\,\mathbf{I}$); GDN adds a rank-1 correction aligned with the current key alone
($\mathbf{A}_t = \gamma_t(\mathbf{I} - \beta_t\,\bm{k}_t\bm{k}_t^\top)$); GKA replaces that single-key correction with one whose direction $\bm{g}_t$ depends on all past keys
($\mathbf{A}_t = \mathbf{I} - \bm{k}_t\,\bm{g}_t^\top$). This yields the expressiveness hierarchy
GKA~$>$~GDN~$>$~Mamba-2, with all three sharing the fading-memory property of a bounded state; they differ only in how selectively that state is updated at each step.

\begin{table}[h]
\centering
\caption{Structural differences between the SSM layers considered in this work, expressed through the transition operator $\mathbf{A}_t$ and the write operator $\bm{B}_t$ of the generic recurrence $\mathbf{S}_t = \mathbf{S}_{t-1}\mathbf{A}_t + \bm{v}_t\,\bm{B}_t$. 
For GKA, $\bm{g}_t = \beta_t\,\bm{\Phi}_t\,\bm{k}_t$ with $\bm{\Phi}_t = \bigl(\sum_{i\leq t}\beta_i\prod_{j>i}\gamma_j\,\bm{k}_i\bm{k}_i^\top + \lambda_t\mathbf{I}\bigr)^{-1}$ the regularized, input-weighted inverse of past key outer products (see Eq.~4 of~\citet{gatedkalmanet}).}
\label{tab:ssm-recurrences-main}
\begin{tabular}{llll}
\toprule
\textbf{Layer} & \textbf{Transition operator} $\mathbf{A}_t$ &  \textbf{Gating operator} $\bm{B}_t$ & \textbf{Parameters} \\
\midrule
Mamba-2 & $\gamma_t\,\mathbf{I}$ & $\bm{k}_t^\top$ & $\gamma_t \in (0,1)$ \\
GDN & $\gamma_t(\mathbf{I} - \beta_t\,\bm{k}_t\bm{k}_t^\top)$ & $\beta_t\,\bm{k}_t^\top$ & $\gamma_t, \beta_t \in (0,1)$ \\
GKA & $\mathbf{I} - \bm{k}_t\,\mathbf{g}_t^\top$ & $\mathbf{g}_t^\top$ & $\gamma_t, \beta_t \in (0,1)$, $\lambda_t > 0$ \\
\bottomrule
\end{tabular}
\end{table}

\subsection{Gated KalmaNet at Scale}
\label{sec:gka_at_scale}

Training large language models with SSMs as sequence-mixing layers demands a chunkwise-parallel execution on modern accelerators, yet the GKA state recurrence in \cref{eq:gka} is
inherently sequential and cannot be parallelized directly. To bridge this gap, \citet{gatedkalmanet} implement an equivalent \emph{information form} of the Kalman filter, which rewrites the state in terms of quantities that admit a simple additive recurrence amenable to chunkwise parallelism. This reparameterization is mathematically identical to \cref{eq:gka}. Computing the per-step output of the GKA layer from the information form still requires solving a linear system of equations, which GKA handles with an iterative solver, Chebyshev Iteration. The number of iterations $r$ is therefore a natural runtime hyperparameter, a knob that existing SSMs like Mamba-2 and GDN do not offer. Building on this foundation, in this work we make three contributions to improve how GKA is deployed and implemented at scale.
\begin{enumerate}
    \item \textbf{Input selectivity via $\bm{\beta_t}$}
(\cref{subsec:beta_gating}). We introduce an input-selectivity
parameter $\beta_t \in (0,1)$ that modulates each token's contribution \emph{before} it enters the state, rather than only controlling post-write decay via $\gamma_t$. This gives the layer two complementary controls and consistently improves long-context performance in our ablations.

\item \textbf{Variable test-time compute} (\cref{subsec:chebyshev}). We characterize the compute-performance trade-off exposed by $r$ on downstream long-context tasks: a single trained model can be served with different values of $r$ without retraining, trading a small amount of approximation error for faster decoding when appropriate.

\item \textbf{Hardware-aware decoding kernel} (\cref{kernel}). We develop a Triton kernel for GKA decoding that tiles the layer's state matrices and exploits their symmetry to reduce HBM traffic, yielding a $1.2\text{--}1.25\times$ decode speedup over the original implementation of~\citet{gatedkalmanet}. We evaluate this kernel
alongside our Primed GKA-based Hybrid models at 8B and 32B parameters, demonstrating that GKA scales effectively to large model sizes.
\end{enumerate}

\subsection{Hybrid Layers: B'MOJO-F}
\label{subsec:bmojof}

B'MOJO-F~\citep{bmojo} couples an SSM module with Sliding Window Attention within a \emph{single} sequence-mixing sublayer, giving each layer simultaneous access to fading memory (via the SSM state) and bounded eidetic memory (via the SWA window). Any of the three SSM families above can serve as the SSM component. This makes B'MOJO-F strictly more expressive than any pure-SSM layer, at the cost of running both an SSM and a SWA sublayer per layer, which increases compute and memory relative to a pure-SSM Hybrid at the same Hybrid ratio. We include B'MOJO-F variants in our experiments alongside Primed Hybrid models that use standalone SSM layers.

\section{Priming: Main Results}
\label{sec:experiments}

Priming reduces the cost of building Hybrid models by over two orders of magnitude. To the best of our knowledge, this is the first work that systematically compares Hybrid models variants at scale under identical training conditions (same data, same training recipe, and same initialization) and with only the SSM layer varied. This uniform protocol lets us isolate the effect of the SSM choice itself, and evaluate how different forms of memory (eidetic Attention vs.~fading SSM state) interact across model scale and a diverse set of tasks.

We train two families of Primed Hybrid models that target complementary use cases, both with a 50\% Hybrid ratio. The \emph{Instruct} (IT) family is optimized for standard instruction-following workloads: question answering, summarization, retrieval-augmented generation, and other tasks where the model is expected to produce a direct response to a user prompt without producing long thinking traces. The \emph{Reasoner} family, by contrast, is trained to natively perform chain-of-thought reasoning and tool use, making it suitable for complex multi-step problems in math, science, and coding. We report Instruct model results in \cref{subsec:long-results} and Reasoner model results in \cref{subsec:reasoning-results}. Both our model families are released on Hugging Face.\footnote{\url{https://huggingface.co/collections/amazon/primed-hybrid-models-collection}} We ablate the key design choices of our Priming procedure in \cref{sec:priming_ablations}. 

Priming is a training recipe that gets us to a capable Hybrid model; once trained, the Hybrid architecture itself unlocks several downstream capabilities that we study in subsequent sections. The compact recurrent state of SSM layers reduces per-token inference cost and KV-cache memory, which in turn lets Hybrids generate more samples in the same latency budget as a Transformer. In \cref{subsec:tts} we show that this translates into a better Pareto frontier in the test-time scaling regime: at fixed inference budget, Hybrids reach higher accuracy than the Transformer by sampling more. Beyond sampling efficiency, GKA's iterative solver exposes a run-time compute-quality knob (detailed in \cref{subsec:chebyshev}) that lets a single trained Hybrid support variable test-time compute, trading accuracy for efficiency without retraining. Finally, in \cref{subsec:state-composition} we introduce a training-free context-extension method for Hybrid models that stitches together locally-processed chunks by concatenating their KV caches and composing their recurrent states, extending the context of our trained Hybrid models from 128K to 256K tokens.

\subsection{Primed Instruct Model Results}\label{subsec:long-results}

Our Primed IT models support a native context length of 128K tokens. We detail the training recipe in \cref{sec:training_recipes} and evaluate the models on a diverse set of long- and short-context tasks that require generating responses following user instructions. We study two model scales, 8B and 32B, and use the corresponding dense Qwen3 checkpoint as the source Transformer in each case. As the Transformer baseline, we fine-tune the same Qwen3 checkpoint with our Stage~2 IT recipe (see \cref{training_recipes:it}), so that any performance gap reflects the architectural difference rather than differences in post-training data or hyper-parameters. This model is referred to as Qwen3-\{8B,32B\}\,[Long] in our experiments. Primed Hybrid models are referred to as \{SSM\}-Primed-HQwen3-IT. In this work, we consider the following candidates for \{SSM\}: Mamba2, GKA, GDN, and B'MOJO-F. For the B'MOJO-F variant, the SSM sublayer uses GKA while the SWA branch uses a window size of 2048 tokens (see \cref{subsec:bmojof} for a discussion of the B'MOJO-F layer).

\subsubsection{Long-context Evaluation}
\label{subsec: IT_results_long_ctx}

In this section we evaluate our Primed Hybrid models on tasks that require retrieval and aggregation of information across long context lengths. This regime is particularly challenging for SSM layers since they maintain a fixed-size compressed summary of the past, which becomes harder to manage as the context length grows. Attention layers, thanks to their eidetic memory that stores every token verbatim, do not face this challenge.\\

\noindent\textbf{Tasks.} For long-context evaluation, we consider three recent benchmarks, HELMET, MRCR, and BABILong, that together span synthetic and real long-context tasks across context lengths from 8K to 128K tokens. From HELMET \citep{helmet}, a broad evaluation suite, we report results on five of its tasks: Synthetic Recall, Retrieval-Augmented Generation (RAG), Many-shot In-Context Learning (ICL), Passage Re-Ranking, and Summarization. Multi-Round Co-reference Resolution (MRCR) \citep{vodrahalli2024michelangelo} embeds multiple identical target passages within a synthetic multi-turn conversation and requires the model to retrieve a specific ordinal instance, testing precise co-reference resolution over long conversational contexts. BABILong \citep{babilong} tests reasoning over long context lengths by hiding supporting facts within large amounts of irrelevant filler text, requiring the model to locate and combine them to answer a question. Its ten subtasks range from single-fact retrieval to multi-hop reasoning.
\\

\begin{figure}[h]
    \centering
    \includegraphics[width=1.0\linewidth]{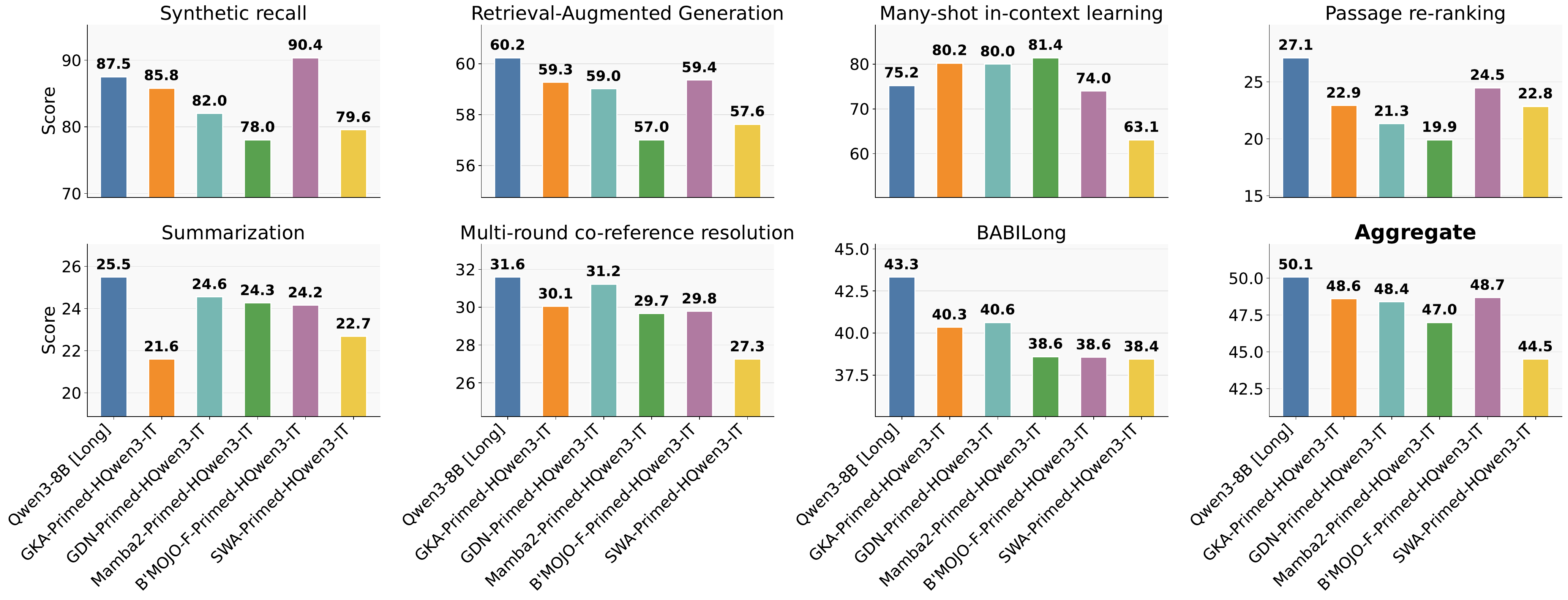}
    \caption{\textbf{Long-context evaluation at 8B scale.} Performance of our Primed Hybrid IT models (GKA, GDN, Mamba2, B'MOJO-F, SWA) against the Qwen3-8B [Long] Transformer baseline fine-tuned with the same Stage~2 IT recipe. Results are reported as a weighted average across context lengths from 8K to 128K, with geometrically increasing weights that double at each successive context length. The rightmost panel (row 2 col 4) aggregates performance across all seven tasks. All SSM-based Hybrids outperform the SWA baseline, with B'MOJO-F-, GKA-, and GDN-based Hybrids within $\sim$3.5\% (relative) of the Transformer in aggregate.}
    \label{fig:it_long_8b}
\end{figure}

\noindent\textbf{8B models.} In \cref{fig:it_long_8b} we report the long-context performance of our Primed Hybrid models across a diverse set of tasks. Each task's reported score is a weighted average across context lengths $s \in \{8\text{K}, 16\text{K}, 32\text{K}, 64\text{K}, 128\text{K}\}$. We adopt a weighting scheme that assigns geometrically increasing weights to longer sequence lengths: the unnormalized weight $w_s$ is taken as $2w_{s/2}$, with initial weight $w_{8\text{K}}{=}1$. This encodes the prior that the \textit{importance} of any single longer context length should exceed the combined importance of all shorter lengths, concentrating the aggregate on the long-context regime. From the results we make the following observations.

\begin{itemize} 
\item \textit{Primed Hybrids deliver 1.8-2.3$\times$ higher decode throughput on long-context tasks while being within $\sim$3.5\% (relative) of the Transformer's performance.} The best Hybrids (B'MOJO-F, GKA, and GDN) are within 1.4-1.7 points of the Qwen3-8B Transformer baseline on aggregate long-context performance despite half of their sequence-mixing layers being SSMs that operate on a fading-memory recurrent state. This comes with a roughly 50\% smaller KV cache, which translates to 1-8-2.3$\times$ higher decode throughput (compared to the source Transformer) for the non-B'MOJO-F Hybrids across context lengths (\cref{tab:decode_8b}). B'MOJO-F is slower at short context because each layer runs both a SWA sublayer and an SSM sublayer with unfused kernels; its throughput starts at $0.88\times$ the Transformer at 16K but surpasses it as context grows, reaching $1.75\times$ at 128K.
\item \textit{Hybrids with fading-memory SSM layers outperform windowed eidetic-memory SWA Hybrids on long-context tasks.} All SSM-based Hybrids outperform the SWA-based Hybrid baseline on aggregate long-context performance by at least 5.5\% (relative). We tested this using SWA-Primed-HQwen3-IT, a model built with the same 50\% Hybrid ratio of the Primed models: 50\% Attention layers and 50\% SWA layers. For a fair comparison, we set the SWA window size so that the SWA KV cache is \emph{at least} the size of the SSM state: with a window of 512 tokens at BF16, the SWA cache is $2\times$ the GKA state (stored at FP32 for stable recurrent dynamics) and $4\times$ the GDN/Mamba2 state (which is half the size of GKA's). Despite this memory advantage, SWA lags all SSM Hybrids, showing that a sliding window cannot substitute for a compact summary of the full context.

\item \textit{Hybrids built on more expressive SSM layers achieve higher aggregate long-context performance.} Among SSM-based Hybrids, the GKA- and GDN-based Hybrids lead in aggregate long-context performance ($48.6$ and $48.4$ respectively), with the Mamba2-based Hybrid trailing at $47.0$. The gap between the GKA- and GDN-based Hybrids is largely driven by recall-heavy tasks: on Synthetic Recall the GKA-based Hybrid outperforms the GDN-based Hybrid by 4.6\% (relative), retrieving specific pieces of information from large amounts of surrounding text, and by 7.5\% (relative) on Passage Re-Ranking, both tasks where a more expressive state update better preserves relevant key-value pairs in the fading memory. The Mamba2-based Hybrid's weaker performance is consistent with the simpler state update of its SSM layer, which applies a scalar data-dependent decay to fade past information uniformly, compared to the more selective per-key write mechanisms of GKA and GDN (see \cref{subsec:ssm-zoo}). These results suggest that more expressive SSM layers complement the eidetic Attention layers more effectively in the Hybrid: the SSM layers contribute a richer compressed summary of the past, leaving Attention to focus on the fine-grained, position-sensitive retrieval that its KV cache handles best. Consequently, we consider only GKA and GDN as the SSM layer for our 32B IT models. 
\end{itemize}

\begin{figure}[h]
    \centering
    \includegraphics[width=1.0\linewidth]{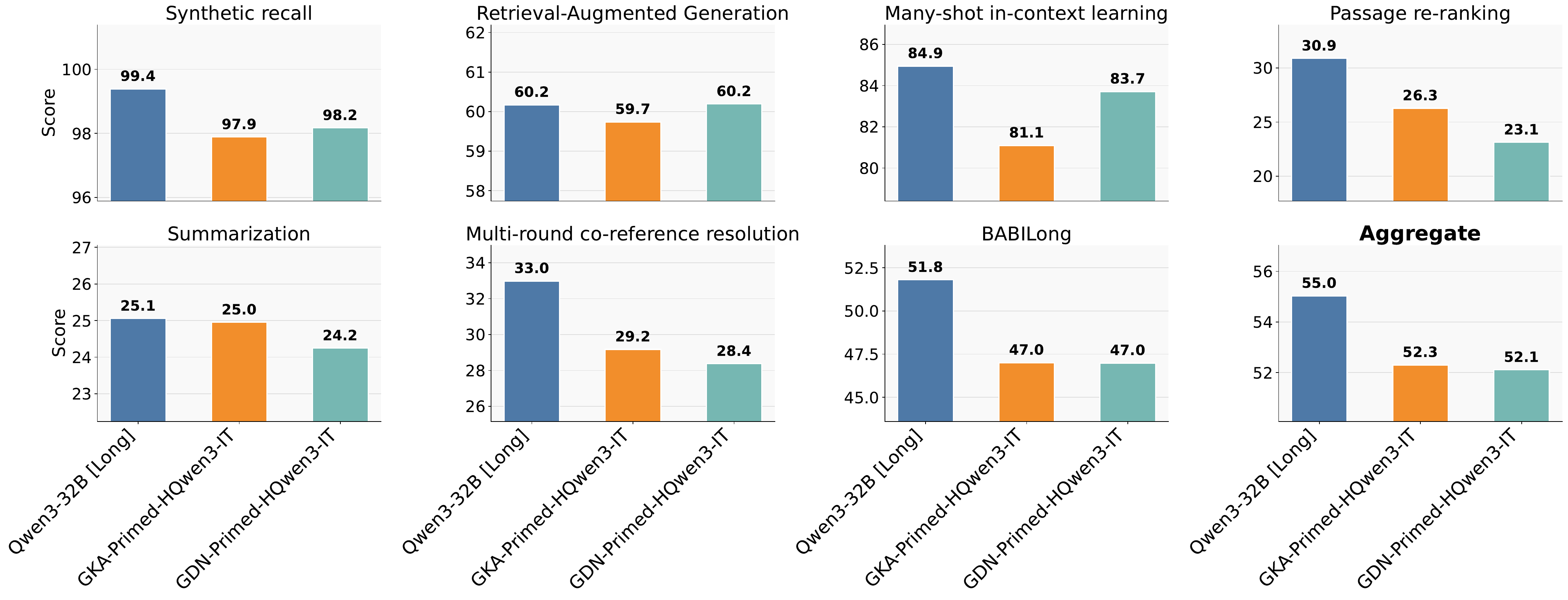}
    \caption{\textbf{Long-context evaluation at 32B scale.} Performance of our Primed Hybrid IT models (GKA and GDN) against the Qwen3-32B [Long] Transformer baseline fine-tuned with the same Stage~2 IT recipe. Results are reported as a weighted average across context lengths from 8K to 128K, with geometrically increasing weights that double at each successive context length. The bottom-right panel (Aggregate) aggregates performance across all seven tasks. The GKA- and GDN-based Hybrids are competitive with each other on long-context performance, within $\sim$5\% (relative) of the Transformer in aggregate.}
    \label{fig:it_long_32b}
\end{figure}
\noindent\textbf{32B models.} In \cref{fig:it_long_32b} we report the long-context performance of our Primed Hybrid models across the same set of tasks used at 8B scale, with scores aggregated across context lengths using the weighting scheme described above. At this scale we evaluate only the GKA- and GDN-based Hybrids, following the 8B finding that the expressivity of the SSM layer tracks aggregate long-context performance among Hybrids. We make the following observations.

\begin{itemize} 
\item \textit{Primed Hybrids at 32B deliver 1.3-2.1$\times$ higher decode throughput on long-context tasks while being within $\sim$5\% (relative) of the Transformer's performance.}  As at 8B, the GKA- and GDN-based 32B Hybrids remain competitive with the Transformer baseline on long-context tasks. The decode speedup over the Transformer, however, is smaller at 32B than at 8B. The reason for this is as follows: the parameter count grows $\sim$4$\times$ from 8B to 32B while the per-layer KV-cache shape\footnote{number of KV heads and head dimension} is unchanged, as a result the fraction of total decode time spent in non-Attention related operations (e.g., MLP weight reads) increases resulting in smaller relative speedup. The exact speedup varies by source Transformer architecture. Mistral, Llama, and other model families each yield different empirical gains, though the Primed Hybrid model is expected to consistently be faster than its Transformer counterpart. See \cref{sec:inference} for further discussion on this topic.
\item \textit{GKA- and GDN-based Primed Hybrids perform on par on aggregate long-context performance.} At 32B scale as at 8B, the two IT Hybrids are within 0.4\% (relative) of each other in aggregate long-context performance. The added expressivity of GKA's state update does not produce an appreciable gain over GDN on the standard long-context tasks considered here. As we show in \cref{subsec:reasoning-results}, however, this expressivity matters on harder tasks that require chain-of-thought reasoning.
\end{itemize}

\subsubsection{Short-context Evaluation}
\label{subsec: IT_results_short_ctx}
Many real-world instruction-following workloads fit comfortably within a few thousand tokens and instead target factual recall, multi-step reasoning, code synthesis, and adherence to user-specified constraints. In this section we evaluate our Primed Hybrid models on such tasks.\\

\noindent\textbf{Tasks.} For short-context evaluation we use the Tulu3-Dev~\citep{tuluv3} benchmark suite, which aggregates tasks across five categories:
\begin{enumerate}
    \item \emph{Knowledge.} (MMLU~\citep{hendrycks2021measuring}, PopQA~\citep{mallen2022not}, TruthfulQA~\citep{lin2022truthfulqa}) Multiple-choice and open-ended QA that probes factual recall across academic subjects, long-tail entities, and common misconceptions.
    \item \emph{Reasoning.} (BBH-CoT~\citep{suzgun2023challenging}) A broad suite of short reasoning problems (logical, symbolic, commonsense, and simple arithmetic) solved with a few steps of chain-of-thought. This probes general reasoning capabilities expected of an instruction-following model, as opposed to the competition-level mathematical and scientific reasoning benchmarks used to evaluate our Reasoning family in \cref{subsec:reasoning-results}.
    \item \emph{Math.} (Minerva-Math~\citep{hendrycks2021measuring}, GSM8K~\citep{cobbe2021training}) Grade-school and undergraduate-level math word problems solved with a few steps of chain-of-thought, as opposed to the olympiad-level problems used to evaluate our Reasoning family.
    \item \emph{Coding.} (HumanEval~\citep{chen2021evaluating}, HumanEval+~\citep{liu2023your}) Python function synthesis from docstrings, evaluated by execution against held-out unit tests.
    \item \emph{Instruction Following (IF).} (IFEval~\citep{zhou2023instruction}) Constrained generation tasks that measure adherence to verifiable format, length, and content constraints.
\end{enumerate}
We refer to these tasks as short-context because they can all be solved within 8K tokens.

\begin{table}[h]
\centering
\caption{
  \textbf{Short-context evaluation at 8B scale.} Performance of our Primed Hybrid IT models (GKA, GDN, Mamba2, B'MOJO-F) against the Qwen3-8B [Long] Transformer baseline fine-tuned with the same Stage~2 IT recipe on the Tulu3-Dev benchmark suite. The \textit{Avg.} column reports the aggregate performance across all five categories, and the \textit{$\Delta$\% (rel.)} column reports the relative gap in aggregate performance with respect to the Qwen3-8B [Long] baseline. The B'MOJO-F-based Hybrid matches the Transformer in aggregate ($+0.1$), and the GKA- and GDN-based Hybrids trail by only $\sim$3\% (relative).}
\label{tab:instruct_short_8b}
\small
\resizebox{0.9\textwidth}{!}{%
\begin{tabular}{lccccc|cc}
\toprule
\textbf{Model} & \textbf{Math} & \textbf{Knowledge} & \textbf{Coding} & \textbf{Reasoning} & \textbf{IF} & \textbf{Avg.} & \textbf{$\Delta$\% \textbf{(rel.)}} \\
\midrule
Qwen3-8B (Long) & 64.6 & 49.8 & 91.0 & 76.3 & 74.5 & 71.2 & --- \\
\midrule
B'MOJO-F-Primed-HQwen3-IT & 65.7 & 48.6 & 90.0 & 76.4 & 75.6 & 71.3 & $+0.1$ \\
GKA-Primed-HQwen3-IT & 64.2 & 47.9 & 90.5 & 72.6 & 71.0 & 69.2 & $-2.8$ \\
GDN-Primed-HQwen3-IT & 59.5 & 48.4 & 91.2 & 73.0 & 73.6 & 69.1 & $-2.9$ \\
Mamba2-Primed-HQwen3-IT & 57.8 & 46.9 & 89.6 & 71.0 & 74.9 & 68.0 & $-4.5$ \\
\bottomrule
\end{tabular}%
}
\end{table}

\noindent\textbf{8B Models.} In \cref{tab:instruct_short_8b} we report short-context performance of our Primed Hybrid models on Tulu3-Dev. We make the following observations.
\begin{itemize}
    \item \textit{Primed Hybrids are competitive with the Transformer on short-context tasks.} The best Hybrid (B'MOJO-F) is on par with the Qwen3-8B Transformer baseline in aggregate, and the GKA- and GDN-based Hybrids trail by only $\sim$3\% (relative). Unlike on long-context benchmarks where the B'MOJO-F-based Hybrid performed on par with the GKA- and GDN-based Hybrids, here it opens a noticeable lead. This is consistent with the eidetic SWA sublayer in B'MOJO-F being well-suited to short-context tasks, where the relevant context comfortably fits inside its 2048-token window and is retained verbatim rather than through a fading SSM state. 
    That said, as noted in \cref{subsec: IT_results_long_ctx}, our Primed B'MOJO-F-based Hybrid is the slowest Hybrid model at short-context lengths because its SWA and SSM sublayers run sequentially with unfused kernels.    
    \item \textit{Hybrids built on more expressive SSM layers achieve higher aggregate performance.} The ordering across SSM layers mirrors the long-context results in \cref{fig:it_long_8b}: GKA and GDN lead, with Mamba2 trailing at a $4.5\%$ relative gap to the Qwen3-8B [Long] baseline in aggregate short-context performance. 
\end{itemize}

\begin{table}[ht]
\centering
\caption{
  \textbf{Short-context evaluation at 32B scale.} Performance of our Primed Hybrid IT models (GKA and GDN) against the Qwen3-32B [Long] Transformer baseline fine-tuned with the same Stage~2 IT recipe on the Tulu3-Dev benchmark suite. The \textit{Avg.} column reports the aggregate performance across all five categories, and the \textit{$\Delta$\% (rel.)} column reports the relative gap in aggregate performance with respect to the Qwen3-32B [Long] baseline. Both Primed Hybrids are competitive with Qwen3-32B [Long], trailing by less than 2\% (relative) in aggregate.
}
\label{tab:instruct_short_32b}
\small
\resizebox{0.9\textwidth}{!}{%
\begin{tabular}{lccccc|cc}
\toprule
\textbf{Model} & \textbf{Math} & \textbf{Knowledge} & \textbf{Coding} & \textbf{Reasoning} & \textbf{IF} & \textbf{Avg.} & $\Delta$\% \textbf{(rel.)}\\
\midrule
Qwen3-32B (Long) & 74.4 & 54.5 & 94.5 & 82.9 & 81.5 & 77.6 & --- \\
\midrule
GKA-Primed-HQwen3-IT & 74.0 & 53.9 & 93.4 & 80.3 & 78.7 & 76.1 & $-1.9$ \\
GDN-Primed-HQwen3-IT & 73.6 & 54.3 & 94.4 & 81.0 & 79.3 & 76.5 & $-1.4$ \\
\bottomrule
\end{tabular}%
}
\end{table}

\noindent\textbf{32B Models.} In \cref{tab:instruct_short_32b} we report short-context performance of our Primed Hybrid models on Tulu3-Dev. Our Primed Hybrid 32B models follow a similar trend as the 8B models, with the GKA- and GDN-based Hybrids trailing the Qwen3-32B [Long] baseline by less than 2\% (relative) in aggregate and performing on par with each other. As at 8B, the added expressivity of GKA's state update does not produce an appreciable gain over GDN on these standard short-context tasks; however, a difference emerges when one considers more challenging reasoning-based tasks (see \cref{subsec:reasoning-results}).

\subsection{Primed Reasoning Model Results}\label{subsec:reasoning-results}
Our Primed reasoning models support a native context length of 128K tokens and are trained to produce chain-of-thought traces before responding. We detail the training recipe in \cref{sec:training_recipes} and evaluate the models on a diverse set of reasoning tasks in math, science, and coding. We study two model scales, 8B and 32B, and use the corresponding dense Qwen3 checkpoint as the source Transformer in each case. Based on our IT results, we restrict the Hybrid SSM layer to our two most promising candidates, GKA and GDN. We do not consider B'MOJO-F in this setting: while competitive on quality, its unfused SWA and SSM sublayers make it slower than pure SSM Hybrids (see \cref{subsec: IT_results_long_ctx}), which is especially detrimental for reasoning models that generate long thinking traces before answering user prompts. Our Primed Hybrid reasoning models are referred to as \{SSM\}-Primed-HQwen3-Reasoner. We have released our family of Primed reasoning models on HuggingFace\footnote{\url{https://huggingface.co/collections/amazon/primed-hybrid-models-collection}}.

At each model scale, we consider two Transformer baselines. Qwen3-\{8B,32B\} is the source Transformer off-the-shelf, which is itself a strong reasoning model. Qwen3-\{8B,32B\}\,[Reasoner-SFT] is the same checkpoint fine-tuned with our Stage 2 reasoning recipe (\cref{subsec:long-context-reasoning}), so that the comparison against our Primed Hybrid reasoning models isolates the architectural difference from differences in post-training data or hyper-parameters.\\

\noindent\textbf{Tasks.} We evaluate on a suite of reasoning-heavy benchmarks spanning math, science, coding, tool calling, and instruction following, grouped into five categories:
\begin{enumerate}
    \item \emph{Math.} (AIME24/25~\citep{aime2025}) Olympiad-qualifier problems from the American Invitational Mathematics Examination, requiring multi-step derivations over long chain-of-thought traces.
    \item \emph{Science.} (GPQA~\citep{rein2024gpqa}) Graduate-level multiple-choice questions in physics, chemistry, and biology, written to be Google-proof and requiring both deep domain knowledge and careful multi-step reasoning.
    \item \emph{Coding.} (LiveCodeBench v5~\citep{jain2025livecodebench}, SciCode~\citep{tian2024scicode}) Competitive-programming and scientific-computing problems evaluated by execution against held-out test cases.
    \item \emph{Tool Calling.} (BFCL v3/v4~\citep{patil2025bfcl}) Function-calling tasks that measure the ability to select the right tool and populate its arguments correctly across single- and multi-turn conversations.
    \item \emph{Instruction Following.} (IFBench~\citep{pyatkin2025generalizing}) Constrained generation tasks that test generalization to unseen, verifiable output constraints on format, length, and content.
\end{enumerate}

\noindent\textbf{Evaluation Setup.} All models are evaluated at 64K generation length. Our Primed GKA- and GDN-based Hybrid Reasoner models, as well as the fine-tuned Qwen3-\{8B,32B\}\,[Reasoner-SFT] baselines (trained with the same reasoning recipe, including context extension), support 128K context natively. The off-the-shelf Qwen3 baseline supports 32K context natively and up to 128K via YaRN~\citep{yarn}; we serve it at 128K using the YaRN settings prescribed in the model card. We evaluate on the tasks above using the NVIDIA NeMo Evaluator SDK\footnote{\url{https://docs.nvidia.com/nemo/evaluator/latest/}}, and provide the exact evaluation configuration\footnote{\url{https://github.com/awslabs/hybrid-model-factory/blob/main/examples/evaluation/nemo_reasoning_evals.yaml}} in our \href{https://github.com/awslabs/hybrid-model-factory}{awslabs/hybrid-model-factory} GitHub repository for full reproducibility.\footnote{BFCL v4 evaluation frequently triggers web-search API rate limits; to facilitate reproducibility we exclude the web-search subtask. For BFCL v4 we report average accuracy across queries, which is equivalent to weighting sub-task accuracy by the number of queries per sub-task.} To reduce variance, we run multiple trials with different random seeds for each task; the exact number of trials per task is specified in the provided configuration. The aggregate standard error across our evaluation tasks is around 0.5 points. %

\begin{table}[ht]
\centering
\caption{\textbf{Reasoning evaluation at 8B scale.} Performance of our Primed Hybrid Reasoner models (GKA and GDN) against the Qwen3-8B off-the-shelf source Transformer and the Qwen3-8B\,[Reasoner-SFT] baseline fine-tuned with the same reasoning recipe. All models are evaluated at 64K generation length with thinking mode on. The \textit{Avg.} column reports the aggregate performance across all eight benchmarks. $\Delta_{\text{GKA-GDN}}$\% reports the relative improvement (per task) of the GKA-based Reasoner model over the GDN-based counterpart; $\Delta_{\text{source}}$\% and $\Delta_{\text{SFT}}$\% report the relative gap (per task) of GKA-Primed-HQwen3-Reasoner against the Qwen3-8B and Qwen3-8B\,[Reasoner-SFT] baselines respectively. Our GKA-based Reasoner model matches or outperforms its GDN-based counterpart uniformly across all tasks.}
\label{tab:reasoning_8b}
\small
\resizebox{\textwidth}{!}{%
\begin{tabular}{lcccccccc|c}
\toprule
\textbf{Model} & \textbf{AIME24} & \textbf{AIME25} & \textbf{GPQA} & \textbf{LCB-v5} & \textbf{BFCL v4\textsuperscript{$\dagger$}} & \textbf{BFCL v3} & \textbf{IFBench} & \textbf{SciCode} & \textbf{Avg.} \\
\midrule
Qwen3-8B                          & 78.7 & 71.0 & 57.8 & 57.9 & 68.3 & 66.5 & 31.6 & 10.6 & 55.3 \\
Qwen3-8B\,[Reasoner-SFT]          & 80.7 & 73.7 & 62.6 & 65.1 & 67.3 & 65.8 & 39.7 &  4.1 & 57.4 \\
\midrule
GKA-Primed-HQwen3-Reasoner     & 82.0 & 73.7 & 61.8 & 63.1 & 66.5 & 62.2 & 39.0 &  6.4 & 56.8 \\
GDN-Primed-HQwen3-Reasoner     & 82.0 & 73.3 & 61.5 & 62.9 & 63.3 & 57.4 & 37.8 &  2.5 & 55.1 \\
\midrule
$\Delta_{\text{GKA-GDN}}\%$ (rel.) & $0.0$   & $+0.5$  & $+0.5$  & $+0.3$  & $+5.0$  & $+8.4$  & $+3.2$  & $+156.0$  & $+3.1$ \\
\midrule
$\Delta_{\text{source}}\%$ (rel.) & $+4.2$  & $+3.8$  & $+6.9$  & $+9.0$  & $-2.6$  & $-6.5$  & $+23.42$  & $-39.62$  & $+2.7$ \\
$\Delta_{\text{SFT}}\%$ (rel.) & $+1.6$  & $ 0.0$  & $-1.3$  & $-3.1$  & $-1.2$  & $-5.5$  & $-1.8$  & $+56.1$  & $-1.0$ \\
\bottomrule
\end{tabular}%
}

\smallskip
{\footnotesize 
\textsuperscript{$\dagger$}BFCL v4 excluding web-search category.\\
\textsuperscript{$\ddagger$}SciCode problem-level pass\,@1 (reported here) requires every subtask to be solved, so absolute scores are single-digit and relative $\Delta$s are sensitive to small gaps. Absolute point gaps: $\Delta_{\text{GKA-GDN}}=+3.9$, $\Delta_{\text{source}}= -4.2$, $\Delta_{\text{SFT}} = +2.3$. \cref{app:scicode-subtask} reports the subtask-level pass\,@1; rankings are preserved.}
\end{table}

\noindent\textbf{8B Models.} In \cref{tab:reasoning_8b} we report the reasoning performance of our Primed Hybrid Reasoner models at 8B scale. We make the following observations.
\begin{itemize}
    \item \textit{Primed Hybrid models deliver 1.8-2.3$\times$ higher decode throughput on long-reasoning tasks while being competitive with the Transformer baseline on reasoning benchmarks.} As in the IT setting (\cref{subsec: IT_results_long_ctx}), the roughly 50\% smaller KV cache of our Primed Hybrids translates to a 1.8-2.3$\times$ decode speedup over the Transformer. Against the Qwen3-8B\,[Reasoner-SFT] baseline, which is fine-tuned on the same reasoning SFT data with the same hyperparameters and thus isolates the architectural difference, the GKA-based Reasoner trails by only $\sim$1\% (relative) in aggregate and the GDN-based Reasoner by $\sim$4.0\% (relative).
    \item \textit{The GKA-based Reasoner matches or outperforms the GDN-based Reasoner uniformly across all tasks.} This is in contrast to the IT setting, where the two Hybrids performed on par at both long and short context (\cref{tab:instruct_short_8b,fig:it_long_8b}). Reasoning tasks are substantially more challenging: unlike IT workloads, which feature long context inputs but short responses, reasoning tasks require the model to generate long chains of thought before answering. GKA's more expressive state update better preserves intermediate computations across these extended thinking traces, which translates directly into higher accuracy on reasoning benchmarks.
    \item \textit{Our reasoning recipe improves both the Transformer and the Primed Hybrids over the off-the-shelf Qwen3-8B.} The Qwen3-8B\,[Reasoner-SFT] baseline and both Primed Hybrid Reasoners outperform the off-the-shelf Qwen3-8B in aggregate, confirming the efficacy of our reasoning training recipe (\cref{subsec:long-context-reasoning}).
\end{itemize}

\begin{table}[ht]
\centering
\caption{\textbf{Reasoning evaluation at 32B scale.} Performance of our Primed Hybrid Reasoner model (GKA) against the Qwen3-32B off-the-shelf source Transformer and the Qwen3-32B\,[Reasoner-SFT] baseline fine-tuned with the same reasoning recipe. All models are evaluated at 64K generation length with thinking mode on. The \textit{Avg.} column reports the aggregate performance across all eight benchmarks. $\Delta_{\text{source}}$\% and $\Delta_{\text{SFT}}$\% report the relative gap (per task) of GKA-Primed-HQwen3-Reasoner against the Qwen3-32B and Qwen3-32B\,[Reasoner-SFT] baselines respectively.}
\label{tab:reasoning_32b}
\small
\resizebox{\textwidth}{!}{%
\begin{tabular}{lcccccccc|c}
\toprule
\textbf{Model} & \textbf{AIME24} & \textbf{AIME25} & \textbf{GPQA} & \textbf{LCB-v5} & \textbf{BFCL v4\textsuperscript{$\dagger$}} & \textbf{BFCL v3} & \textbf{IFBench} & \textbf{SciCode} & \textbf{Avg.} \\
\midrule
Qwen3-32B                          & 86.3 & 70.0 & 65.4 & 64.4 & 69.3 & 69.6 & 32.6 & 15.9 & 59.2 \\
Qwen3-32B\,[Reasoner-SFT]          & 88.3 & 82.3 & 68.9 & 70.8 & 70.7 & 69.3 & 49.0 & 13.1 & 64.1 \\
\midrule
GKA-Primed-HQwen3-Reasoner     & 87.7 & 81.7 & 67.3 & 70.2 & 70.1 & 66.3 & 48.2 & 12.3 & 63.0 \\
\midrule
$\Delta_{\text{source}}$\% (rel.)   & $+1.6$ & $+16.7$ & $+2.9$ & $+9.0$ & $+1.2$ & $-4.6$ & $+47.9$ & $-22.6$ & $+6.4$ \\
$\Delta_{\text{SFT}}$\%  (rel.)   & $-0.8$ & $-0.8$  & $-2.3$ & $-0.8$ & $-0.8$ & $-4.2$ & $-1.6$  & $-6.0$  & $-1.7$ \\
\bottomrule
\end{tabular}%
}

\smallskip
{\footnotesize
\textsuperscript{$\dagger$}BFCL v4 excluding web-search category.\\}
\end{table}

\noindent\textbf{32B Models.} Following our 8B finding that the GKA-based Reasoner outperforms the GDN-based Reasoner uniformly across all tasks, we scale only the GKA variant to 32B. In \cref{tab:reasoning_32b} we report the reasoning performance of our GKA-Primed-HQwen3-Reasoner at 32B scale. The trends mirror the 8B results. Our GKA-Primed-HQwen3-Reasoner delivers a 1.3-2.1$\times$ decode speedup in long-reasoning tasks over the Qwen3-32B\,[Reasoner-SFT] Transformer baseline while being competitive in quality, trailing by only $\sim$1.7\% (relative) in aggregate despite half of its sequence-mixing layers being SSMs. As discussed in the 32B IT setting (\cref{subsec: IT_results_long_ctx}), due to the specific Qwen3-32B architecture the speedups at 32B are smaller than those at 8B because a larger fraction of decode time is spent on non-Attention operations (e.g., MLP layers) at this scale. Our reasoning recipe further improves both the fine-tuned Transformer and the Primed Hybrid over the off-the-shelf Qwen3-32B by a large margin ($+6.4\%$ relative in aggregate for GKA). Together with the 8B results, these numbers demonstrate that Primed Hybrid models are amenable to large-scale post-training, matching fine-tuned Transformers on quality while retaining the inference efficiency benefits of a compact recurrent state.

\subsection{Priming Pipeline Ablation}
\label{sec:priming_ablations}

The Priming pipeline (\cref{subsec: priming pipeline}) is built around three key design choices, each motivated by realization theory: importance-based layer selection to select which retain the Attention layers with highest \ref{eq:max_hankel_rank}, State Expansion via AGQA to restore the SSM's effective state capacity lost due to the grouped KV structure inherited from the source Transformer, and layerwise alignment objective used in Stage~1 for aligning the SSM's finite-horizon input-output matrix with the mixing matrix of the corresponding Attention layer from the source Transformer. Specifically, we study: (i) whether importance-based layer selection outperforms uniform layer pattern (\cref{subsec:pattern_ablations}); (ii) whether end-to-end MSE outperforms layerwise MSE for Stage~1 alignment (\cref{subsubsec:e2e-vs-lw_ablations}); and (iii) whether AGQA outperforms standard GQA replication for State Expansion (\cref{subsubsec:agqa_vs_gqa_ablations}).

\subsubsection{Layer Selection}\label{subsec:pattern_ablations}
\paragraph{Importance Score Analysis.}

\begin{figure}[h]
  \centering
  \includegraphics[width=\textwidth]{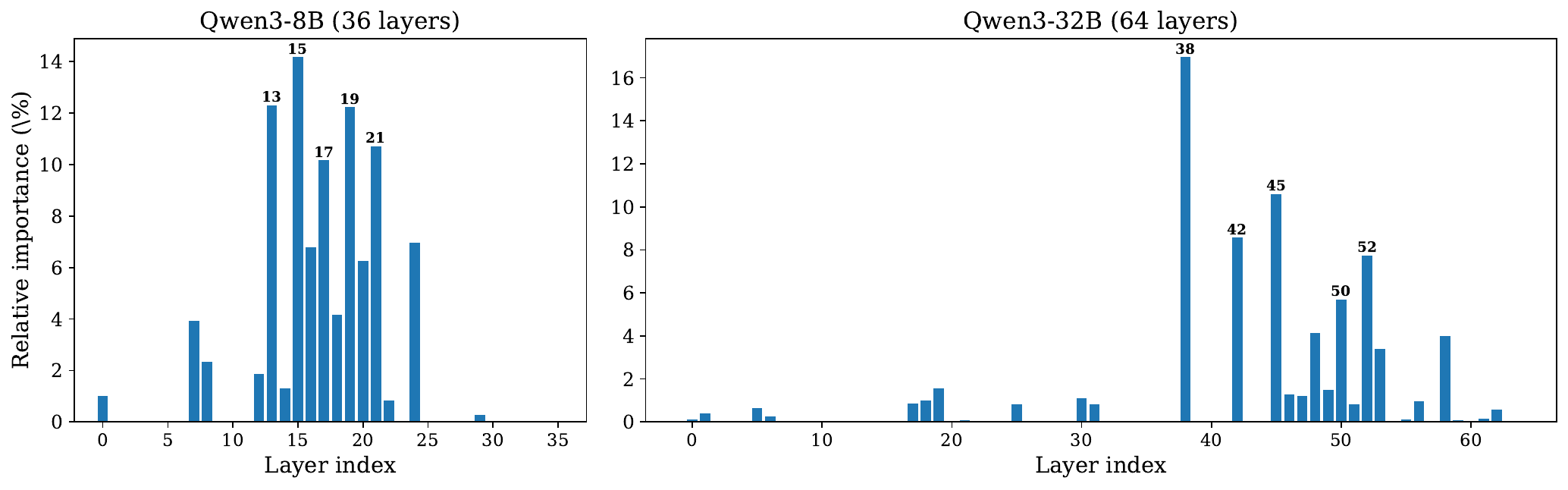}
   \caption{\textbf{Per-layer importance scores for Qwen3-8B and Qwen3-32B.} Each bar is the relative drop in mean HELMET performance (across five sub-tasks: Synthetic Recall, RAG, Many-shot ICL, Generation with Citations, and Passage Re-ranking) when the corresponding Attention layer is individually replaced with SWA of window size $w = 2048$. Importance is concentrated in a small set of middle-to-late layers in both models; the five most important layer indices are labeled.}
 
  \label{fig:importance_scores}
\end{figure}

We apply the importance-based layer selection procedure introduced in \cref{subsec:layer_selection} to Qwen3-8B (36 layers) and Qwen3-32B (64 layers), the source Transformers for our Primed Hybrid models at 8B and 32B scale respectively. We consider 5 sub-tasks from HELMET \citep{helmet} as our target benchmark (Synthetic Recall, Retrieval-Augmented Generation, Many-shot In-Context Learning, Generation with citations, and Passage re-ranking) all evaluated at 32K context length. 

Figure~\ref{fig:importance_scores} shows the resulting per-layer \emph{importance} for both the Transformer models. We find that \emph{importance} varies substantially across layers and does not follow a regular pattern across models. For example, for Qwen3-8B, the highly important layers cluster in the middle of the network (layer indices 13 to 21, which is roughly 36-58\% of model depth); on the other hand, for Qwen3-32B, the most important layers are towards the later layers (layer indices 38-52, which is roughly 59-81\% of the model depth). Recall that the \emph{importance} of a layer is a practical proxy for identifying layers with low \ref{eq:max_hankel_rank} and thus prime candidates for realization via SSM layers. Consistently, across both models we find the early layers to have low (almost zero) importance scores. This aligns with the intuition that early-layer representations emphasize local structure (token identity, syntax, positional relationships) that does not require the long-range retrieval capabilities of an eidetic memory. We release the importance scores identified by our layer-selection algorithm for Qwen3-8B\footnote{\url{https://github.com/awslabs/hybrid-model-factory/blob/main/vllm-inference/selection/importance_scores_qwen3-8b.csv}} and Qwen3-32B\footnote{\url{https://github.com/awslabs/hybrid-model-factory/blob/main/vllm-inference/selection/importance_scores_qwen3-32b.csv}}.

\paragraph{Selective vs.\ Uniform (Layer) Pattern.}
\begin{figure}[h]
  \centering
  \includegraphics[width=0.9\textwidth]{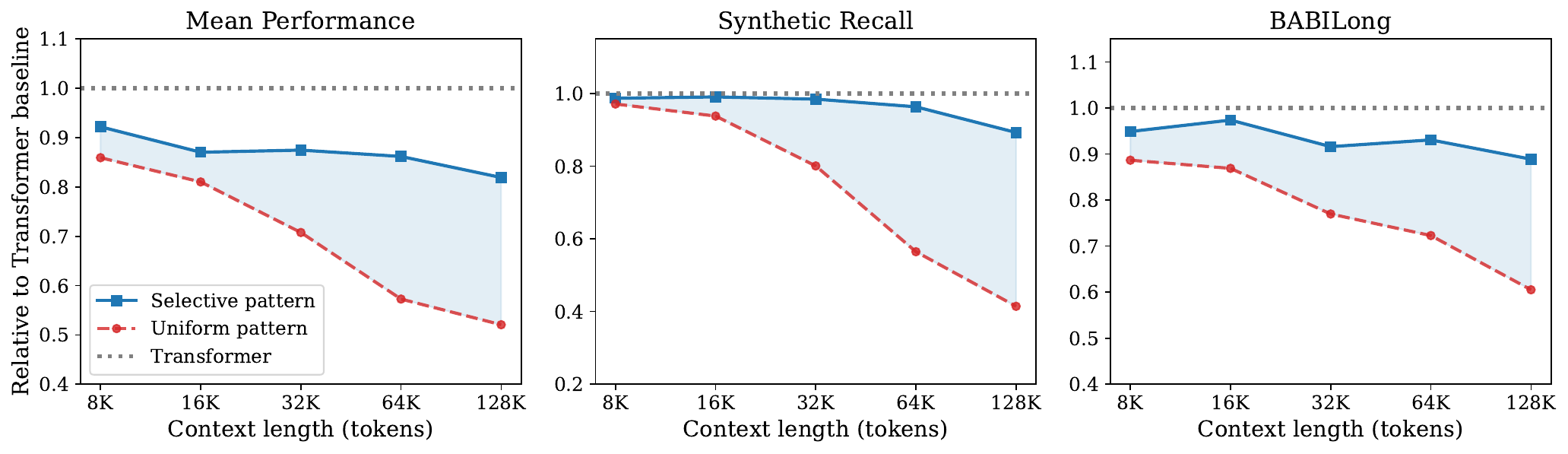}
  \caption{\textbf{Effect of layer pattern strategy on long-context performance for Hybrid models with a 50\% Hybrid ratio sourced from Qwen3-8B.} We compare our Selective Pattern against the Uniform Pattern baseline. Each line averages across three Hybrid models, one per SSM layer type (Mamba2, GDN, GKA); scores are reported as a fraction of the Transformer baseline's score at the same context length. All models, including the Transformer baseline, are trained with an identical recipe.  \textbf{Left:} Mean across long-context tasks.
    \textbf{Middle:} Synthetic Recall, which requires verbatim retrieval
    of strings from the input context.
    \textbf{Right:} BABILong,
    which requires reasoning over facts hidden within large amounts of irrelevant filler context. The gap between Selective and Uniform Pattern widens with context
    length, confirming that layer selection is critical for Priming Hybrid models with strong long-context capabilities.}
  \label{fig:layer_selection}
\end{figure}

Here, we evaluate the effectiveness of our layer-selection algorithm for building strong Hybrid models via Priming. We refer to our method as \textit{Selective Pattern} and compare it against a baseline, \textit{Uniform Pattern}, which designates 1 out of every $\frac{1}{1-p}$ layers as Attention for a Hybrid ratio $p$. 

For our experiment, we Prime three Hybrid models, one per SSM layer type (Mamba2, GDN, GKA), each with a 50\% Hybrid ratio and sourced from Qwen3-8B. Figure~\ref{fig:layer_selection} compares Selective Pattern vs.\ Uniform Pattern for these Primed Hybrid models, with scores averaged over the three SSM layer types on a suite of long-context tasks (described in \cref{subsec: IT_results_long_ctx}). All models, including the Transformer baseline (Qwen3-8B\,[Long], described in \cref{subsec: IT_results_long_ctx}), are trained with the same recipe. All scores reported are relative to the Transformer baseline.

At short context (8K tokens), the Selective Pattern leads the Uniform Pattern by roughly 6 points in mean long-context performance. The gap widens with context length: at 128K, the Selective Pattern obtains 82\% of the baseline Transformer's aggregate score, whereas the Uniform Pattern obtains only 52\%. The effect is most pronounced on Synthetic Recall, a task that requires the model to locate and reproduce verbatim strings from the input, exercising precisely the long-horizon retrieval capability that eidetic Attention layers provide: at 128K, the Selective Pattern obtains 89\% of the baseline Transformer's recall score while the Uniform Pattern obtains only 41\%.

\begin{table}[h]
\centering
\caption{\textbf{Effect of layer placement strategy on short-context performance for Hybrid models with a 50\% Hybrid ratio sourced from Qwen3-8B.} We compare our Selective Pattern against the Uniform Pattern baseline on the Tulu3-Dev benchmark suite. Each Primed Hybrid row averages across three Hybrid models, one per SSM layer type (Mamba2, GDN, GKA). The \textit{Mean} column reports the aggregate performance across all five categories. All models, including the Transformer baseline, are trained with an identical recipe. The Selective Pattern outperforms the Uniform Pattern in aggregate, with the largest gaps on Math and Reasoning, confirming that layer selection benefits Primed Hybrid models beyond long-context tasks.}

\label{tab:tulu3dev_pattern}
\small
\resizebox{0.9\textwidth}{!}{
\begin{tabular}{llcccccc}
\toprule
\textbf{Model} & \textbf{Layer Pattern} & \textbf{Math} & \textbf{Knowledge} & \textbf{Coding} & \textbf{Reasoning} & \textbf{IF} & \textbf{Mean} \\
\midrule
Source Transformer & --- & 69.6 & 49.3 & 91.2 & 68.1 & 77.8 & 68.4 \\
Primed-HQwen3-IT & Selective & 65.1 & 47.2 & 90.3 & 62.6 & 74.9 & 65.2 \\
Primed-HQwen3-IT & Uniform   & 53.5 & 46.6 & 89.8 & 52.0 & 75.2 & 60.6 \\
\bottomrule
\end{tabular}
}
\end{table}

The importance scores are computed using only a subset of long-context tasks and at 32K context length, yet the benefit of Selective Pattern over Uniform Pattern generalizes well beyond this scoring suite. First, as shown in \cref{fig:layer_selection}, the Selective Pattern's advantage persists at context lengths up to 128K. Second, it also extends to long-context tasks not used in the importance-scoring procedure (for example, BABILong in \cref{fig:layer_selection}). Finally, we compare the two layer patterns on short-context benchmarks (described in \cref{subsec: IT_results_short_ctx}) and report results in \cref{tab:tulu3dev_pattern}: in aggregate, the Selective Pattern obtains 7.6\% higher score relative to the Uniform Pattern.

\begin{wrapfigure}{r}{0.5\textwidth}
\vspace{-2mm}
\centering
\includegraphics[width=0.5\textwidth]{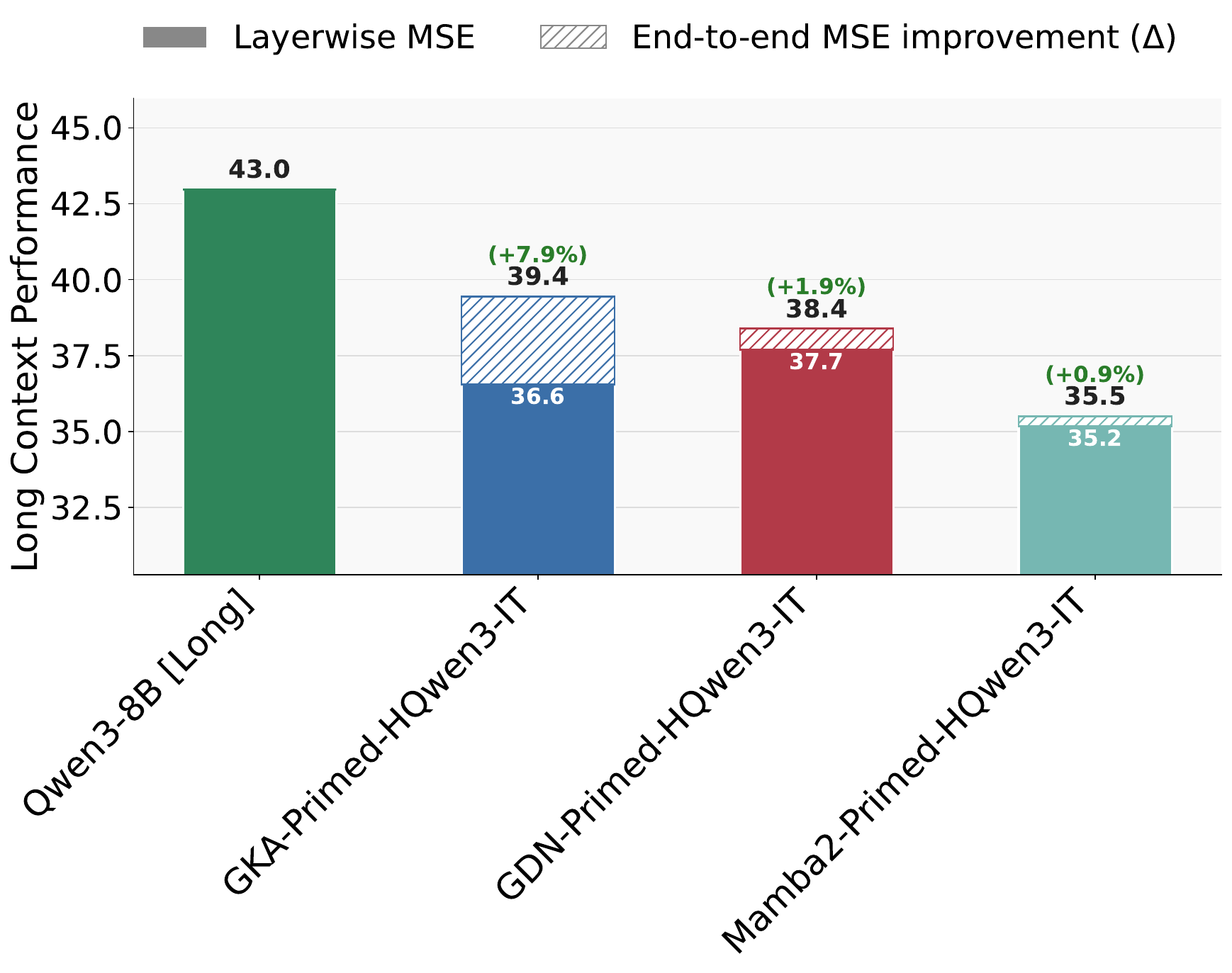}
\caption{\textbf{Effect of Stage~1 supervision granularity on long-context performance for Hybrid models with a 50\% Hybrid ratio sourced from Qwen3-8B.} We compare end-to-end MSE (\cref{eq:e2e-loss}) against layerwise MSE (\cref{eq:lw-loss}) for three SSM layer types (GKA, GDN, Mamba2). For each SSM type, the solid bar shows performance under layerwise supervision and the hatched overlay shows the improvement from end-to-end supervision. Qwen3-8B\,[Long] is the Transformer baseline. Scores are reported as a weighted average over long-context tasks across context lengths from 8K to 128K, using the same geometrically increasing context-length weighting as in \cref{subsec: IT_results_long_ctx}. End-to-end MSE outperforms layerwise MSE for all three SSM types, with the relative improvement increasing with the expressivity of the SSM layer.}
\label{fig:layerwise_vs_e2e}
\vspace{-15mm}
\end{wrapfigure}

These results align with the realization-theory motivation of our layer-selection algorithm (\cref{subsec:layer_selection}): high-Hankel-rank Attention layers in the source Transformer are best retained as Attention in the Hybrid, not realized with a bounded-state SSM.

\paragraph{Remark:} The results in this ablation are from a controlled recipe designed to isolate the effect of layer-pattern choice. The final Primed models released with this work incorporate importance-based layer selection alongside our full training recipe (\cref{sec:training_recipes}); see \cref{{sec:experiments}} for their performance.

\subsubsection{Stage 1: End-to-end vs.\ Layerwise MSE}\label{subsubsec:e2e-vs-lw_ablations}
In this subsection, we ablate the Stage~1 supervision granularity for Primed Hybrid models. We compare \textit{End-to-end MSE} \cref{eq:e2e-loss}, which supervises the Hybrid only at the final hidden states before the LM head, against \textit{Layerwise MSE} \cref{eq:lw-loss}, which averages a per-layer MSE across all decoder layers, forcing each SSM layer to individually approximate its Attention counterpart at the same depth. For our experiment, we Prime three Hybrid models, one per SSM layer type (Mamba2, GDN, GKA), each with a 50\% Hybrid ratio and sourced from Qwen3-8B. Each Hybrid is first trained via Stage~1 at 8K context using one of the two objectives, then undergoes long-context continued training at 128K followed by supervised fine-tuning with intruction-following data. \cref{fig:layerwise_vs_e2e} reports the weighted average score across our long-context benchmarks (described in \cref{subsec: IT_results_long_ctx}) over context lengths from 8K to 128K, using the same weighting scheme as in the main results.

End-to-end MSE outperforms Layerwise MSE for all three SSM types, with the relative improvement increasing with the expressivity of the SSM layer (\cref{subsec:ssm-zoo}): 0.9\% for Mamba2, 1.9\% for GDN, and 7.9\% for GKA. We hypothesize that layerwise supervision over-constrains the SSM layers: each must match an Attention layer that has access to the full context using only a fixed-size compressed state. End-to-end supervision avoids this by constraining only the final output during Stage~1, giving the SSM layers more freedom to develop internal representations suited to their recurrent dynamics.

Based on these results, we adopt end-to-end MSE as the default Stage~1 objective for all Hybrid models in this work. This also simplifies the Priming pipeline relative to progressive distillation methods such as MOHAWK~\citep{mohawk}, which rely on layerwise alignment as an intermediate stage before end-to-end training. In Priming, Stage~0 transfers the Q/K/V/O projections directly from the source Attention layer (\cref{subsec:stage0-theory}), giving each SSM layer a well-aligned starting point; the layerwise alignment stage can then be skipped without degrading end-to-end performance.

\paragraph{Remark:} The results in this ablation are from a controlled recipe designed to isolate the effect of Stage~1 supervision granularity. The final Primed models released with this work use our full training recipe (\cref{sec:training_recipes}).

\subsubsection{State Expansion: Adaptive GQA vs.~GQA}\label{subsubsec:agqa_vs_gqa_ablations}
\begin{wrapfigure}{r}{0.45\textwidth}
\vspace{-5mm}
\centering
\includegraphics[width=0.45\textwidth]{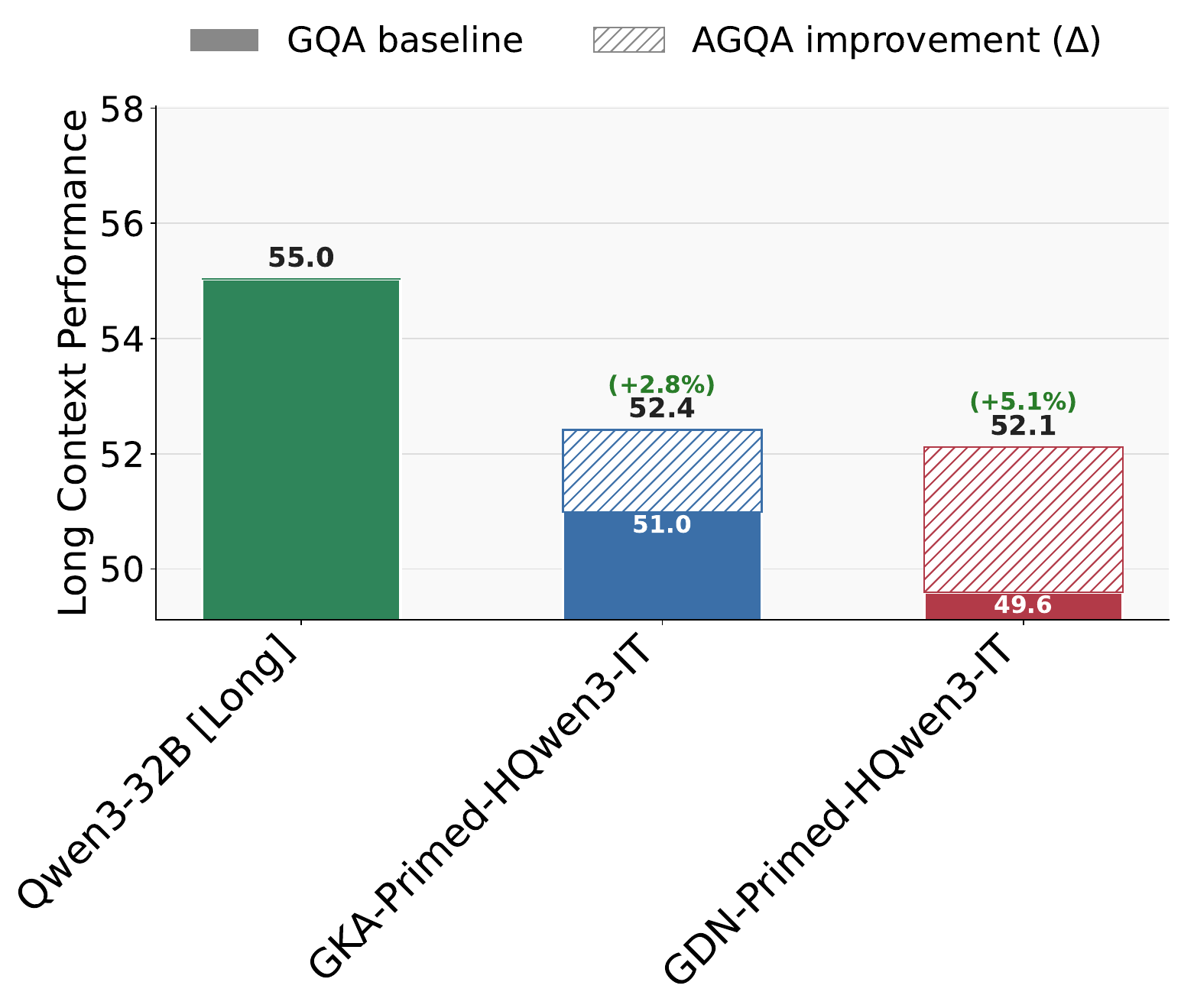}
\caption{\textbf{Effect of State Expansion via AGQA on long-context performance for Hybrid models with a 50\% Hybrid ratio sourced from Qwen3-32B.} We compare Adaptive GQA (AGQA) against the standard GQA baseline for two SSM layer types (GKA, GDN). For each SSM type, the solid bar shows performance under standard GQA and the hatched overlay shows the improvement from AGQA. Qwen3-32B\,[Long] is the Transformer baseline. Scores are reported as a weighted average over long-context tasks across context lengths from 8K to 128K, using the same geometrically increasing context-length weighting as in \cref{subsec: IT_results_long_ctx}. AGQA improves over standard GQA for both SSM types.}
\label{fig:agqa}
\vspace{-1mm}
\end{wrapfigure}

In this subsection, we ablate the State Expansion design choice for Primed Hybrid models. We compare our \textit{Adaptive GQA (AGQA)} (\cref{eq:agqa}) against a standard \textit{GQA} baseline, which replicates each of the $H_{\mathrm{KV}}$ shared KV heads $m = H_Q / H_{\mathrm{KV}}$ times to produce $H_Q$ heads. As discussed in \cref{subsec:stateexp-theory}, this repetition reduces the SSM's effective state capacity relative to having $m$ independent key-value heads. AGQA replaces the fixed repetition with a learned low-rank projection that maps the $H_{\mathrm{KV}}$ shared heads to $H_Q$ potentially distinct ones, recovering per-head diversity with minimal overhead.

For our experiment, we Prime two Hybrid models, one per SSM layer type (GKA, GDN), each with a 50\% Hybrid ratio and sourced from Qwen3-32B (GQA group size $m = 8$). We set the rank of the AGQA projection to 512, adding less than $1\%$ of total model parameters at this scale. Figure~\ref{fig:agqa} reports the weighted average score on our long-context benchmarks, using the same geometrically increasing context-length weighting as in \cref{subsec: IT_results_long_ctx}. AGQA improves over standard GQA for both Hybrid models in aggregate long-context performance (Figure~\ref{fig:agqa}), as predicted by the capacity analysis in \cref{subsec:stateexp-theory}.

At 8B scale, the GQA group size for our source Transformer (Qwen3-8B) is $m = 4$, half the 32B value. The corresponding capacity reduction from head sharing is smaller, and we find standard GQA to be sufficient at this scale; consequently, we use AGQA only for our 32B Hybrid models. AGQA does introduce a modest inference cost: when served with tensor parallelism, the additional all-gather for the low-rank projection tensors makes our 32B Hybrid models up to 5\% slower during decoding compared to GQA.

\section{Test-Time Scaling: Intelligence per FLOP} 
\label{subsec:tts}

Modern reasoning models spend substantial inference compute per query: they generate long chains-of-thought, and accuracy on hard benchmarks can be further improved by sampling multiple trajectories and aggregating them (majority voting, best-of-$n$). The relevant efficiency question is therefore not how fast a single token is produced, but how quickly a model reaches a target accuracy on a given problem.

Our Primed Hybrids decode faster than their source Transformers at long context (\cref{sec:inference}), thanks to a smaller KV cache and a compact recurrent state. This per-token speedup does not necessarily translate to end-to-end efficiency: a model that is cheaper per token can still be slower to reach a target accuracy if its reasoning is weaker and it needs many more tokens, or many more parallel samples, to compensate. In this section we ask: \emph{in wall-clock terms, how much faster does a Primed Hybrid reach a target accuracy compared to the source model it was Primed from?} This matters most for reinforcement learning pipelines, where training throughput is gated by how quickly the policy can generate \emph{correct} trajectories to learn from~\citep{areal,verl}. Because Primed Hybrids require less KV-cache memory per sequence, more rollouts fit concurrently on the same hardware (\cref{sec:inference}). On AIME~2025, our 32B GKA-Primed-HQwen3-Reasoner matches the accuracy of its Transformer baseline (Qwen3-32B trained on the same reasoning recipe) with up to a 1.6$\times$ wall-clock speedup.

\subsection{Evaluation Setup}
We compare our 32B GKA-Primed-HQwen3-Reasoner (referred to as ``Hybrid'' in this section) against the Qwen3-32B Transformer baseline trained on the same reasoning recipe (referred to as ``Transformer'' in this section). Both models are served on 8$\times$H200 GPUs (TP\,=\,8) via vLLM. For each concurrency level we sweep \texttt{max-num-seqs} (the maximum number of concurrent requests processed by vLLM) independently per model and report the fastest configuration.

\paragraph{Problem sets.} We evaluate on two subsets of AIME~2025: (i)~the full 30-problem benchmark and (ii)~a \emph{hard subset} comprising the 15 problems on which Qwen3-32B attains the lowest accuracy (pass@1). The hard subset elicits longer thinking traces, which amplify the Hybrid's decode throughput advantage.

\paragraph{Concurrency sweep.} We control total concurrency by varying the number of rollouts per problem across five levels: 480, 720, 960, 1200, and 1440 total rollouts. On the full set this corresponds to 16--48 rollouts per problem; on the hard subset we double the rollouts per problem (32--96) so both sets hit the same total concurrency.

\paragraph{Test-time compute for GKA.} GKA exposes a runtime compute knob via its Chebyshev iteration count $r$: a single trained model supports variable inference compute by adjusting $r$ post-hoc (see \cref{subsec:chebyshev}). Our GKA-Primed-HQwen3-Reasoner is trained with $r{=}30$. At inference time we evaluate two operating points: $r{=}30$ and $r{=}10$ (reduced compute, with a small accuracy cost). The Transformer baseline has no comparable knob.

\paragraph{KV-cache saturation.} We say a model is \emph{KV-cache saturated} when the GPU memory budget for its KV cache is full: at that point vLLM can no longer run all requests concurrently and some must wait or be paused and resumed later. Below saturation, additional rollouts fit in memory and run in parallel; above saturation, they extend wall-clock time. On the full set, neither model is saturated at 480 rollouts; at 720, only the Transformer saturates; at 960 and above, both models saturate. The Hybrid's smaller KV cache per sequence pushes its saturation point to higher concurrency than the Transformer's, which is the mechanism behind the Hybrid's wall-clock advantage.

\paragraph{Time-to-accuracy metric.} We measure how quickly each model produces a target number of correct rollouts. At wall-clock time $t$, running accuracy is defined as the fraction of correct rollouts completed so far, divided by the \emph{total} number of rollouts submitted at that concurrency level (e.g., 480). For each accuracy threshold $p$, let $t_{\text{Hybrid}}(p)$ and $t_{\text{Transformer}}(p)$ denote the wall-clock times at which each model's running accuracy first crosses $p$. We report the speedup
\[
\text{S-}p = \frac{t_{\text{Transformer}}(p)}{t_{\text{Hybrid}}(p)},
\]
where values $>1$ indicate the Hybrid reaches the threshold first. ``S-Completion'' is the speedup to finish all rollouts. Fixing the denominator to total submitted rollouts (rather than completed-so-far) ensures both models are compared at the same absolute count of correct answers.

\subsection{AIME 2025 (Full)}
\cref{tab:tts-speedup} reports the Hybrid's speedup over the Transformer on the full 30-problem AIME~2025 set across five concurrency levels. At low concurrency (480 and 720 rollouts), the Transformer is faster at the lowest accuracy threshold (S-60\%), while the Hybrid is faster at higher thresholds. The reason traces back to the per-token cost profile of the two models: at short context the Transformer's decode throughput exceeds the Hybrid's, since MLP weight reads dominate over KV-cache traffic and the SSM state is larger than the KV cache of short sequences (\cref{sec:inference}). Early in the run, the rollouts that finish first and contribute to low-threshold accuracy are the ones with short correct traces, produced in the regime that favors the Transformer. As the target threshold rises, later-completing long-trace rollouts start to matter, and the Hybrid's advantage at long context takes over.

At saturation (960+ rollouts), the Hybrid is consistently faster across all thresholds. The $r{=}10$ variant achieves $1.22$--$1.47\times$ speedup at every threshold and concurrency level, with the advantage growing at higher load ($1.34$--$1.47\times$ at 1440 rollouts). The $r{=}30$ variant sees slightly smaller but still consistent gains ($1.13$--$1.41\times$). As more rollouts pile on, the Transformer's KV-cache budget throttles its throughput more severely than the Hybrid's compact state, which extends the Hybrid's wall-clock lead across thresholds. 

Notably, the $r{=}10$ variant is uniformly faster than $r{=}30$ at the same accuracy threshold, with only a ~1\% drop in final accuracy (82.2\% vs 83.2\%). Both rows come from the same trained model: $r$ is reduced post-hoc at inference time (\cref{subsec:chebyshev}), so the additional speedup requires no retraining.

\begin{table}[h]
\centering
\caption{\textbf{Speedup of GKA-Primed-HQwen3-Reasoner vs.~Transformer on AIME~2025 (full set).} Speedup of the Hybrid over Qwen3-32B \,[Reasoner-SFT] (the Qwen3-32B model finetuned on the same reasoning data as the Hybrid), defined as the ratio of Transformer wall-clock time to Hybrid wall-clock time to reach each accuracy threshold $p$ (S-$p$) on  AIME~2025. Values $>1$ indicate the Hybrid is faster. Below KV-cache saturation (480 and 720 rollouts), the Transformer is faster at S-60\% because short correct traces complete first in the regime that favors its per-token throughput; the Hybrid overtakes at higher thresholds as long-trace rollouts start to matter. At saturation (960+ rollouts), the Hybrid is consistently faster across all thresholds. Average final accuracy across concurrency levels: Qwen3-32B\,[Reasoner-SFT] $83.3 \pm 0.7\%$, GKA-Primed-HQwen3-Reasoner ($r{=}30$) $83.2 \pm 0.6\%$, GKA-Primed-HQwen3-Reasoner ($r{=}10$) $82.2 \pm 0.2\%$.}
\label{tab:tts-speedup}
\small
\begin{tabular}{@{}llcccc@{}}
\toprule
Model & Rollouts & S-60\% $\uparrow$ & S-70\% $\uparrow$ & S-80\% $\uparrow$ & S-Completion $\uparrow$ \\
\midrule
\multirow{5}{*}{GKA-Primed-HQwen3-Reasoner ($r{=}10$)}
& 480  & 0.96$\times$ & 1.00$\times$ & 1.19$\times$ & 1.19$\times$ \\
& 720  & 0.87$\times$ & 1.11$\times$ & 1.36$\times$ & 1.34$\times$ \\
& 960  & 1.38$\times$ & 1.42$\times$ & 1.22$\times$ & 1.34$\times$ \\
& 1200 & 1.30$\times$ & 1.25$\times$ & 1.43$\times$ & 1.39$\times$ \\
& 1440 & 1.34$\times$ & 1.47$\times$ & 1.39$\times$ & 1.41$\times$ \\
\midrule
\multirow{5}{*}{GKA-Primed-HQwen3-Reasoner ($r{=}30$)}
& 480  & 0.92$\times$ & 0.96$\times$ & 1.17$\times$ & 1.13$\times$ \\
& 720  & 0.85$\times$ & 1.13$\times$ & 1.39$\times$ & 1.29$\times$ \\
& 960  & 1.30$\times$ & 1.24$\times$ & 1.13$\times$ & 1.22$\times$ \\
& 1200 & 1.16$\times$ & 1.16$\times$ & 1.38$\times$ & 1.31$\times$ \\
& 1440 & 1.26$\times$ & 1.41$\times$ & 1.37$\times$ & 1.33$\times$ \\
\bottomrule
\end{tabular}
\end{table}

\subsection{AIME 2025 (Hard Subset)}
On the 15 hardest AIME~2025 problems (\cref{tab:tts-speedup-hard}), the Hybrid is faster than the Transformer across \emph{every} threshold and concurrency level tested, with speedups ranging from $1.12\times$ to $1.59\times$. The regime at low concurrency where the Transformer was faster on the full set disappears here: hard problems require uniformly long thinking traces, so there are no short correct trajectories to complete early at short context, and the Hybrid's long-context throughput advantage applies from the start. The longer traces also keep each rollout in the KV cache for longer, sustaining KV-cache pressure across more of the run and extending the window over which the Hybrid's smaller per-sequence state pays off. The $r{=}10$ variant reaches up to $1.59\times$ at 1440 rollouts and sustains $1.36$--$1.52\times$ even at the lowest concurrency (480), with a $\sim$1\% accuracy cost relative to $r{=}30$ (66.3\% vs 67.2\%) and again with no retraining (\cref{subsec:chebyshev}).\\

\begin{table}[h]
\centering
\caption{\textbf{Speedup of GKA-Primed-HQwen3-Reasoner vs.~Transformer on the 15 hardest AIME~2025 problems.} Speedup of the Hybrid over Qwen3-32B\,[Reasoner-SFT] (the Qwen3-32B model finetuned on the same reasoning data as the Hybrid) on the hard subset (15 problems with lowest Qwen3-32B pass@1), defined as the ratio of Transformer wall-clock time to Hybrid wall-clock time to reach each accuracy threshold $p$ (S-$p$). Values $>1$ indicate the Hybrid is faster. Hard problems require uniformly long thinking traces, amplifying the Hybrid's decode throughput advantage, so the Transformer-favorable low-threshold regime seen on the full set does not appear. Average final accuracy across concurrency levels: Qwen3-32B\,[Reasoner-SFT] $68.1 \pm 0.9\%$, GKA-Primed-HQwen3-Reasoner ($r{=}30$) $67.2 \pm 2.0\%$, GKA-Primed-HQwen3-Reasoner ($r{=}10$) $66.3 \pm 0.9\%$.}
\label{tab:tts-speedup-hard}
\small
\begin{tabular}{@{}llcccc@{}}
\toprule
Model & Rollouts & S-55\% $\uparrow$ & S-60\% $\uparrow$ & S-65\% $\uparrow$ & S-Completion $\uparrow$ \\
\midrule
\multirow{5}{*}{GKA-Primed-HQwen3-Reasoner ($r{=}10$)}
& 480  & 1.52$\times$ & 1.45$\times$ & 1.36$\times$ & 1.39$\times$ \\
& 720  & 1.20$\times$ & 1.19$\times$ & 1.12$\times$ & 1.24$\times$ \\
& 960  & 1.25$\times$ & 1.33$\times$ & 1.24$\times$ & 1.31$\times$ \\
& 1200 & 1.42$\times$ & 1.40$\times$ & 1.44$\times$ & 1.42$\times$ \\
& 1440 & 1.43$\times$ & 1.41$\times$ & 1.43$\times$ & 1.59$\times$ \\
\midrule
\multirow{5}{*}{GKA-Primed-HQwen3-Reasoner ($r{=}30$)}
& 480  & 1.38$\times$ & 1.36$\times$ & 1.39$\times$ & 1.28$\times$ \\
& 720  & 1.31$\times$ & 1.31$\times$ & 1.24$\times$ & 1.28$\times$ \\
& 960  & 1.21$\times$ & 1.20$\times$ & 1.14$\times$ & 1.20$\times$ \\
& 1200 & 1.32$\times$ & 1.32$\times$ & 1.35$\times$ & 1.34$\times$ \\
& 1440 & 1.30$\times$ & 1.27$\times$ & 1.29$\times$ & 1.47$\times$ \\
\bottomrule
\end{tabular}
\end{table}

\textbf{Takeaway.} The reduced per-sequence memory footprint of Primed Hybrid models translates into faster time-to-accuracy under realistic concurrent workloads. One particularly impactful application is RL post-training, where training throughput is bottlenecked by the rate at which the policy generates correct trajectories~\citep{areal,verl}. This is an architecture-specific advantage: any Primed Hybrid (GKA, GDN, Mamba2) benefits from this KV-cache reduction and thus the same
ability to serve more concurrent rollouts.  Concretely, our GKA-Primed-HQwen3-Reasoner ($r=10$) completes all rollouts 1.19--1.41$\times$ faster on the full set and 1.24--1.59$\times$ faster on the hard subset than the Transformer baseline, effectively increasing the RL training data throughput at fixed GPU budget.

\subsection{Token Length Distributions}
A natural question is whether the wall-clock speedups reported above are an artifact of the Hybrid producing shorter responses. \cref{fig:tts-token-hist} rules this out: the output token length distributions of GKA-Primed-HQwen3-Reasoner and Qwen3-32B\,[Reasoner-SFT] are closely matched on both the full AIME~2025 set and the hard subset, with the Hybrid producing slightly \emph{longer} traces on average (median 14.3K vs.\ 13.1K tokens on the full set; 27.9K vs.\ 26.4K on the hard subset). The speedups therefore stem from the Hybrid's higher decode throughput, not from shorter generations.

\begin{figure}[h]
\centering
\includegraphics[width=\textwidth]{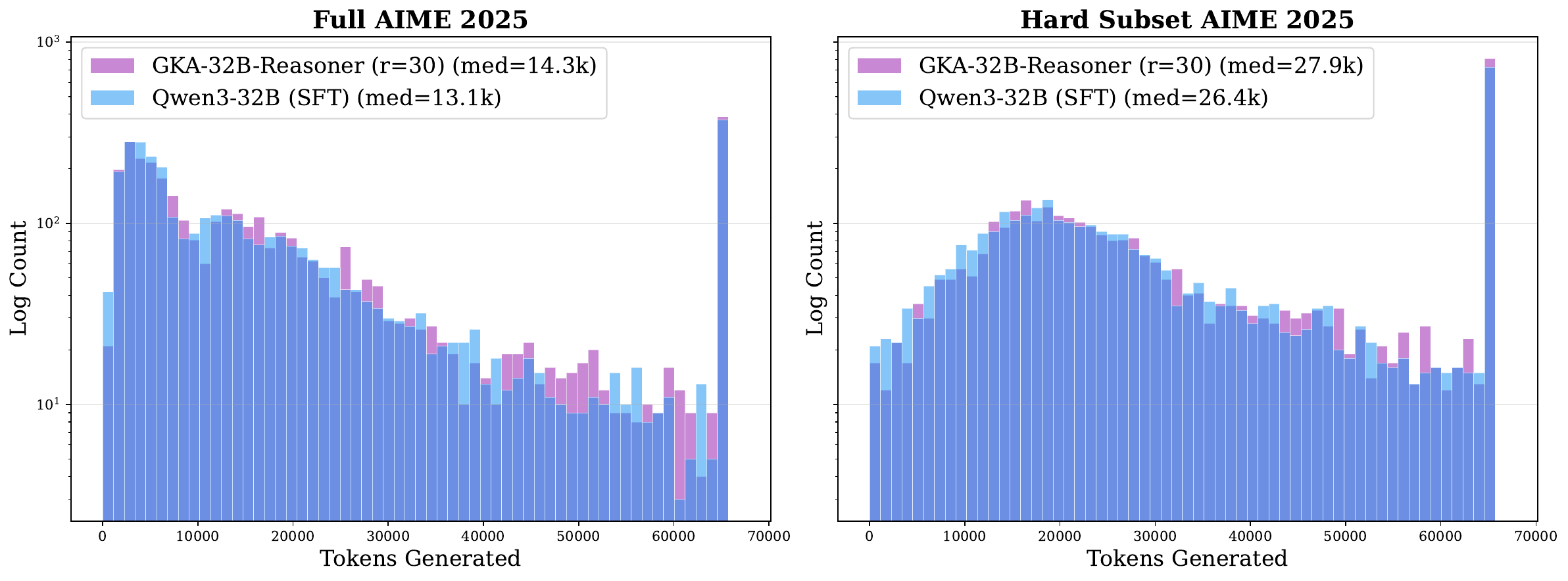}
\caption{\textbf{Reasoning output token length distributions.} Output token length distributions for GKA-Primed-HQwen3-Reasoner ($r{=}30$) and Qwen3-32B\,[Reasoner-SFT] on AIME~2025 ($n{=}3600$ rollouts per model). \textbf{Left:} full benchmark (30 problems). \textbf{Right:} hard subset (15 hardest problems). Both models were trained with the same reasoning SFT recipe and produce similar-length thinking traces, confirming that the Hybrid's speedups are not an artifact of shorter responses.}
\label{fig:tts-token-hist}
\end{figure}

\section{Training-free Context Extension}
\label{subsec:state-composition}

As context grows beyond a model's native training window, quality degrades. While our Primed Hybrids inherit long-context capability from the source Transformer, extending beyond that still requires a dedicated continued-pretraining stage with curated long-document data (see \cref{training_recipes:it}). A natural question is whether the Hybrid architecture itself can support length generalization \emph{without any additional training}.

Existing training-free context extension methods fall broadly into two categories: those designed for pure Transformers and those designed for pure SSMs.  On the Transformer side, positional-embedding interpolation schemes such as YaRN~\citep{yarn} rescale positional encodings so that the model can attend over longer sequences than it was trained on; DRoPE~\citep{drope} drops positional encodings entirely after a short recalibration phase; DCA~\citep{dca} chunks attention and recalibrates, while APE~\citep{ape} separately encodes context chunks in parallel and realigns their attention distributions at inference, enabling many-chunk contexts without re-prefilling. On the SSM side, PICASO~\citep{picaso} shows that the permutation-invariant structure of certain recurrent states allows independently processed context chunks to be composed at inference time, and recent work has begun to characterize when length generalization succeeds or fails in recurrent models~\citep{lengthgenrecurrent}. However, none of these methods address the \emph{hybrid} setting directly.   

Extending context for Hybrid models requires a composition strategy that handles both memory types jointly: the KV caches from independently processed chunks must be reconciled, and the SSM states must be merged in a way that respects the recurrence structure of each layer type. We introduce \emph{state composition} as, to the best of our knowledge, the first training-free context extension method for Hybrid models. The core idea is to partition a long input into chunks at the native context length, process each chunk independently to obtain per-chunk KV caches and SSM states, and merge these before generating. SSM states are fixed-size and admit principled merging operations: CASO decomposition~\citep{picaso} for Mamba2 and GDN, and additive composition in information filter form for GKA (\cref{app:state-comp-theory}). The result is a fully training-free method that extends effective context by $2\times$ (128K $\to$ 256K), with exploratory results at $4\times$ (512K) and no parameter updates required.

\subsection{Hybrid States Composition}

\begin{figure}[h]
  \centering
  \includegraphics[width=0.85\textwidth]{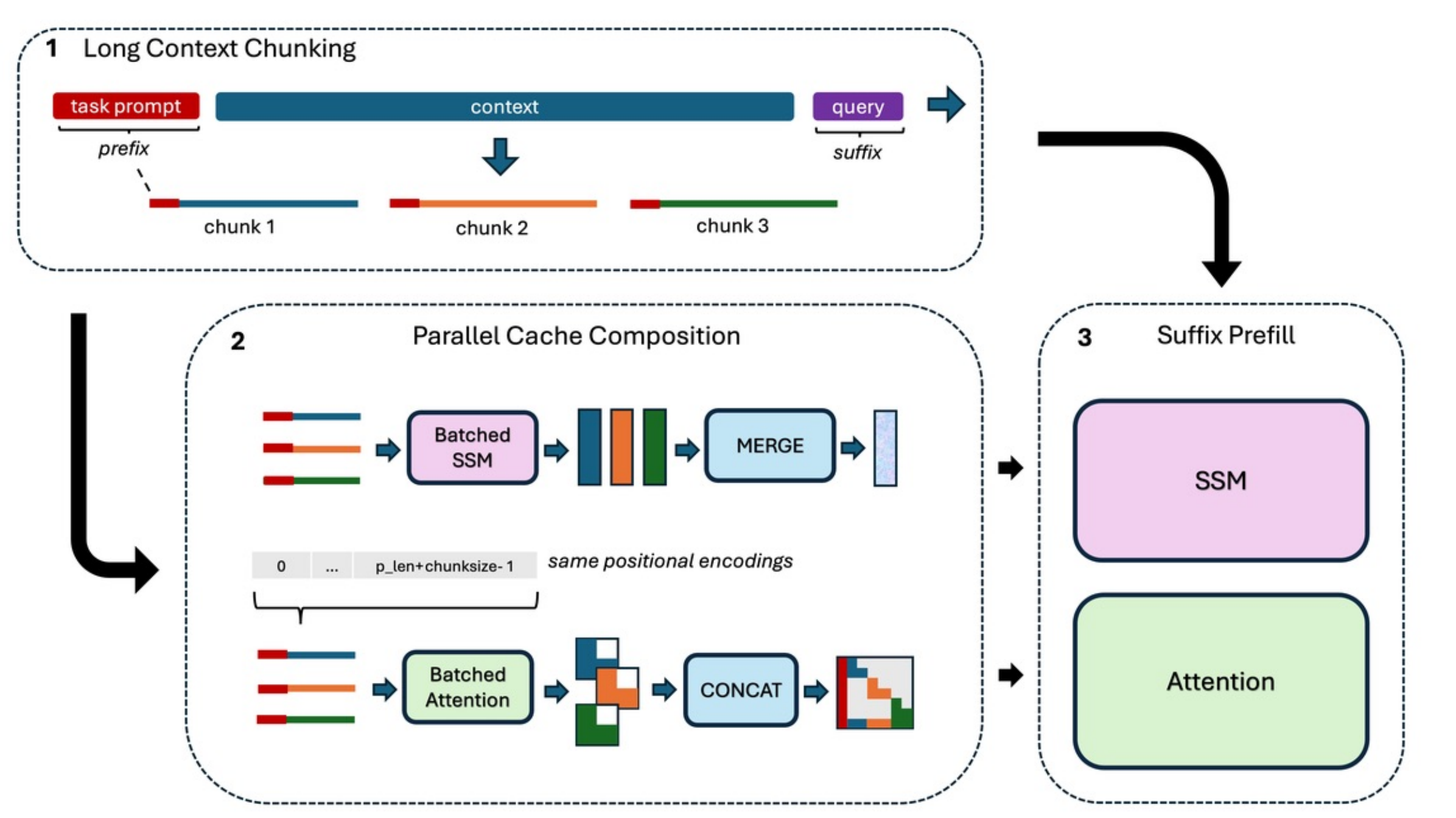}
  \caption{\textbf{State composition for hybrid models.} Long inputs are partitioned into chunks at the native context length, each processed independently. KV caches are concatenated while SSM states are merged (via averaging or an alternative method as in ~\cref{subsubsec:theory}), yielding training-free context extension.}
  \label{fig:state-composition}
\end{figure}
Let $L$ denote the model's native context length and suppose we wish to
process an input of length $N = K \cdot L$ for some integer $K > 1$.
State composition proceeds in three steps (\cref{fig:state-composition}).

\paragraph{Step 1: Chunking.}
The input sequence $\mathbf{x}_1, \ldots, \mathbf{x}_N$ is partitioned into $K$ chunks of length $L$.  A shared prefix $\mathbf{p}$ (e.g.\ a system prompt) is prepended to each chunk, similar to the parallel encoding strategy of APE~\citep{ape}.  This ensures that attention sink tokens~\citep{attentionsinksurvey}, the small set of initial tokens that absorb disproportionate Attention mass in trained Transformers, are present in every chunk, preserving the Attention distribution the model expects.

\paragraph{Step 2: Independent prefill and state merging.}
Each chunk is processed independently via a standard forward pass. Let $P = |\mathbf{p}|$ denote the prefix length; the position ids assigned to each chunk run from $0$ to $P + L - 1$, so that every chunk sees identical positional encodings regardless of its location in the original sequence.  This yields per-chunk KV caches $\{\mathbf{K}^{(c)}, \mathbf{V}^{(c)}\}_{c=1}^{K}$ for the Attention layers and per-chunk SSM states
$\{\mathbf{S}^{(c)}\}_{c=1}^{K}$ for the SSM layers %
The chunks can be processed as a batch, making this step highly parallel.

The KV caches are concatenated across chunks, retaining only a single copy of the prefix entries:
\begin{equation}
  \mathbf{K} = [\mathbf{K}_{\mathrm{prefix}};\;
    \mathbf{K}^{(1)}_{\mathrm{body}};\; \ldots;\;
    \mathbf{K}^{(K)}_{\mathrm{body}}],
  \qquad
  \mathbf{V} = [\mathbf{V}_{\mathrm{prefix}};\;
    \mathbf{V}^{(1)}_{\mathrm{body}};\; \ldots;\;
    \mathbf{V}^{(K)}_{\mathrm{body}}].
\end{equation}
For the SSM states, we consider two merging strategies.
\begin{itemize}[nosep]
  \item \textbf{Souping} (simple averaging):
    $\mathbf{S}_{\mathrm{merged}} = \frac{1}{K}\sum_{c=1}^{K} \mathbf{S}^{(c)}$.
    Each chunk contributes equally; no learned parameters are involved. This is most similar to KV cache concatentation as there is no inter-chunk attention
  \item \textbf{Fusion} (e.g. PICASO~\citep{picaso}):
    a weighted combination using PICASO coefficients that exploit the
    linear recurrence structure of each SSM
    (\cref{app:state-comp-theory}).
\end{itemize}
In practice, we default to souping for our experiments, since the two strategies perform comparably at long context lengths (256K+) while souping is cheaper to compute.

\paragraph{Step 3: Query prefill.}
The final query segment (e.g.\ the user question) is prefilled normally, starting from the merged KV cache and SSM state. Generation then proceeds autoregressively as usual.

\subsection{From native 128k to 256k context length}

\begin{figure}[h]
  \centering
  \includegraphics[width=\textwidth]{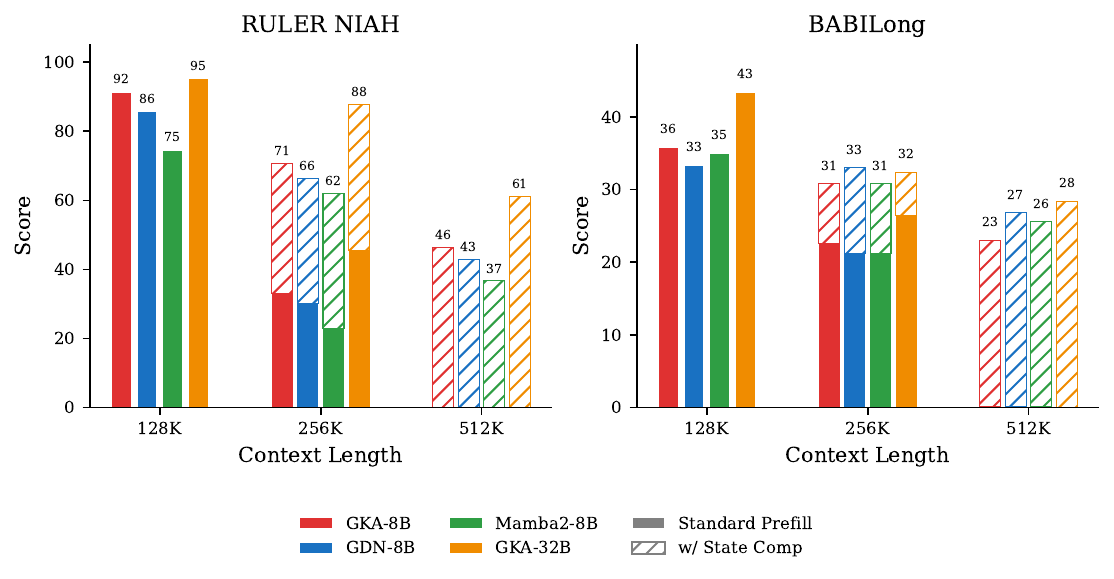}
  \caption{\textbf{State composition extends effective context well beyond the 128K native training window.} Performance of Primed Hybrid IT models (GKA-8B, GDN-8B, Mamba2-8B, GKA-32B) on RULER NIAH and BABILong at 128K (native), 256K ($2\times$), and 512K ($4\times$) contexts. Solid bars show standard prefill (running the model on the full input in a single pass); hatched bars show the additional gain from state composition with 128K chunks. Standard prefill degrades sharply beyond 128K and collapses to near zero at 512K for all 8B models (GKA-32B OOMs), while state composition recovers substantial performance on both retrieval (RULER NIAH) and reasoning (BABILong) tasks.} %
  \label{fig:state-comp-results}
\end{figure}

\paragraph{Setup.} We evaluate state composition on four Primed Hybrid IT models, all trained with a native context of 128K tokens: GKA-Primed-HQwen3-8B-IT (``GKA-8B''), GDN-Primed-HQwen3-8B-IT (``GDN-8B''), Mamba2-Primed-HQwen3-8B-IT (``Mamba2-8B''), and GKA-Primed-HQwen3-32B-IT (``GKA-32B''). We test at the native 128K context and at $2\times$ (256K) and $4\times$ (512K) extensions on two synthetic long-context benchmarks: RULER NIAH~\citep{ruler}, which measures multi-key and multi-value retrieval, and BABILong~\citep{babilong}, which embeds reasoning questions within long distractor text. All composition experiments use a 128K chunk size and the full state composition configuration: SSM state souping combined with zero-shot YaRN scaling on the Attention layers. No chat template is applied.

\paragraph{Results.} \cref{fig:state-comp-results} summarizes the main results. Standard prefill (running the model on the full input in a single pass, without chunking) holds up at 128K but degrades sharply at $2\times$ native context and effectively collapses at $4\times$, dropping to near zero on NIAH for all 8B models and OOMing entirely for GKA-32B: the Attention layers' KV cache grows linearly with sequence length, so materializing it for 512K tokens exceeds GPU memory at 32B scale. State composition sidesteps this by processing independent 128K chunks whose per-chunk KV cache fits in memory, and merging fixed-size SSM states that do not grow with the context length. This lets the method recover most of the performance loss at 256K across all four models and extend practical utility to 512K, a regime that standard prefill cannot handle at all for GKA-32B. The SSM expressivity hierarchy (GKA $>$ GDN $>$ Mamba2) observed in earlier sections is preserved under state composition.

\paragraph{Ablation: Attention scaling.} \cref{tab:sc-ablation} isolates the contribution of SSM state composition from positional-embedding scaling. We compare three configurations: \emph{YaRN only}~\citep{yarn}, which rescales RoPE frequencies but does not compose SSM states; \emph{State Comp w/o Attention scaling}, which applies our chunked composition pipeline without any RoPE scaling; and \emph{State Comp (full)}, which combines both. On RULER NIAH at 256K, composition without Attention scaling already matches or exceeds YaRN for all three 8B Hybrid models (e.g., GKA-8B: 66.6 vs.\ 64.5), and the gap widens at 512K as YaRN degrades more sharply. On BABILong the two methods are comparable at 256K, but composition w/o Attention scaling pulls ahead at 512K. The full method (State Comp) yields the best results across the board, indicating that positional-embedding scaling and SSM state composition address complementary failure modes: YaRN addresses the Attention layers' out-of-distribution RoPE positions, while state composition addresses the SSM layers' out-of-distribution state dynamics.

\begin{table}[h]
\centering
\caption{\textbf{Ablation: contribution of Attention scaling vs.\ SSM state composition.} RULER NIAH and BABILong scores at 128K (native), 256K, and 512K for the three 8B Primed Hybrid IT models (GKA-8B, GDN-8B, Mamba2-8B). We compare four configurations: \emph{Standard Prefill} runs the model on the full input in a single pass, without chunking; \emph{YaRN only}~\citep{yarn} rescales RoPE frequencies on the Attention layers but does not compose SSM states; \emph{w/o Attn.\ scaling} applies our chunked SSM state composition pipeline without any RoPE scaling; \emph{State Comp} is the full method, combining both. All composition variants use a 128K chunk size. The final column averages across both benchmarks at 256K and 512K. The full method yields the best aggregate score for every model, indicating that positional-embedding scaling and SSM state composition address distinct long-context failure modes.}
\label{tab:sc-ablation}
\small
\begin{tabular}{llccccccc}
\toprule
& & \multicolumn{3}{c}{\textbf{RULER NIAH}} & \multicolumn{3}{c}{\textbf{BABILong}} & \\
\cmidrule(lr){3-5} \cmidrule(lr){6-8}
\textbf{Model} & \textbf{Method} & 128K & 256K & 512K & 128K & 256K & 512K & \textbf{Avg.\ $\geq$256K} \\
\midrule
GKA-8B & Standard Prefill   & 91.5 & 32.9 & 0.0  & 35.8 & 22.6 & 0.2  & 13.9 \\
       & YaRN only          & --   & 64.5 & 32.8 & --   & 29.0 & 17.2 & 35.9 \\
       & w/o Attn.\ scaling & --   & 66.6 & 34.2 & --   & 28.0 & 20.0 & 37.2 \\
       & State Comp         & --   & 70.7 & 46.3 & --   & 30.8 & 23.0 & \textbf{42.7} \\
\midrule
GDN-8B & Standard Prefill   & 85.8 & 30.1 & 0.0  & 33.4 & 21.2 & 0.2  & 12.9 \\
       & YaRN only          & --   & 61.7 & 32.2 & --   & 26.8 & 19.8 & 35.1 \\
       & w/o Attn.\ scaling & --   & 63.3 & 29.3 & --   & 29.4 & 19.0 & 35.3 \\
       & State Comp         & --   & 66.3 & 43.0 & --   & 33.0 & 26.8 & \textbf{42.3} \\
\midrule
Mamba2-8B & Standard Prefill   & 74.7 & 22.9 & 0.0  & 35.0 & 21.2 & 0.2  & 11.1 \\
          & YaRN only          & --   & 50.5 & 31.5 & --   & 29.2 & 17.2 & 32.1 \\
          & w/o Attn.\ scaling & --   & 46.9 & 23.9 & --   & 28.8 & 19.8 & 29.9 \\
          & State Comp         & --   & 62.0 & 36.7 & --   & 30.8 & 25.6 & \textbf{38.8} \\
\bottomrule
\end{tabular}
\end{table}

\paragraph{Remark.} State composition is best suited for recall-heavy tasks (e.g., NIAH, BABILong) where the query requires retrieving facts distributed across the context. For tasks with strong ordering dependence, where the relative position of information across chunks matters, assigning sequential position ids across chunks rather than independent per-chunk ids may be preferable, at the cost of serializing chunk prefill. Separately, our Attention scaling uses the default YaRN temperature schedule; a more thorough exploration of temperature scaling (per-layer, per-head, or as a function of chunk count) may yield additional gains at higher extension factors.

\section{Gated KalmaNet: A Fading Memory SSM Layer with Variable Test-time Compute}
\label{sec:gka_main_ideas}
In this section we develop the three GKA contributions listed in \cref{sec:gka_at_scale}: the input-selectivity parameter $\beta_t$(\cref{subsec:beta_gating}), the variable test-time compute trade-off exposed by the Chebyshev iteration count $r$ (\cref{subsec:chebyshev}), and the symmetric tiled Triton kernel for GKA decoding (\cref{kernel}).

\subsection{Input Selectivity for GKA via $\beta_t$}
\label{subsec:beta_gating}
GKA implements the information form of the Kalman filter, maintaining two recurrent states: an information matrix $\mathbf{H}_t$ and an information state $\mathbf{U}_t$. In the original formulation of~\citet{gatedkalmanet}, both are decayed by a learned forgetting factor $\gamma_t \in [0,1]$ at each step and then updated with the current key-value pair:
\begin{align}
\mathbf{H}_t &= \gamma_t \cdot \mathbf{H}_{t-1} + \mathbf{k}_t \mathbf{k}_t^\top \quad &\text{(Information Matrix)} \nonumber \\
\mathbf{U}_t &= \gamma_t \cdot \mathbf{U}_{t-1} + \mathbf{v}_t \mathbf{k}_t^\top, \quad &\text{(Information State)} \label{eq:info_state_original}
\end{align}

\noindent $\gamma_t$ controls how quickly information about old tokens fade from the state, distant observations accumulate a product of $\gamma$ values that shrinks their effective contribution (to the current state at time $t$), so the filter naturally down-weights the remote past. However, $\gamma_t$ only governs how long information is retained. Every incoming token still writes its full (unmodulated) key-value outer product into the information Matrix and State, regardless of whether that token is relevant. This consumes capacity in both $\mathbf{H}_t$ and $\mathbf{U}_t$ until $\gamma_t$ eventually decays them away. 
To address this, we introduce an input-selectivity parameter $\beta_t \in [0,1]$ that modulates each token's contribution \textit{before} it enters the information Matrix and State:
\begin{align}
\mathbf{H}_t &= \gamma_t \cdot \mathbf{H}_{t-1} + \beta_t \mathbf{k}_t \mathbf{k}_t^\top \quad &\text{(Information Matrix)} \nonumber \\
\mathbf{U}_t &= \gamma_t \cdot \mathbf{U}_{t-1} + \beta_t \mathbf{v}_t \mathbf{k}_t^\top, \quad &\text{(Information State)} \label{eq:info_state_beta}
\end{align}

Here $\beta_t$ is computed as a linear projection of the input followed by a sigmoid. When $\beta_t \approx 0$, the token is effectively filtered out and the state remains unchanged; when $\beta_t \approx 1$, the token is fully incorporated. From the KF perspective, a small $\beta_t$ is equivalent to treating the observation as highly noisy, causing the filter to largely ignore it. This gives the model two complementary controls: $\gamma_t$ determines how long past information persists, while $\beta_t$ determines what enters the state in the first place. The original GKA formulation in \cref{eq:info_state_original} is recovered as the special case
$\beta_t \equiv 1$. Unless stated otherwise, all GKA layers in this work use the $\beta_t$-scaled update in \cref{eq:info_state_beta}.

\begin{figure}[h]
  \centering
  \includegraphics[width=0.85\linewidth]{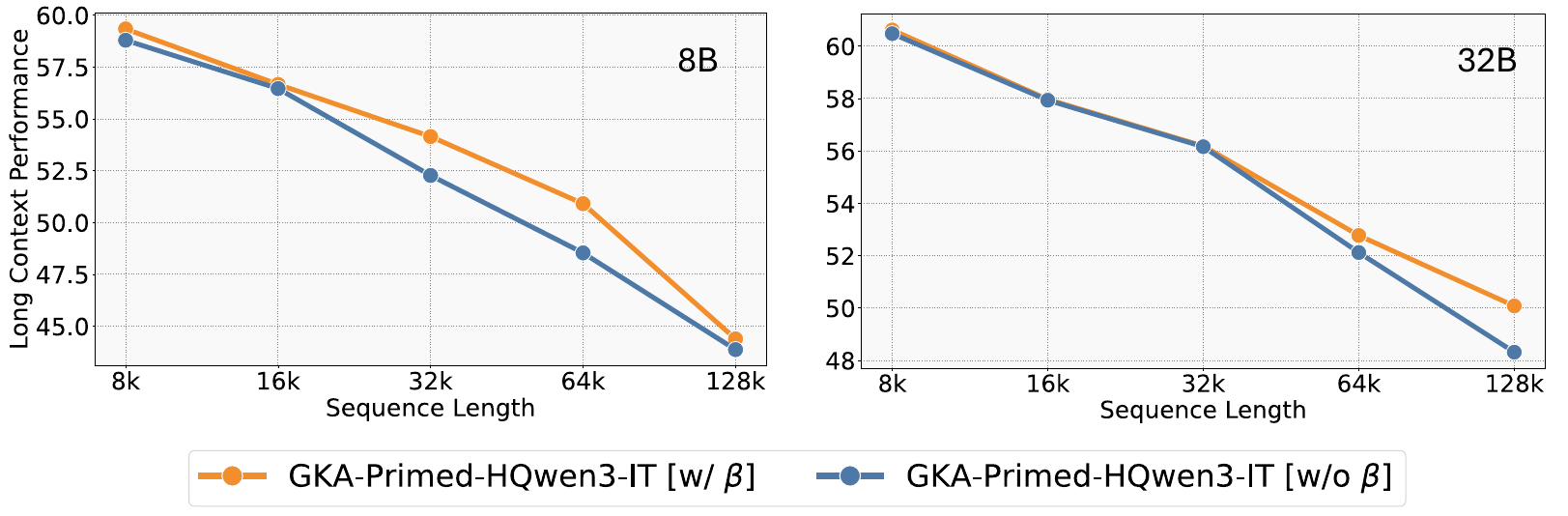}
  \caption{\textbf{Input selectivity through $\beta_t$ improves long context capabilities of GKA-Primed models.} Long-context performance, averaged over the tasks described in \cref{subsec: IT_results_long_ctx} for 8B and 32B GKA-Primed-HQwen3-IT models with and without input selectivity ($\beta_t$), across sequence lengths from 8k to 128k. Adding $\beta_t$ generally improves performance, with the gains becoming more pronounced at longer context lengths.}
  \label{fig:beta_ablation}
\end{figure}

\noindent\textbf{Experiments.} \cref{fig:beta_ablation} shows the average performance on long-context tasks described in \cref{subsec: IT_results_long_ctx} benchmark across sequence lengths ranging from 8k to 128k. For both the 8B and 32B GKA-Primed-HQwen3-IT models, adding $\beta_t$ generally improves performance, with the gains becoming more pronounced at longer contexts. This is consistent with our motivation: as sequence length grows, the SSM layers in the hybrid must summarize an increasingly large number of tokens into a fixed-dimensional state. Without input selectivity, irrelevant tokens accumulate in the state and degrade its quality. The $\beta_t$ parameter mitigates this by filtering out uninformative tokens before they enter the state, preserving capacity for future tokens. %

\subsection{Variable Test-time Compute with GKA}
\label{subsec:chebyshev}

\begin{figure}[h]
    \centering
    \includegraphics[width=0.85\textwidth]{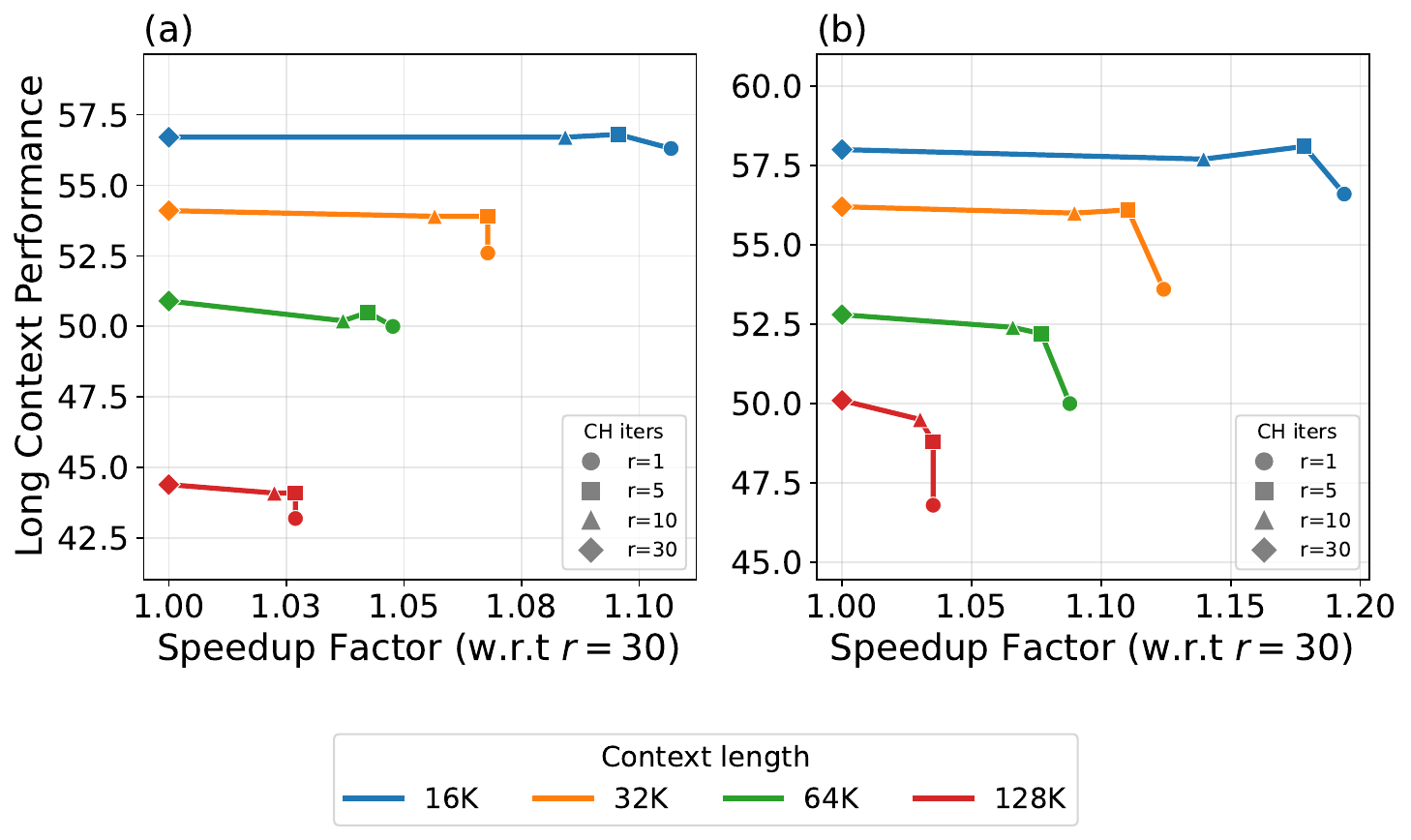}
    \caption{\textbf{GKA: Trading compute for speed at inference time}. (a): GKA-Primed-HQwen3-IT-8B model. (b): GKA-Primed-HQwen3-IT-32B model. Each curve represents a different context length, and marker shapes denote the number of Chebyshev iterations $r$. Reducing $r$ from its training value (30) increases throughput with only a modest drop in long-context performance. In both cases, a single trained model supports variable test-time compute, allowing practitioners to trade accuracy for efficiency by simply adjusting $r$.}
    \label{fig:accuracy_speedup}
\end{figure}

At each step $t$, GKA aims to compute its output $\mathbf{y_t}$ from its Information Matrix and State, see \cref{eq:info_state_beta} for definitions, by solving a linear system of equations:
\begin{equation}\label{eq:gka_solve}
(\mathbf{H}_t + \lambda_t \mathbf{I})\,\mathbf{x}_t = \mathbf{q}_t, \qquad \mathbf{y}_t = \mathbf{U}_t\,\mathbf{x}_t
\end{equation}

Solving \cref{eq:gka_solve} exactly (for instance via \texttt{torch.linalg.solve}) incurs $O(d_k^3)$ cost per time step, where $d_k$ is the dimension of the key vectors $\mathbf{k}_t$, and requires materializing every $\mathbf{H}_t$, both of which become prohibitive at scale. GKA circumvents this issue by employing \emph{Chebyshev Iteration} (CH), an iterative solver\footnote{See \cite{gatedkalmanet} for a comparison with other choices for iterative solvers like Conjugate Gradient.} that at each iteration performs only a matrix-vector product with $\mathbf{H}_t$ at $O(d_k^2)$ cost. Moreover, since this product can be expressed in terms of the underlying $\mathbf{k}_t$ factors, the full matrix never needs to be stored. 

Let $\mathbf{x}^{\textrm{CH}}_{t, r}$ denote the solution computed after $r$ iterations of CH at
time step $t$. The residual $\|\mathbf{x}_t - \mathbf{x}^{\textrm{CH}}_{t, r}\|_2$ between the
exact solution $\mathbf{x}_t$ of \cref{eq:gka_solve} and the CH approximation converges
exponentially in $r$~\citep{barret1994solvers}. Thus, increasing $r$ yields a more accurate
solution at the cost of additional compute.

In principle, one could train a separate model for each value of $r$ to explore this compute-accuracy tradeoff, but this is not practical to serve and expensive to train. Instead, we adopt the following strategy: Train with a single, heuristically chosen large $r$ so that the solver is ``well-converged" during training.\footnote{In practice, we select $r$ via a synthetic experiment: we draw query and key tensors from a standard Gaussian, form the regularized Information Matrix $(\mathbf{H}_t + \lambda_t \mathbf{I})$, and measure the CH residual as a function of $r$. For our chosen regularization strength $\lambda_t$, we then select $r$ such that the residual drops below $10^{-2}$, which is sufficient given the precision limits of BF16 arithmetic. See Appendix~G of \cite{gatedkalmanet} for details on the experiment setup.} At inference time, $r$ can then be freely reduced according to the deployment's compute budget, trading a small amount of approximation error for faster generation. Since no retraining is involved, a single trained GKA model serves the entire accuracy-efficiency trade-off dictated by $r$. Existing SSMs like Mamba-2 and GDN do not offer a comparable \textit{knob} at inference time to trade off accuracy for speed.\footnote{MesaNet~\citep{von2025mesanet}, a recently proposed SSM, also employs an iterative solver (Conjugate Gradient) in its state computation and in principle supports varying the iteration count at inference time. However, \cite{gatedkalmanet} report that the open-source MesaNet implementation is unstable under low-precision training, producing NaNs in BF16, which makes it difficult to scale in current LLM training pipelines.}
\\

\noindent\textbf{Experiments.} We evaluate this trade-off on our GKA-Primed-HQwen3-IT models at 8B and 32B scale. For each model, we vary $r \in \{1, 5, 10, 30\}$, where $r{=}30$ is the value used during training, and measure both downstream long-context performance and sustained decode throughput (tokens/s) on $8{\times}$H200 GPUs with tensor parallelism (TP=8). Throughput is measured during pure decode with the GPU saturated (i.e., sufficient batch size to fully utilize the GPU memory), and we report speedup relative to the $r{=}30$ baseline. Results are shown in \cref{fig:accuracy_speedup}. We make the following observations:
\begin{enumerate}
\item The decode speedup from reducing $r$ is most pronounced at shorter context lengths, where the Chebyshev solver constitutes a larger fraction of the per-step cost. At 16K context, the 32B model gains roughly 20\% throughput at $r{=}1$ relative to $r{=}30$, while the 8B model gains about 10\%. As context length grows, the Attention layers in the Hybrid models increasingly dominate the decode cost, and the relative benefit of fewer CH iterations diminishes.\footnote{We note that our Primed models use a 50\% Hybrid ratio; increasing the proportion of GKA layers would further amplify the speedup gains from reducing $r$.
}

\item The performance regression from lowering $r$ is mild at shorter contexts but grows with sequence length, consistent with the intuition that approximation error (the residual) from a low $r$ compounds over longer generation trajectories. At 16K, $r{=}5$ performs on par with $r{=}30$ for both model sizes. At 128K, however, the gap widens: the 8B model loses roughly 2.7\% when moving from $r{=}30$ to $r{=}1$, while the 32B model loses about 6.6\%, indicating that larger models are additionally more sensitive to the approximation quality of the CH solver.

\end{enumerate}

Together, these results show that the compute vs.~efficiency tradeoff depends on both model scale and context length. The key practical takeaway is that both models were trained once at $r{=}30$, yet at deployment one can cheaply sweep over $r$ on the downstream task of interest to find the right balance between model performance and speed without any retraining. For example, on the long-context tasks considered in \cref{fig:accuracy_speedup}, reducing $r$ from 30 to 5 at shorter contexts is nearly lossless in performance while providing a speedup. At longer contexts, a higher $r$ is preferable, and the additional cost is small since Attention layers dominate decode time at those lengths. One could also adapt $r$ dynamically within a single generation pass, starting with fewer iterations early in the sequence and increasing $r$ as the context grows; we leave this for future work.

\subsection{A Symmetric Tiled Triton Kernel for GKA Decoding} \label{kernel}

During decoding, GKA generates one token at a time, and the cost of each decoding step is dominated by four operations: (i) updating the Information Matrix $\mathbf{H}_t$ and Information State $\mathbf{U}_t$ from \cref{eq:info_state_beta}, (ii) computing the regularization strength $\lambda_t = \alpha \|\mathbf{H}_t\|_F$, which requires a full pass over $\mathbf{H}_t$ to evaluate its Frobenius norm, (iii) running $r$ CH iterations to solve $(\mathbf{H}_t + \lambda_t \mathbf{I})\mathbf{x}_t = \mathbf{q}_t$ (\cref{eq:gka_solve}), and (iv) computing the output $\mathbf{y}_t = \mathbf{U}_t \mathbf{x}_t$. Operations (i) and (iv) each traverse the full $d_v \times d_k$ state $\mathbf{U}_t$, while operations (i), (ii), and (iii) repeatedly traverse the full $d_k \times d_k$ state $\mathbf{H}_t$: once for the update, once for the norm, and $r$ times inside the CH loop. Decode throughput is thus bounded by how efficiently we can move these states between the GPU's High-Bandwidth Memory (HBM) and registers, and how much parallelism the GPU can extract while doing so. The original implementation from \cite{gatedkalmanet} loads $\mathbf{H}_t$ and $\mathbf{U}_t$ from HBM in one shot, which prevents the GPU from effectively overlapping compute with HBM transfers and leaves the Streaming Multiprocessors (SMs)\footnote{Streaming Multiprocessors are the parallel compute units on an NVIDIA GPU (an H200 has 132 of them, while an A100 has 108), each with its own cores, on-chip memory, and register file. Compute throughput on GPUs is bounded by how busy the SMs can be kept.} under-utilized. We address this with a \emph{symmetric tiled} Triton kernel that exploits two key properties of the GKA states: (i) all operations involving $\mathbf{H}_t$ and $\mathbf{U}_t$ are matrix-matrix or matrix-vector computations that are amenable to tiling, that is, we can process the state in smaller tiles instead of loading it in one shot; and (ii) symmetry of $\mathbf{H}_t$, since it is a cumulative weighted sum of rank-1 key outer products $\mathbf{k}_t \mathbf{k}_t^\top$, which allows us to read and write fewer tiles to and from HBM by only materializing the tiles belonging to the lower-triangular half.\\

\noindent \textbf{Tiling to reduce memory traffic and register pressure.} Triton kernels are launched as a collection of \textit{program instances} that called a \textit{grid}---for all of our kernels, we launch a grid of size $\texttt{batch\_size} \times \texttt{num\_heads}$, so that each program instance processes a single $\mathbf{H}_t \in \mathbf{R}^{d_k \times d_k}$ and $\mathbf{U}_t \in \mathbf{R}^{d_v \times d_k}$ for a single head of a single batched sequence. Each program instance is scheduled many-to-one to run on a single SM in the GPU. Each SM has a small, fixed-size register file (256 KiB per SM on an H200) that is shared across all program instances running on it concurrently. The more registers each program instance uses, the fewer the GPU can schedule side-by-side on the same SM, and the less opportunity the hardware has to hide memory-access latency by switching between them while is stalled waiting for data. Loading the full $\mathbf{H}_t$ and $\mathbf{U}_t$ into registers at once consumes a significant fraction of this budget: for our Primed Hybrid models with GKA, $d_k{=}128$ and $d_v{=}128$, so $\mathbf{H}_t$ and $\mathbf{U}_t$ are each $128 \times 128$ matrices in FP32 (64 KB each), and the intermediate results of the $\mathbf{H}_t$ update, Frobenius norm computation, CH matrix-vector operations, and output computation $\mathbf{y}_t = \mathbf{U}_t \mathbf{x}_t$ push register usage even higher. Partitioning $\mathbf{H}_t$ into smaller $b_k \times b_k$ tiles, and $\mathbf{U}_t$ into $b_v \times b_k$ tiles, solves this: each time $\mathbf{H}_t$ is used for computation, a limited number of tiles at a time are loaded from HBM or cache, consumed, and (when needed) written back, before more tiles are read into the same reused registers. With a reduced register footprint, the GPU can schedule more work concurrently on each SM, increasing parallelism and better overlapping compute with memory traffic. In practice, we have found that $b_k{=}64$ and $b_v{=}64$ gives the best trade-off with our Triton kernels, for our models on H200 GPUs: larger tiles amortize HBM loads over more compute, while smaller tiles free up enough registers for the GPU to run more work concurrently on each SM.\\

\noindent \textbf{Exploiting symmetry of $\mathbf{H}_t$.} Because $\mathbf{H}_t = \gamma_t \mathbf{H}_{t-1} + \beta_t \mathbf{k}_t \mathbf{k}_t^\top$ is symmetric, each strict upper-triangular tile is the transpose of its lower-triangular counterpart. Leveraging this, we never store the strict upper-triangular tiles to HBM, nor read them back: in the state-update loop, we skip their computation entirely and fold their contribution into the Frobenius norm by doubling the corresponding lower-triangular tile's accumulation; in the CH loop, we reconstruct each upper-triangular tile on the fly by transposing its lower-triangular mirror already in registers. The HBM savings depend on the tile size: for our Primed Hybrid models with $d_k{=}128$ and $b_k{=}64$, we skip 1 of the 4 tiles of $\mathbf{H}_t$, yielding a 25\% reduction in HBM traffic for $\mathbf{H}_t$; the savings grow toward 50\% as the ratio $d_k / b_k$ increases.\\

\begin{figure}[h]
\centering
\includegraphics[width=0.9\textwidth]{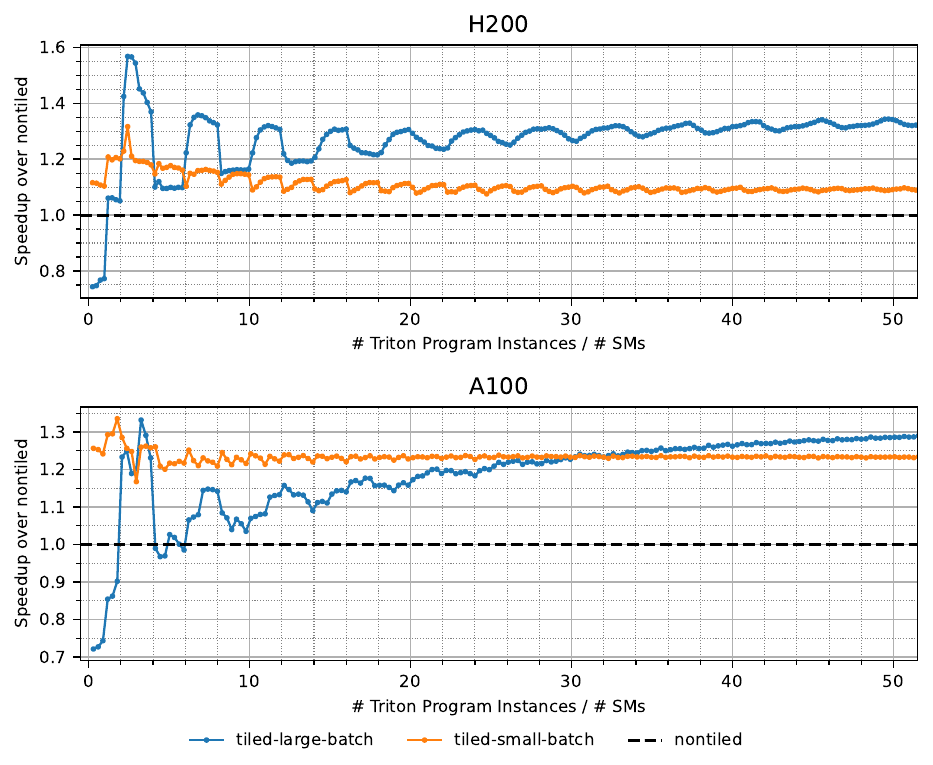}
\caption{\textbf{Speedup of symmetric tiled kernels over the non-tiled baseline} on H200 and A100 GPUs. The x-axis shows batch size normalized by the number of SMs on each GPU, so that the periodic pattern, caused by successive waves of Triton program instances filling and then exceeding SM capacity, aligns across GPU models. The \texttt{tiled\_small\_batch} variant provides a consistent speedup at all batch sizes by exploiting $\mathbf{H}_t$ symmetry while keeping all lower-triangular tiles in registers. The \texttt{tiled\_large\_batch} variant is slower for small batch sizes due to its per-iteration-reload overhead, but becomes the fastest kernel as the batch size grows large enough to benefit from the increased parallelism. This crossover occurs at a much smaller batch size on H200 than on A100.}
\label{fig:tiled-kernel-speedup}
\end{figure}

\noindent \textbf{Kernel variants for different batch sizes.} The tiled kernel with all of the optimizations described above (\texttt{tiled\_large\_batch}) reloads $\mathbf{H}_t$ tiles from HBM in each CH iteration, which reduces register usage for each program instance enough for the GPU to schedule twice as many concurrently per SM compared to a version that holds all lower-triangular tiles in registers throughout the CH loop. At large batch sizes this is a significant win, since there are enough program instances to run to fill all the SMs in the GPU with parallel work. However, at small batch sizes there are not enough program instances to benefit from the extra parallelism, the repeated HBM reloads become pure overhead, and the kernel runs slower than the \texttt{nontiled} baseline. Therefore, we implement a second tiled kernel variant (\texttt{tiled\_small\_batch}) that retains all lower-triangular $\mathbf{H}_t$ tiles in registers across CH iterations, avoiding the per-iteration-reload overhead, while still exploiting $\mathbf{H}_t$ symmetry and using the same storage layout that only loads and stores lower-triangular tiles of $\mathbf{H}_t$. This allows us to include both variants and dispatch to the appropriate kernel based on the number of program instances per GPU. \cref{fig:tiled-kernel-speedup} shows the speedup of each tiled kernel variant over the non-tiled baseline kernel implementation at different batch sizes, on both H200 and A100 GPU models. The \texttt{tiled\_large\_batch} variant becomes the fastest as the batch size increases on both GPU models, but on A100 it takes a much larger batch size for the benefits to outweigh the costs enough for it to run faster than the \texttt{tiled\_small\_batch} variant. \\

\noindent Having described the principles behind the tiled kernel, we describe the two loops that implement it (state update and Chebyshev iteration) in \cref{app:tiled-kernel-details}.\\

\section{Scalable Implementations}
\label{sec:scalable-impl}
\subsection{Priming Fused Architecture}
\label{subsec:fused-arch}

In Stage~1 of Priming (see \cref{subsec:stage1-theory}), the SSM layers in the Hybrid model are trained to replicate the source Transformer's end-to-end behavior on diverse textual data from different domains (see \cref{training_recipes:stage1}). This is achieved by treating the source Transformer as the teacher and distilling it into the Hybrid model via the objectives discussed in \cref{subsec:stage1-theory}. These require computing a simultaneous forward pass through both the source Transformer and the Hybrid model on the same input at every training step. A na\"ive implementation that loads both models into GPU memory as separate networks roughly doubles the parameter footprint, making Stage~1 infeasible for large-scale model priming: at 32B scale, for example, fitting two bf16 copies alone (without gradients) consumes about 128 GB of GPU memory. To reduce this footprint, we propose a \emph{fused} architecture that merges the two models into a single network, sharing every parameter common to both while keeping separate parallel layers only where they differ.

\paragraph{Architecture.}
The student Hybrid model shares all non-SSM layers (Multilayer Perceptrons, LayerNorms, LM head, token embeddings, and the retained Attention layers) with the source Transformer, and Stage~1 only trains the SSM-specific parameters. Leveraging this, we construct a fused architecture (illustrated in \cref{fig:fused-arch}) so that corresponding to each SSM layer in the Hybrid, the fused model maintains two parallel layers: a \emph{frozen} Attention layer (from the teacher) and a \emph{trainable} SSM layer (from the student), both operating on the same input. All other layers exist only once in memory, and only the SSM-specific parameters receive gradients.

\begin{figure}[h]
\centering
\includegraphics[width=0.8\textwidth]{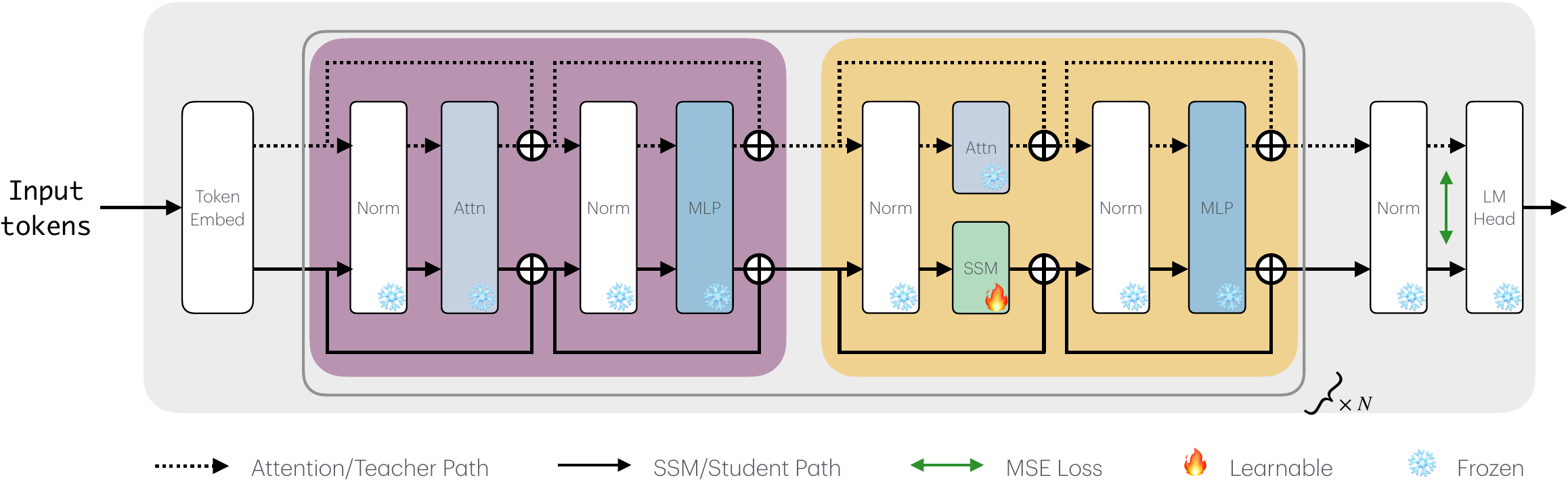}
\caption{\textbf{Fused teacher-student architecture used during Stage~1 of Priming.} For every SSM layer in the student Hybrid, the fused architecture maintains two parallel layers, corresponding to two parallel pathways for processing the input: the teacher's Attention layer (dashed, frozen) and the student's SSM layer (solid, learnable). These layers are highlighted in yellow in the figure. All other layers (highlighted in purple) are shared between the teacher and student and thus exist only once in memory. During Stage~1, the SSM layers are trained such that the student Hybrid model learns to match the output of the teacher source Transformer model. Afterwards, the Attention layers in the yellow region are discarded (\emph{unfused}), yielding a standalone Hybrid model.}
\label{fig:fused-arch}
\end{figure}

\paragraph{Memory savings.}
By sharing every non-SSM parameter between teacher and student, the fused architecture avoids a full duplicate of the shared weights that a separate teacher-student setup would carry. Concretely, it saves $2 \times (\mathcal{M} - \mathcal{A})$ bytes in bf16, where $\mathcal{M}$ is the total source Transformer parameter count and $\mathcal{A}$ is the parameter count of the Attention layers that correspond to SSM layers in the Hybrid model. For our Primed Hybrid models at a 50\% Hybrid ratio, this translates to a ${\sim}33\%$ reduction in total training memory. 

\paragraph{Unfusing.}
After Stage~1, only the Hybrid model is of interest for Stage 2 of Priming for task adaptation (\cref{subsec:stage2-theory}). We therefore discard the frozen teacher Attention layers that ran in parallel with the student's SSM layers, and Stage~2 proceeds on the student Hybrid alone.

\subsection{Sequence Parallelism for Long-context Training of Hybrid Models}
\label{app:sp_for_hybrids}
Training large-scale LLMs at long context lengths (128K+) poses a fundamental memory challenge: activation memory grows linearly with sequence length, quickly exhausting per-GPU capacity. While model-sharding strategies like Tensor Parallelism (TP), Pipeline Parallelism (PP) and Fully-sharded Data Parallelism (FSDP) reduce per-GPU model footprint, none of them shard along the sequence dimension. This in turn limits their ability to scale with increasing sequence length. Sequence Parallelism (SP) addresses this directly by partitioning the input across GPUs along the sequence dimension. This reduces each GPU's activation footprint proportionally to the SP degree, enabling scaling to arbitrarily long sequences given sufficient GPUs.

While SP for Transformers is well-studied \citep{grattafiori2024llama,jacobs2023deepspeed}, extending it to hybrid models introduces a design challenge. Attention layers require communication of all past tokens (the KV cache) across GPUs, naturally suiting collective operations like AllToAll (A2A)\footnote{as in Ulysses SP}. SSM layers, however, are fundamentally different as they compute a fixed-dimensional latent state that acts as a \textit{compressed summary} of the entire past. This makes collective communication unnecessary for SSMs and instead it suffices to pass the state across GPUs via Point-to-Point (P2P) communication protocols. An effective SP implementation for Hybrids must therefore accommodate both communication patterns within the same model.

In this work, we develop two SP algorithms for SSM layers that can readily be integrated with existing implementations for Attention to allow for long-context training of Hybrid models. Our contributions are as follows.
\begin{enumerate}
    \item A P2P algorithm for SP that is applicable to any SSM layer with linear recurrent dynamics. Each GPU processes its local sequence chunk independently, then corrects for missing context by passing the finite-sized recurrent state to the next GPU, resulting in communication volume that is independent of sequence length. We contrast this with existing A2A-based SP implementations for SSM layers in large-scale LLM training frameworks \citep{shoeybi2019megatron}, showcasing the efficacy of our approach.
    \item While effective, developing P2P algorithms requires expertise into the internals of these SSM layers, which are often implemented in Triton and not straightforward to modify. As a solution, we propose Universal SP (USP), inspired by Ulysses SP for Attention \citep{jacobs2023deepspeed}, that trades efficiency for versatility. The core idea is to use AllGather operations to reconstruct the full sequence on each GPU, allowing the SSM kernel to run unmodified on the complete sequence as a purely local computation, thereby requiring zero changes to the underlying Triton kernels. Moreover, unlike P2P, USP is applicable to any sequence-mixing layer (SSMs, LSTMs, Sparse Attention variants, etc.) without modification.
\end{enumerate}

We begin by describing how sequences are partitioned across GPUs under SP in \cref{subsec: how seq shard?}. We then present our P2P SP algorithm in \cref{subsec: p2p sp} and contrast it with existing A2A-based approaches in \cref{subsec: p2p vs a2a}. Finally, we introduce USP in \cref{sec: usp}, a drop-in alternative that enables rapid prototyping of long-context Hybrid models without requiring custom kernel implementations. We include both P2P-SP and USP in our Hybrid Model Factory release.\footnote{\url{https://github.com/awslabs/hybrid-model-factory/blob/main/docs/Inference.md}} Our P2P-SP supports GKA, GDN, Mamba2 and BMOJO-F layers.

\subsubsection{How are sequences sharded in SP?}
\label{subsec: how seq shard?}
Let $N_{\text{SP}}$ denote the SP size, that is, the number of GPUs across which the sequence is sharded. There exist two popular sharding patterns in the literature.

\begin{figure}[h]
  \centering
  \includegraphics[width=0.8\linewidth]{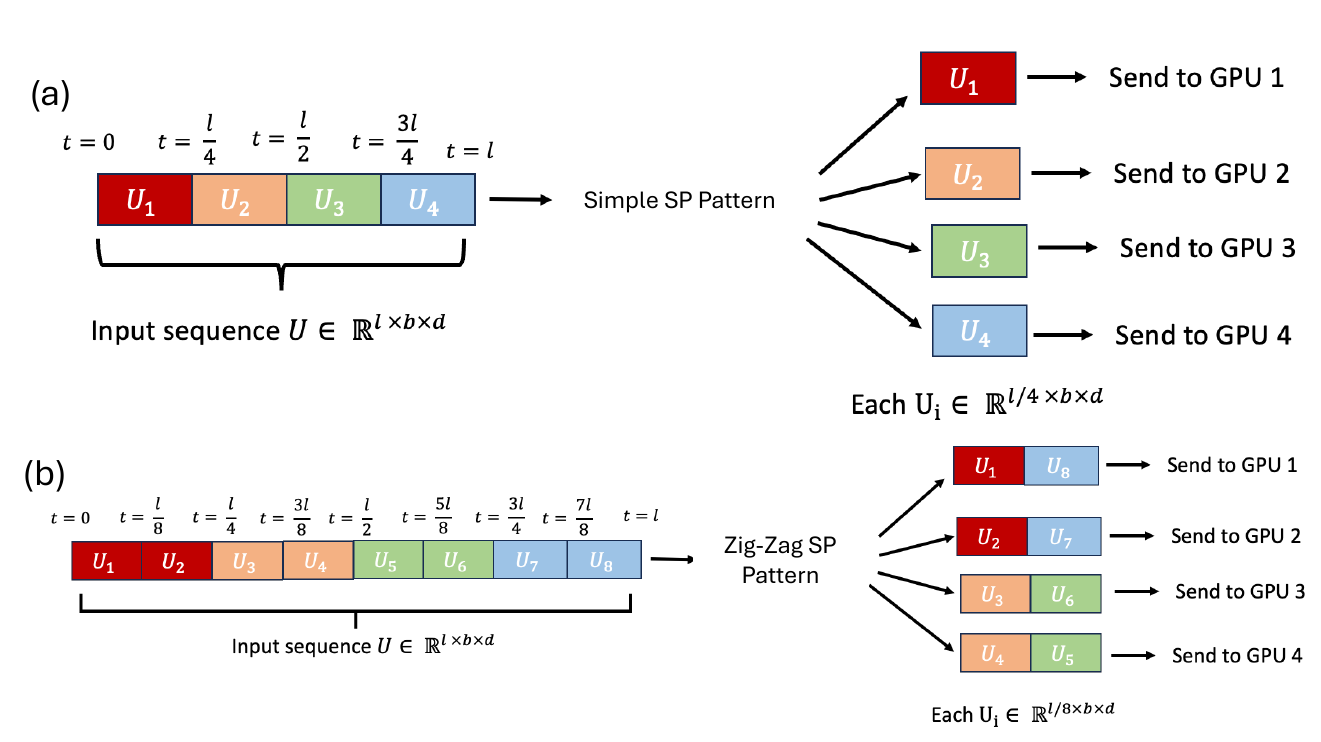}
  \caption{\textbf{SP sharding patterns for $N_{\text{SP}}=4$.} \textbf{(a)} Simple SP: the sequence is split into $N_{\text{SP}}$ contiguous chunks, one per GPU. \textbf{(b)} Zig-zag SP: the sequence is split into $2N_{\text{SP}}$ chunks and each GPU receives two discontiguous chunks (e.g., GPU 1 gets chunks 1 and 8), balancing the causal sequence mixing layers workload across ranks.}
  \label{fig:SP_sharding_patterns}
\end{figure}

\paragraph{Simple SP pattern.} Let the input sequence be of shape $[l, b, d]$, where $l$ is the sequence length, $b$ is the batch size and $d$ is the model dimension. In the simple SP pattern, the input sequence is sharded into $N_{\text{SP}}$ equal chunks of shape $[l/N_{\text{SP}}, b, d]$, with each chunk assigned to a unique GPU. \cref{fig:SP_sharding_patterns}a illustrates this for $N_{\text{SP}}=4$.

\paragraph{Zig-zag SP pattern.} The simple pattern leads to load imbalance for Attention layers during autoregressive training: queries at time $t$ must attend to all preceding keys, so GPUs holding later chunks require progressively more communication and computation. To address this, \citet{grattafiori2024llama} introduced the zig-zag pattern, now standard in frameworks such as Megatron-LM. The input is split into $2N_{\text{SP}}$ chunks numbered sequentially along the sequence dimension $l$, and GPU $i$ receives chunks $i$ and $2N_{\text{SP}} - 1 - i$, balancing the causal workload across ranks (\cref{fig:SP_sharding_patterns}).

Since SSM layers maintain a recurrent state that needs to be communicated from one chunk to the next, the sharding pattern is of little consequence in terms of load balancing. For simplicity, we will explain our approach following the Simple SP pattern (panel \textbf{(a)} in \cref{fig:SP_sharding_patterns}). Adapting to the zig-zag pattern is straightforward: the same P2P state passing applies, but each GPU must process its two discontiguous chunks in their original sequential order and exchange states with the appropriate neighboring rank accordingly.

\subsubsection{The P2P-SP algorithm for SSMs}
\label{subsec: p2p sp}
In modern SSM layers\footnote{specifically the ones of interest in this work, GKA, GDN and Mamba2.}, there exist two computations that operate across tokens and thus need special care to handle SP, 1D convolutions and the SSM recurrence. All other computations in an SSM layer, projections, activations and gating, do not have cross-token dependencies and thus do not need any inter-GPU communication. We describe how the P2P algorithm handles 1D convolutions and the SSM recurrence next.

\noindent \textbf{1D convolutions.} In SSMs, 1D convolutions are applied along the sequence dimension independently across each channel. Without loss of generality, we illustrate with a scalar sequence. 

Consider a scalar sequence $\bm{U} = \{u_i\}_{i=1}^{l}$ of length $l$ and a 1D convolution filter of size $d_{\text{conv}}$. The 1D convolution operation can be written as follows: for all $t \in [1, l]$,
\begin{equation}
y_t = \sum_{i=1}^{d_{\text{conv}}} w_i \cdot u_{t - d_{\text{conv}} + i},
\label{eq:sp_conv1d}
\end{equation}
where $w \in \mathbb{R}^{d_{\text{conv}}}$ is the convolution kernel and $y_t$ is the output at time $t$, with the convention that $u_j = 0$ for all $j \leq 0$.

Notice from \cref{eq:sp_conv1d} that the 1D convolution for any token in the input sequence depends on itself and the previous $d_{\text{conv}} - 1$ tokens. Under the simple SP pattern, computing the output for the first $d_{\text{conv}} - 1$ tokens on each GPU requires context from the preceding chunk. This is resolved via P2P communication: each GPU sends its last $d_{\text{conv}} - 1$ tokens to the next GPU in parallel, providing the necessary boundary context for \cref{eq:sp_conv1d}.

\noindent \textbf{SSM Recurrence.} SSMs that employ a input dependent linear recurrence for the state update and a linear readout function for the output are very amenable to sequence parallelization across GPUs. Specifically, these are SSMs whose recurrence can be written in the form
\begin{align}
\mathbf{S}_t &= \mathbf{S}_{t-1} \mathbf{A}_t + \bm v_t\mathbf{B}_t, \nonumber \\
\mathbf{y}_t &= \mathbf{S}_t \, \mathbf{q}_t, \label{eq:ssm_output_sp}
\end{align}
where $\mathbf{S}_t$ is a matrix-valued state; $\mathbf{A}_t$ and $\mathbf{B}_t$ are matrices computed as (potentially non-linear) functions of the input sequence up to time $t$; $\bm{v}_t$ and $\bm{q}_t$ are the value and query vectors at time $t$. Refer to \cref{tab:ssm-recurrences-main} for explicit instantiations of \cref{eq:ssm_output_sp} for specific SSM layers like GKA. At any timestep $n$, we can separate the contribution of the initial state $\mathbf{S}_0$ as
\begin{align}
\mathbf{S}_n &= \underbrace{\mathbf{S}_n\big|_{\mathbf{S}_0 = 0}}_{\text{zero-init state}} + \underbrace{\mathbf{S}_0  \mathbf{A}_{1:n}}_{\text{cross-chunk state correction term}}, \label{eq:ssm_state_decomp} \\
\mathbf{y}_n &= \underbrace{\mathbf{y}_n\big|_{\mathbf{S}_0 = 0}}_{\text{zero-init output}} + \underbrace{\left( \mathbf{S}_0  \mathbf{A}_{1:n} \right) \mathbf{q}_n}_{\text{cross-chunk output correction term}}, \label{eq:ssm_output_decomp}
\end{align}
where $\mathbf{A}_{1:n} = \mathbf{A}_1  \mathbf{A}_2  \cdots  \mathbf{A}_n$ is the cumulative state transition operator from time $1$ to $n$. The first terms are obtained by running the recurrence \cref{eq:ssm_output_sp} with $\mathbf{S}_0 = 0$.

This decomposition makes the adaptation to SP clear. Each GPU processes its local chunk assuming zero initial state. The GPUs then sequentially communicate their final states to the next rank, which uses the received state to compute the correction terms in \cref{eq:ssm_state_decomp}--\cref{eq:ssm_output_decomp}.

\subsubsection{Comparing P2P SP with A2A-based approaches}
\label{subsec: p2p vs a2a}
Scalable training frameworks like Megatron-LM employ A2A-based SP for SSM layers such as GDN and Mamba2\footnote{Megatron-Core, commit \texttt{20ba03f} on \texttt{main}. Note: Megatron-Core refers to the SP described in this section as \emph{context parallelism} to avoid confusion with its own \emph{sequence parallelism}, a distinct mode of parallelism that is out of scope here.}. Rather than exploiting the compact recurrent state of SSMs, this approach uses AllToAll collectives to redistribute tokens: the sequence dimension is gathered across GPUs while the head dimension is scattered, so that each GPU holds the full sequence for a subset of heads. The SSM kernel then runs on the complete sequence per head, and a reverse AllToAll restores the original sharding. The following are the advantages of our P2P approach over A2A SP for SSM layers.
\begin{itemize}
\item \textbf{Communication volume independent of sequence length.} P2P SP communicates only the fixed-size recurrent state and $d_{\text{conv}} - 1$ boundary tokens between adjacent ranks. A2A SP, in contrast, communicates all tokens across all ranks via AllToAll, with volume scaling linearly in the sequence length.

\item \textbf{No head-divisibility constraint.} A2A SP requires the number of heads to be divisible by $N_{\text{SP}}$, since heads are scattered across ranks. P2P SP operates on all heads locally and has no such constraint, making it compatible with any SP size.

\item \textbf{No restriction to intra-node SP.} Inter-node interconnects (e.g., InfiniBand) are typically an order of magnitude slower than intra-node links (NVLink). The high communication volume of A2A necessitates that all GPUs sharding a sequence in SP reside on the same node to avoid severe throughput degradation. For example, the total A2A volume at 128K with $N_{\text{SP}}=8$ for a single 
Qwen3.5-35B-A3B GDN layer's forward pass is ${\sim}4$\,GB. Over intra-node NVLink for DGX H200 nodes (${\sim}900$\,GB/s per GPU), this completes in under 5\,ms; over inter-node interconnects (${\sim}50$\,GB/s per GPU), the same transfer takes ${\sim}80$\,ms, making multi-node A2A SP impractical. P2P SP communicates only the ${\sim}1$MB recurrent state (for the same Qwen3.5 MoE model), making it viable across nodes.

\item \textbf{Better utilization of GPU compute.} SSMs process each head independently, naturally mapping to parallel execution across GPU Streaming Multiprocessors (SMs). A2A SP reduces the number of heads per GPU by a factor of $N_{\text{SP}}$, which can underutilize tensor cores — especially when $N_{\text{SP}}$ is large relative to the number of heads. P2P SP retains all heads on each GPU, preserving full parallelism across SMs.
\end{itemize}

\cref{tab:p2p_vs_a2a} reports the speedup (fwd + bwd) of P2P-SP over A2A-SP for a single GDN layer across a range of sequence lengths. In the intra-node setting (\cref{tab:p2p_vs_a2a}a), P2P-SP achieves a consistent $1.2$--$1.4\times$ speedup, with gains increasing as the sequence length grows and A2A-SP communication volume rises accordingly. The speedup peaks around 256K and slightly decreases at longer sequences. This can be attributed to the fact that P2P-SP incurs additional compute for the zero-init run and cross-chunk correction (\cref{eq:ssm_state_decomp}--\cref{eq:ssm_output_decomp}), whereas A2A-SP runs the SSM kernel in a single pass after communication. At longer sequences this compute overhead of P2P-SP grows and slightly narrows the gap. In the multi-node setting (\cref{tab:p2p_vs_a2a}b), the advantage of P2P-SP becomes more pronounced: at 128K, P2P-SP is nearly $2\times$ faster, reaching $2.3\times$ at 512K. This is a direct consequence of the inter-node bandwidth bottleneck discussed earlier: A2A-SP must transfer the full token tensor over the slower inter-node interconnect, while P2P-SP communicates only the compact recurrent state.

\begin{table}[t]
\centering
\caption{\textbf{Speedup (Fwd + Bwd) of P2P SP over A2A SP.} We measure the speedup for a single GDN layer (Qwen3.5-35B-A3B, $N_{\text{SP}}=8$) on H200 GPUs. P2P consistently outperforms A2A, with gains amplified in the multi-node setting where inter-node communication becomes the bottleneck. \textbf{(a)} Intra-node: 8 GPUs on a single P5en EC2 node. \textbf{(b)} Multi-node: 8 GPUs across 2 P5en EC2 nodes (4 GPUs each).}
\label{tab:p2p_vs_a2a}
\begin{subtable}[t]{0.45\linewidth}
\centering
\caption{Intra-node (1$\times$8 GPUs)}
\begin{tabular}{rc}
\toprule
Seq.\ Length & Speedup \\
\midrule
16K  & 1.04 \\
32K  & 1.23 \\
64K  & 1.29 \\
128K & 1.35 \\
256K & 1.37 \\
512K & 1.35 \\
1M   & 1.33 \\
\bottomrule
\end{tabular}
\end{subtable}
\hfill
\begin{subtable}[t]{0.45\linewidth}
\centering
\caption{Multi-node (2$\times$4 GPUs)}
\begin{tabular}{rc}
\toprule
Seq.\ Length & Speedup \\
\midrule
16K  & 1.14 \\
32K  & 1.73 \\
64K  & 2.15 \\
128K & 2.27 \\
256K & 2.33 \\
512K & 2.28 \\
1M   & 2.24 \\
\bottomrule
\end{tabular}
\end{subtable}
\end{table}

\paragraph{Why A2A is preferred over P2P for Attention but not for SSMs.} For Attention layers, Ulysses SP (A2A-based) is often more efficient than Ring Attention \citep{liu2023ring} (P2P-based). Ulysses performs two AllToAll collectives around the Attention kernel: the first redistributes the local $[l/N_{\text{SP}}, d]$ query, key, and value shards so that each GPU holds the full sequence for a $1/N_{\text{SP}}$ subset of heads, and the second reverses this redistribution on the Attention output to restore the original sequence sharding. The per-GPU communication volume for Ulysses is thus $O(l \cdot d / N_{\text{SP}})$, since every rank sends and receives the full token tensor restricted to its local head subset. Ring Attention, by contrast, keeps queries local and rotates key and value shards of size $O(l \cdot d / N_{\text{SP}})$ through $N_{\text{SP}} - 1$ sequential P2P hops on every rank, for a total communication volume of $O(l \cdot d)$ per GPU \citep{huggingface2025ulysses}. For SSMs, the situation reverses: the recurrent state is a fixed-size object independent of sequence length, so P2P-SP communicates only the state per GPU, whose size is constant in $l$. A2A-SP for SSMs still moves the full token tensor at $O(l \cdot d / N_{\text{SP}})$ per GPU and gains no benefit from the compact state. Thus, while A2A-SP is often preferred for Attention, P2P-SP is the natural fit for SSMs.

\subsubsection{Universal Sequence Parallel (USP) for Sequence-Mixing Layers}
\label{sec: usp}

While P2P SP is communication-efficient, it requires implementing custom autograd functions for state passing across GPUs, along with the mathematical decomposition in \cref{eq:ssm_state_decomp}--\cref{eq:ssm_output_decomp} tailored to each SSM variant. This demands familiarity with the internals of the SSM kernel, which is often implemented in Triton or CUDA and not straightforward to modify.

USP takes a simpler approach that works with any sequence-mixing layer out of the box. Instead of passing states between GPUs, each GPU gathers the full sequence, runs the layer's unmodified forward pass on it, then scatters the output back to the local chunk:
\begin{align}
\mathbf{x}_{\text{full}} &= \textsc{AllGather}(\mathbf{x}_{\text{local}}), \label{eq:usp_gather} \\
\mathbf{y}_{\text{full}} &= f(\mathbf{x}_{\text{full}}), \label{eq:usp_compute} \\
\mathbf{y}_{\text{local}} &= \textsc{Slice}(\mathbf{y}_{\text{full}}), \label{eq:usp_scatter}
\end{align}
where $f$ is the layer's original forward function. $\textsc{AllGather}$ is an NCCL collective that reconstructs the full sequence on each GPU from the local chunks, while $\textsc{Slice}$ extracts the local chunk corresponding to each GPU. This ensures that the full sequence is only materialized within the sequence-mixing layers; feedforward and normalization layers, which typically dominate the activation footprint, continue to operate on local chunks only.\footnote{This is because LLMs typically employ a 3-4x expansion factor in the intermediate activations of their FeedForward layers.} USP trades communication bandwidth and redundant compute for implementation simplicity. Every GPU runs the full SSM kernel on the entire sequence, so no state-passing logic, 1D convolution boundary handling, or output correction is needed. Moreover, USP is not restricted to SSMs, it is applicable to any sequence-mixing layer (GRUs, LSTMs, Sparse Attention variants, etc.) since it treats the layer as a black box. This makes USP ideal for rapid prototyping of long-context Hybrid models with novel sequence-mixing layers. 

\paragraph{Relation to A2A-SP.} USP can be viewed as a simplified variant of A2A-SP that forgoes head distribution: rather than scattering heads across GPUs via A2A collectives, USP replicates the full computation on every GPU via AllGather. This makes USP strictly less efficient in both communication and compute. However, A2A-SP assumes the layer computation decomposes cleanly along a head dimension and that the number of heads is divisible by $N_{\text{SP}}$. Moreover, because each GPU only holds a subset of heads after the A2A, any per-head parameters (e.g., the 1D convolution weights, gating vectors, and normalization parameters) must be sharded to match the local head subset. USP sidesteps all of this: it treats the layer as a black box and requires no modification to parameters or kernels, making it ideal for exploring novel sequence-mixing layers and architecture modifications.

\section{Training Recipes}
\label{sec:training_recipes}
This section describes the training recipes for our instruction-tuned and reasoning Hybrid models. All recipes share a common Stage~1 alignment phase (\cref{training_recipes:stage1}), after which the pipelines diverge into task-specific Stage~2 recipes: instruction tuning (\cref{training_recipes:it}) targets general-purpose long-context instruction following, while the reasoning recipe (\cref{subsec:long-context-reasoning}) targets chain-of-thought reasoning.

\subsection{Stage~1: Alignment}
\label{training_recipes:stage1}

Stage~0 (\cref{subsec:stage0-theory}) initializes our Hybrid model's SSM layers from the source Transformer's Attention weights, but this correspondence is approximate. 
Stage~1 closes the performance gap with the source Transformer via end-to-end MSE alignment (\cref{eq:e2e-loss}), using a memory-efficient fused architecture (\cref{subsec:fused-arch}) that avoids loading separate teacher and student models. In this stage, only the SSM-specific parameters are trained while all other parameters are kept frozen.

We run Stage~1 at 8K context length using a curated mix of approximately 40B tokens. Our training data is predominantly web text, supplemented with mathematical reasoning data and a smaller proportion of general instruction-following samples (\cref{fig:stage1_data}). This mix exposes the SSM layers to the source Transformer's behavior across diverse domains during alignment. Training sequences are constructed by packing multiple samples into 8K length sequences using \texttt{<|endoftext|>} as the delimiter between samples. \Cref{tab:stage1_hparams} summarizes the training hyperparameters used for this stage.

\begin{figure}[h]
\centering
\begin{subfigure}[b]{0.52\textwidth}
  \centering
  \includegraphics[width=\textwidth]{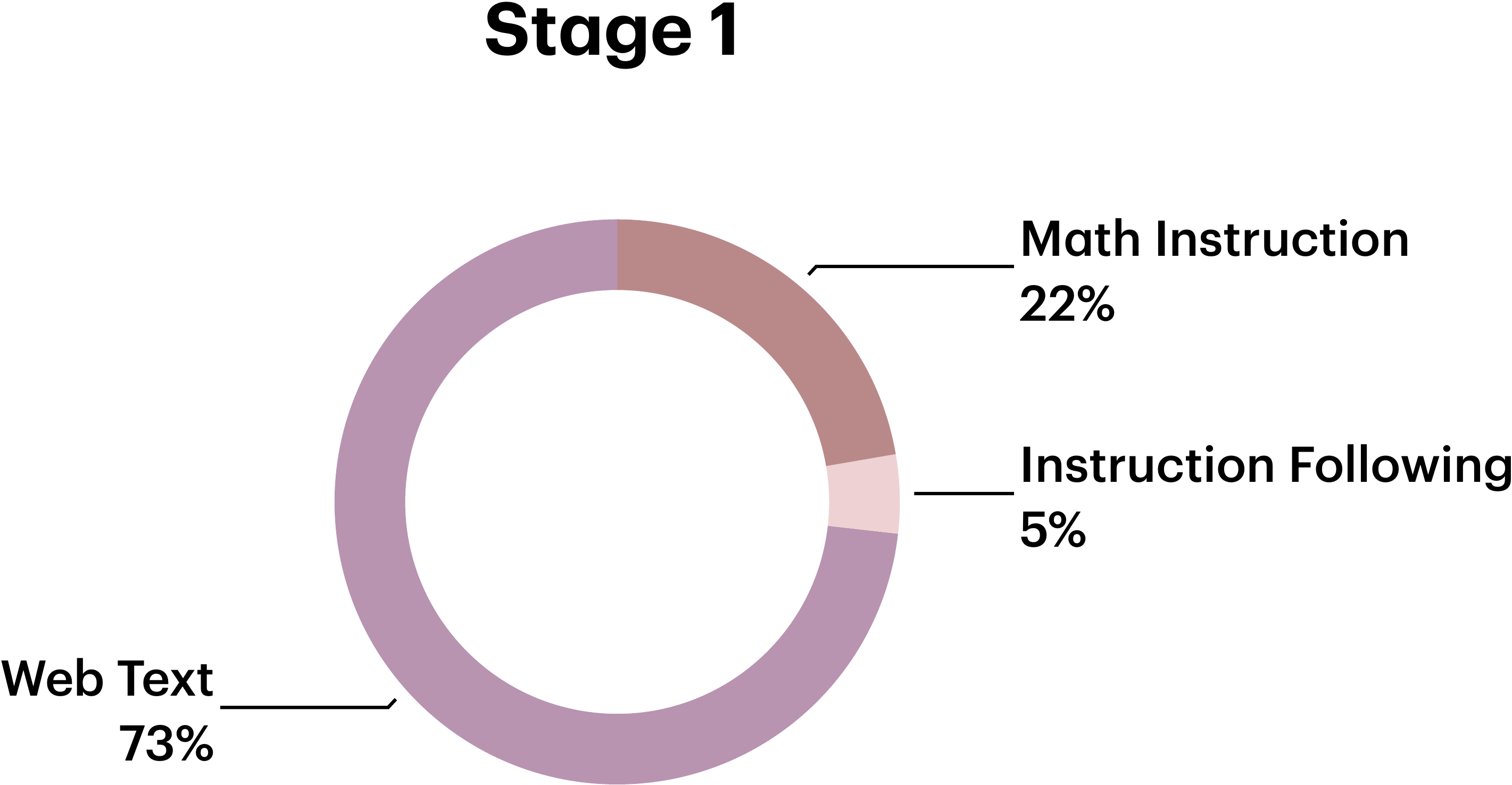}
  \caption{Stage~1 data distribution by domain.}
  \label{fig:stage1_data}
\end{subfigure}
\hfill
\begin{subfigure}[b]{0.44\textwidth}
  \centering
  \small
  \begin{tabular}{lc}
  \toprule
  Context length & 8K \\
  \# training tokens & ${\sim}$40B \\
  Global batch size (tokens) & 2M \\
  Learning rate & $1 \times 10^{-4}$ \\
  LR schedule & constant \\
  Warmup ratio & 4\% \\
  \bottomrule
  \end{tabular}
  \caption{Stage~1 training hyperparameters.}
  \label{tab:stage1_hparams}
\end{subfigure}
\caption{\textbf{Stage~1 alignment recipe.} Data composition by domain (left) and training configuration (right). Stage~1 models are trained on 40B tokens consisting of a mix of instruction-following and web text.}
\label{fig:stage1_recipe}
\end{figure}

\subsection{Instruction Tuning}
\label{training_recipes:it}

Starting from a Hybrid model that has undergone Stage~0 initialization (\cref{subsec:stage0-theory}) and Stage~1 alignment (\cref{training_recipes:stage1}), we describe our Stage 2 recipe for producing long-context instruction-tuned models. The recipe consists of two sequential phases, each initializing: (i) \textit{long-context} continued pre-traing followed by (ii) \textit{SFT} on instruction-tuning data.

\paragraph{Data construction.}
For the long-context phase, training sequences are constructed by packing multiple samples into fixed-length sequences using end-of-sequence tokens as delimiters.\footnote{For our Primed Hybrid models, this is the \texttt{<|endoftext|>} token from the Qwen3 tokenizer.}
Since individual documents can span >$100K$ tokens, we use best-fit packing~\citep{bestfitpacking} to minimize truncations and maximize utilization of the 128K context window. For the SFT phase, as is common practice, training sequences are constructed by first tokenizing individual samples using a chat template\footnote{We use the source Transformer's chat template for our experiments.} and then packing them into fixed-length training sequences using LlamaFactory's greedy packing algorithm \citep{zheng2024llamafactory}.

\begin{figure}[h]
\centering
\includegraphics[width=0.8\textwidth]{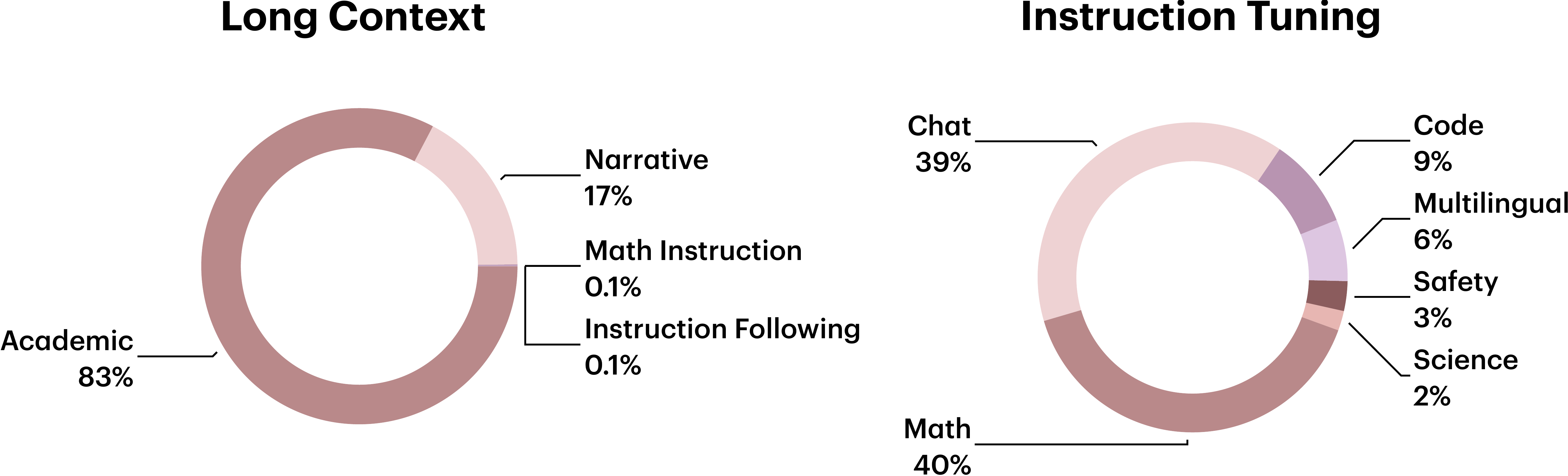}
\caption{\textbf{Instruction model training data distribution.}
Domain composition by token count for long-context continued pre-traing (left) and instruction-tuning SFT (right). ``Safety'' refers to samples that train appropriate refusal behavior and robustness to adversarial prompts.}
\label{fig:data_it_dist}
\end{figure}

\paragraph{Long-context continued pre-training.}
To enable a native long context (128K context tokens) for our Instruction-Tuned (IT) models, we perform continued pre-training on a mix of long documents and shorter instruction-following samples (\cref{fig:data_it_dist}, left). Long documents, distributed across GPUs via Sequence Parallelism (\cref{app:sp_for_hybrids}), are necessary to teach the model to attend over the full 128K window. 
In our early experiments, we found that packed short-context data alone was insufficient to induce long-range recall capabilities. Additionally, including a small amount of instruction-tuning data ($<$0.2\%) prevents degradation of short-context capabilities during context extension. To support the extended context length, we increase the RoPE base frequency from 1M (the source Transformer's default) to 5M~\citep{xiong2024effective}.

\paragraph{SFT for instruction tuning.}
This phase fine-tunes our long-context Hybrid models on instruction-following data spanning math, chat, code, safety, and multilingual data (\cref{fig:data_it_dist}, right). The training objective is next-token prediction, with loss computed only on response tokens (prompt tokens are masked). Our instruction-following data is predominantly short-context: over 95\% of samples have fewer than 32K tokens. This has implications for our Primed models. Carrying out SFT at 32K context length causes a 3--6\% regression (on 8B models) in long-context performance relative to the checkpoint obtained after the preceding long-context phase. On the other hand, SFT at 128K recovers long-context performance, but introduces a 2--2.5\% relative drop on short-context math, coding, and knowledge benchmarks. To retain the strengths of both, we average the parameters of the two checkpoints (model souping ~\citep{wortsman2022model}) to obtain our final 8B models. Interestingly, for our 32B Primed models, performing SFT at 128K on the same data is sufficient and introduces no regression on short-context tasks. We therefore train our 32B models at 128K only, without a separate 32K SFT run. \Cref{tab:it_hparams} summarizes the training hyperparameters for our long context and SFT phases.
\\~\\
The total token count across all stages of our long context IT pipeline is approximately 80B tokens.

\begin{table}[h]
\centering
\caption{\textbf{Instruction-tuning training hyperparameters.} Long-context continued pre-training of our Primed models share the same configuration across 8B and 32B scales. The SFT phase is split by scale: at 8B, we train separate 32K and 128K checkpoints that are later souped, while at 32B we train a single 128K run with a lower learning rate.}
\label{tab:it_hparams}
\small
\begin{tabular}{lccc}
\toprule
& \textbf{Long-Context} & \textbf{SFT (8B)} & \textbf{SFT (32B)} \\
\midrule
Context length       & 128K   & 32K / 128K & 128K \\
\# training tokens      & 28B & 1.6B & 1.6B \\
Global batch size (tokens) & 6M & 2M  & 4M \\
Learning rate        & $5 \times 10^{-5}$ & $5 \times 10^{-5}$ & $5 \times 10^{-6}$ \\
LR schedule          & cosine & linear & linear \\
Warmup ratio         & 5\%    & 3\%     & 3\% \\
\bottomrule
\end{tabular}
\end{table}

\subsection{Long-Context Reasoning}
\label{subsec:long-context-reasoning}

Starting from a Hybrid model that has undergone Stage~0 initialization (\cref{subsec:stage0-theory}) and Stage~1 alignment (\cref{training_recipes:stage1}), we describe our Stage~2 recipe for producing long-context reasoning models. The recipe consists of three sequential phases, each initialized from the final checkpoint of the previous phase: (i) \textit{short-context SFT} on reasoning data at 32K context length, (ii) \textit{context extension} to 128K on long-context data, and (iii) \textit{long instruction-alignment} to align the model to the instruct-format seen during deployment.  \cref{tab:reasoning_hparams} summarizes the training hyperparameters for each phase of the long-context reasoning pipeline.

\begin{table}[ht]
\centering
\caption{\textbf{Long-context reasoning training hyperparameters.} We report the hyperparameters for each phase of our multi-stage reasoning pipeline described in \cref{subsubsec:multi-stage-reasoning-training}.}
\label{tab:reasoning_hparams}
\small
\begin{tabular}{lccc}
\toprule
& \textbf{Short-ctx SFT} & \textbf{Context Extension} & \textbf{Long Instruct-Alignment} \\
\midrule
Context length       & 32K & 128K & 128K \\
\# training tokens   & 100B & 10B/25B\textsuperscript{$\dagger$} & 3B \\
Global batch size (tokens) & 4M & 6M  & 6M \\
Learning rate        & $5 \times 10^{-5}$ & $5 \times 10^{-5}$ & $5 \times 10^{-5}$ \\
LR schedule          & cosine & cosine & cosine \\
Warmup ratio/steps         & 10\%    & 200 steps  & 100 steps \\
\bottomrule
\end{tabular}
\smallskip
{\footnotesize \\
  \textsuperscript{$\dagger$}10B tokens for 32B model scale; 25B tokens for 8B model scale.
}
\end{table}

\subsubsection{Multi-stage Training For Long-context Reasoning}\label{subsubsec:multi-stage-reasoning-training}

\begin{figure}[ht]
\centering
\includegraphics[width=0.8\textwidth]{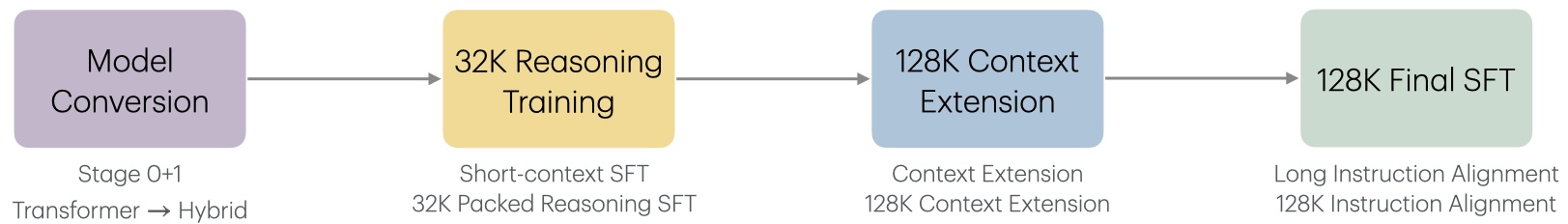}
\caption{\textbf{Our multi-stage long-context reasoning pipeline.}  Stages 0+1 of the Priming procedure allow our Hybrid architectures to sidestep pre-training costs by leveraging a pre-trained Transformer model. Stage 2 for reasoning consists of three phases. (i) \textit{Short-context SFT} consists of SFT on packed reasoning examples at 32K context. Next, (ii) a \textit{context-extension} phase where we train on a mix of reasoning samples and long-context documents to extend the context from 32K to 128K.  Finally, (iii) we perform a \textit{long instruction-alignment} phase where we train on packed reasoning samples at 128K context to align the model to the instruct-format seen during deployment.}
\label{fig:reasoning-recipe}
\end{figure}

Figure~\ref{fig:reasoning-recipe} depicts our multi-stage reasoning pipeline. Next we describe in detail the three phases for our Stage 2 recipe for reasoning models.

\paragraph{Short-context SFT.} We train on packed reasoning traces at 32K context, filtering conversations to $\leq$32K tokens to avoid truncation. Samples are concatenated using simple example packing with no Attention masking or state reset (for SSM layers) between sequences.

\paragraph{Context-extension.}  This phase extends the context length of our Primed Hybrid models from 32K to 128K. We train on a mixture of reasoning traces in chat SFT format and long documents in pre-training format\footnote{Pre-training format is tokenization of samples without any chat template, that is, no prompt-response chat structure or any special role tokens.} at 128K context length. The model is trained using a heterogeneous loss where tokens corresponding to reasoning traces receive the SFT loss by masking out prompt tokens, whereas the tokens belonging to the long documents obtain the next-token-prediction loss (without any masking, as is standard in pre-training). Mixing pre-training and chat SFT data requires separate end-of-sequence delimiters for the two data-types. Without this
separation, the model may output text that mimics pre-training documents or produce excessively long reasoning chains for simple problems. In our experiments, we use the  \texttt{<|endoftext|>} token from the Qwen3 tokenizer as the delimiter for pre-training documents and \texttt{<|im\_start|>}
\& \texttt{<|im\_end|>} delimiters (between turns) for SFT reasoning traces.

To ensure sufficient coverage of long samples in the training mix, longer reasoning traces are up-sampled via a length-weighted formula. Specifically, given $N$ different document sources each with a sampling probability $p_i$, let $\bar{L}_i$ denote the average conversation length for document source $i$ and let $\bar{L}$ denote the probability-weighted average conversation length across all sources.  We modify the original sampling probability $p_i$ of source $i$ to get modified sampling probability $\rho_i$:
\[ \rho_i = \frac{p_i\bar{L}_i }{ \sum_{j=1}^N p_j \bar{L}_j} = p_i \Big(\frac{\bar{L}_i}{\bar{L}}\Big). \]
This has the effect of up-weighting document sources that have higher average length $\bar{L}_i$ than the global average $\bar{L}$.

\paragraph{Long instruction alignment.}  Since the context-extension phase consists of a heterogeneous mix of chat SFT data and pre-training data, we perform a final training stage to align the model to the chat format it will be required to have during deployment.  In this final long-instruction-alignment phase, we perform SFT on packed reasoning traces at 128K context.
\\~\\
The total token count across all stages of our reasoning pipeline is approximately 150B tokens.

\subsubsection{Reasoning Data Mix Construction}
\label{subsubsec:data_mix}

\begin{figure}[ht]
\centering
\includegraphics[width=0.8\textwidth]{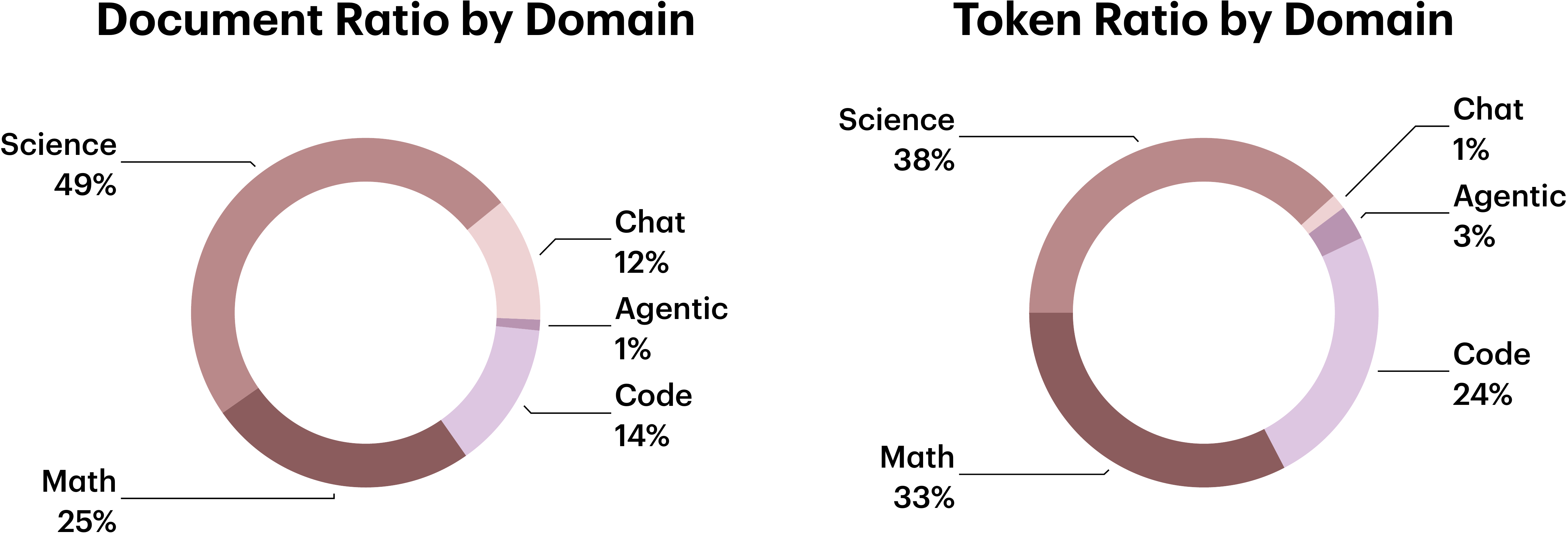}
\caption{We depict the document and token ratios of our reasoning SFT data mixture by domain.}
\label{fig:reasoning-sampling-ratios}
\end{figure}

Our reasoning traces consist of curated chain-of-thought traces across various domains, including math, coding, agentic/tool-calling, science, and chat/instruct.  The document and token ratios by domain are depicted in \cref{fig:reasoning-sampling-ratios}.  Constructing a data mixture of high-quality reasoning traces requires several filtering and balancing steps, each addressing a specific failure mode. We describe these below.

\paragraph{Filtering out degenerate reasoning traces}
Degenerate reasoning traces like loops, stutters and repetitive chains of thought that would otherwise decrease the signal-to-noise ratio in training are filtered our using repeated n-gram filtering within a sliding window.

\paragraph{Controlling the tool use frequency}
In the final data mix, $70$--$80\%$ of samples with tool definitions actually
invoke tools.  Our raw reasoning math data contained Python tool usage without
corresponding definitions in the system prompt. These were corrected by inserting appropriate tool definitions, not doing so causes either tool hallucination or tool use in situations where it is not appropriate. 

\paragraph{Invalid Python use filtering.} A small fraction of math-domain samples ($\sim$0.89\%) contained code blocks despite no code-execution tool being defined in the conversation. This caused the model to answer math problems with \texttt{print()} statements rather than mathematical reasoning. These erroneous samples were filtered out.

\paragraph{Random packing over greedy packing.}
Common implementations of greedy (best-fit) packing bias longer reasoning traces toward the beginning
of the context window (on average, the first reasoning trace in a sample was $\sim$1K tokens longer).  Random packing
removes this positional bias, which is important for heterogeneous data mixes
where trace lengths vary widely.

\paragraph{Packed reasoning is insufficient for long-context.}

Training on packed reasoning samples at 128K context, where individual
reasoning traces max out at ${\sim}32$K tokens, causes long-context performance to severely degrade beyond 32K.  Despite being exposed to 128K tokens, the model has no incentive to retain memory in the state across document boundaries when trained in this fashion.  Consequentially, long documents and reasoning traces must be included in the training mix to induce long-context recall capabilities.  This finding motivates the dedicated context extension phase described in \cref{subsubsec:multi-stage-reasoning-training}.

\section{Inference and Deployment}
\label{sec:inference}

Hybrid models have a fundamentally different memory structure from Transformers.  A Transformer's sequence-mixing layers are exclusively eidetic: every layer maintains a KV cache that stores each token's representation exactly, at a cost that grows linearly with context length.  A Hybrid model mixes eidetic layers (Attention) with fading memory layers (SSMs), whose recurrent state is fixed-size regardless of context length (\cref{sec:hybrid-arch}).  This structural difference induces different per-layer costs with respect to context length, and its impact on inference efficiency becomes concrete when serving a Primed Hybrid model alongside its source Transformer.  We use vLLM~\citep{vllm} as our production inference engine: its plugin system allows us to register custom model architectures without forking the codebase, which is essential for supporting the novel layer types introduced by Priming while continuing to benefit from upstream improvements.\footnote{Our vLLM plugin registers Primed Hybrid architectures via vLLM's entry-point system, preserving continuous batching, chunked prefill, prefix caching, and tensor parallelism over both Attention and SSM layers.}

The dynamics differ between the two phases of autoregressive serving. Prefill, processing the input prompt, is primarily compute-bound; our Primed Hybrid models see moderate gains here by replacing quadratic Attention computations with linear-cost SSM recurrences. Decode, generating tokens one at a time, is primarily memory-bandwidth-bound, dominated by reading KV cache entries from GPU memory; this is where the Primed Hybrid's reduced memory footprint
translates most directly into throughput. The decode phase dominates wall-clock time in the workloads we
target. In mathematical reasoning, long chains of thought produce far more output tokens than the input prompt.  In agentic coding, conversations are multi-turn with large tool outputs, but prefix caching ensures that only newly appended tokens require prefilling each turn; cumulative decode, the total time spent generating tokens across a session, far exceeds cumulative prefill.  As a reference point, a single full 128K-context prefill on Qwen3-32B costs roughly the same wall-clock time as generating just ${\sim}1{,}500$ tokens at that context length, a small fraction of the cumulative decode in a typical reasoning trace or multi-turn agentic session.

\paragraph{A model of decode throughput.}
One of the principal benefits of Hybrid models is their inference efficiency, inherited directly from the SSM memory layers described in \cref{subsec:ssm-zoo}.  Before deciding which SSM type to instantiate and which Hybrid ratio to target via Priming, it is useful to have a simple predictive model of the expected throughput improvement.  We develop such a model under KV-cache saturation, the production serving regime in which the GPU's available memory is fully occupied by KV cache entries, so that serving additional requests requires evicting existing state or queuing, and show that it explains the pattern of speedups reported in \cref{tab:decode_32b}--\cref{tab:decode_8b}.

Consider a Transformer serving requests at KV-cache saturation on a fixed GPU: $b$~concurrent sequences, each with context length~$l$, and~$n$ Attention layers.  Each decode step produces one token per sequence, so the model produces~$b$ tokens per forward pass.  At saturation the KV cache is proportional to~$b \cdot n \cdot l$, and the product~$bl$ is near constant, doubling the context length requires approximately halving the batch size.

Let~$r$ denote the Hybrid ratio, i.e.\ the fraction of the~$n$ sequence-mixing layers being SSM layers.  Our Primed Hybrid models have~$rn$ SSM layers whose recurrent state is fixed-size (i.e.~independent of~$l$) and~$(1{-}r)n$ Attention layers. At the Transformer's saturated batch size~$b$, the Hybrid's KV cache occupancy is only~$b \cdot (1{-}r)n \cdot l$, a fraction~$(1{-}r)$ of the total budget, leaving a fraction~$r$ unused.  Filling this headroom allows the Hybrid to serve a larger batch~${\approx}\,b/(1{-}r)$ on identical hardware.  The \emph{aggregate} KV cache traffic per decode step remains approximately unchanged:
\begin{equation}
  \underbrace{\frac{b}{1-r} \cdot (1{-}r)n \cdot l}_{\text{Hybrid}}
  \;\approx\; \underbrace{b \cdot n \cdot l}_{\text{Transformer}}
  \label{eq:kv-traffic}
\end{equation}
The approximation neglects the cost of reading each SSM layer's fixed-size recurrent state, equivalent to at most~$2$K tokens of KV cache (for our Primed Hybrid models), which becomes negligible at modest context lengths. While the aggregate memory utilization is approximately equal, the two architectures distribute this traffic differently across layers: the Transformer makes~$n$ Attention calls each serving batch~$b$, while the Primed Hybrid makes~$(1{-}r)n$ Attention calls each serving the larger batch.

A decode forward pass comprises three main cost components.  Writing $b' = b/(1{-}r)$ for the Primed Hybrid batch size:
\begin{enumerate}[nosep]
  \item \emph{Attention layers}, dominated by KV cache reads;
    per-layer cost depends on both batch size and context length.  The
    Transformer has~$n$ such layers, our Primed Hybrid has~$(1{-}r)n$.
  \item \emph{MLP layers}, dominated by weight reads; cost scales
    with batch size, independent of context length.  All~$n$ MLP layers
    are present in both architectures.
  \item \emph{SSM layers}, state update cost scales with batch size,
    independent of context length.  Present only in our Primed Hybrid
    ($rn$~layers).
\end{enumerate}
Let~$t_{\mathrm{attn}}(b, l)$, $t_{\mathrm{mlp}}(b)$, and~$t_{\mathrm{ssm}}(B)$ denote the per-layer wall-clock time of a single Attention, MLP, and SSM layer respectively.  The aggregate forward-pass times for the source Transformer ($T_{\mathrm{ST}}$) and our Primed Hybrid ($T_{\mathrm{PH}}$) are then:
\begin{align}
  T_{\mathrm{ST}}   &= n\,t_{\mathrm{attn}}(b, l)
                        + n\,t_{\mathrm{mlp}}(b),
                        \label{eq:T-st} \\
  T_{\mathrm{PH}} &= (1{-}r)n\,t_{\mathrm{attn}}(b', l)
                        + n\,t_{\mathrm{mlp}}(b')
                        + rn\,t_{\mathrm{ssm}}(b').
                        \label{eq:T-ph}
\end{align}
The Transformer produces~$b$ tokens per step and our Primed Hybrid produces~$b'$.  The decode throughput ratio is therefore
\begin{equation}
  R \;=\;
  \frac{b'}{b}\,\frac{T_{\mathrm{ST}}}{T_{\mathrm{PH}}}
  \;\approx\;
  \frac{1}{1{-}r}\,
  \frac{t_{\mathrm{attn}}(b, l) + t_{\mathrm{mlp}}(b)}
       {(1{-}r)t_{\mathrm{attn}}(b', l) + t_{\mathrm{mlp}}(b') + r t_{\mathrm{ssm}}(b')}.
  \label{eq:throughput-ratio}
\end{equation}
\paragraph{Context-length dependence.}
At large~$l$, each Attention layer reads a KV cache proportional to~$l$, while MLP weight reads and SSM state updates remain fixed regardless of~$l$.  The per layer Attention cost~$t_{\mathrm{attn}}$ therefore dominates the forward pass, and the MLP and SSM contributions become negligible.  In this regime the ratio simplifies to
\begin{equation}
  R \;\xrightarrow{}\;
  \frac{1}{1{-}r}\,
  \frac{t_{\mathrm{attn}}(b, l)}
       {(1{-}r)t_{\mathrm{attn}}(b', l)}
  \;=\;
  \frac{1}{(1{-}r)^2}\,
  \frac{t_{\mathrm{attn}}(b, l)}{t_{\mathrm{attn}}(b', l)}.
  \label{eq:throughput-limit}
\end{equation}
The limiting throughput ratio is determined entirely by how the single-layer Attention forward pass scales with batch size. If~$t_{\mathrm{attn}}$ scales linearly in~$b$ (scaling the batch by $1/(1{-}r)$ scales the time equally), then~$R \to 1/(1{-}r)$.  If the scaling is sub-linear, the Attention kernel becomes more efficient at larger batch sizes, for example through better HBM bandwidth utilization, then~$t_{\mathrm{attn}}(b', l) < t_{\mathrm{attn}}(B, L)/(1{-}r)$ and~$R > 1/(1{-}r)$.  Conversely, super-linear scaling would yield~$R < 1/(1{-}r)$. All Primed Hybrid models in this paper use $r = 1/2$, so $b' \approx 2b$ and the linear-scaling baseline is $R \to 1/(1{-}r)=2$. Our implementation uses FlashAttention-3~\citep{flashattention3} for all Attention layers; the empirical observation of~$R = 2.26\times$ at 128K indicates sub-linear scaling of this kernel at the operating points tested.

At short~$l$, the batch~$b$ is large and the MLP and SSM terms are no longer negligible, they add cost to~$T_{\mathrm{PH}}$ without a corresponding increase in~$T_{\mathrm{ST}}$, pulling~$R$ below the Attention-determined limit.  This explains the observed pattern: $R = 1.86\times$ at 16K growing to~$R = 2.26\times$ at 128K for our 8B GKA-Primed-HQwen3 model (\cref{tab:decode_8b}).  At very short contexts ($l \lesssim 4$K), $R$ can drop below~$1$ due to two compounding effects: the SSM layers carry fixed-size recurrent state equivalent to at most $2$K tokens of KV cache, which at short~$l$ erodes the batch-size advantage; and FlashAttention-3's highly optimized short-context kernels give the source Transformer a lower per-token decode cost at these short contexts.  In practice this regime is negligible: the workloads we target quickly accumulate context well beyond $4$K, after which the Primed Hybrid's throughput advantage takes hold and grows with context length.

We note that these gains are a function of the Attention kernel implementations and hardware.  Future kernel improvements on efficiency at large batch sizes, strengthening the sub-linear trend, would directly increase the Primed Hybrid throughput advantage.

\subsection{Decode Throughput}
\label{subsec:decode-throughput}

\paragraph{32B models.}
\cref{tab:decode_32b} reports sustained decode throughput on
$8\times$~H200 GPUs (TP=8) under KV-cache saturation. GDN-Primed-HQwen3 reaches
$2.11\times$ at 128K, and GKA-Primed-HQwen3, whose symmetric-tiled decode kernel is detailed in \cref{kernel}, reaches $2.02\times$ at 128K at the default
\texttt{num\_iter}=30, and $2.06\times$ at \texttt{num\_iter}=10.  GKA's throughput can be further modulated via the variable test-time compute tradeoff described in \cref{subsec:chebyshev}.

\begin{table}[t]
\centering
\caption{%
  Measured decode throughput (tokens/s) for our 32B Primed Hybrid models on $8\times$ H200
  (TP=8).  Speedup over Qwen3-32B in parentheses.
}
\label{tab:decode_32b}
\small
\begin{tabular}{lcccc}
\toprule
\textbf{Model} & \textbf{16K} & \textbf{32K} & \textbf{64K} & \textbf{128K} \\
\midrule
Source Transformer (Qwen3-32B) & 5{,}299 & 2{,}865 & 1{,}308 & 586 \\
GKA-Primed-HQwen3 (\texttt{num\_iter}=30) & 7{,}172\;(1.35$\times$) & 4{,}253\;(1.48$\times$) & 2{,}448\;(1.87$\times$) & 1{,}187\;(2.02$\times$) \\
GKA-Primed-HQwen3 (\texttt{num\_iter}=10) & 7{,}922\;(1.50$\times$) & 4{,}565\;(1.59$\times$) & 2{,}556\;(1.95$\times$) & 1{,}206\;(2.06$\times$) \\
GDN-Primed-HQwen3 & 8{,}133\;(1.53$\times$) & 4{,}876\;(1.70$\times$) & 2{,}688\;(2.06$\times$) & 1{,}238\;(2.11$\times$) \\
\bottomrule
\end{tabular}
\end{table}

\paragraph{8B models.}
\cref{tab:decode_8b} reports the same benchmark at 8B scale.  All three pure-SSM-based Primed Hybrid types, GKA, GDN, and Mamba2, show the pattern predicted by Equation~\ref{eq:throughput-ratio}: speedups grow monotonically with context length, from $1.86$--$1.95\times$ at 16K to $2.26$--$2.33\times$ at 128K.  GKA-Primed-HQwen3 is approximately $5$--$10\%$ slower than GDN- and Mamba2-based Hybrids across all context lengths, reflecting the additional compute of the Chebyshev solve within each GKA layer; reducing \texttt{num\_iter} narrows this gap (\cref{subsec:chebyshev}).

\begin{table}[t]
\centering
\caption{%
  Measured decode throughput (tokens/s) for our 8B Primed Hybrid models on $8\times$ H200 (TP=8) under KV-cache saturation.  Speedup over the Qwen3-8B Transformer baseline shown in parentheses.
}
\label{tab:decode_8b}
\small
\begin{tabular}{lcccc}
\toprule
\textbf{Model} & \textbf{16K} & \textbf{32K} & \textbf{64K} & \textbf{128K} \\
\midrule
Source Transformer (Qwen3-8B) & 8{,}951 & 5{,}174 & 2{,}740 & 1{,}227 \\
GKA-Primed-HQwen3 (\texttt{num\_iter}=30) & 16{,}613\;(1.86$\times$) & 9{,}341\;(1.81$\times$) & 5{,}230\;(1.91$\times$) & 2{,}772\;(2.26$\times$) \\
GKA-Primed-HQwen3 (\texttt{num\_iter}=10) & 17{,}368\;(1.94$\times$) & 9{,}745\;(1.88$\times$) & 5{,}366\;(1.96$\times$) & 2{,}811\;(2.29$\times$) \\
GDN-Primed-HQwen3 & 17{,}479\;(1.95$\times$) & 10{,}080\;(1.95$\times$) & 5{,}521\;(2.01$\times$) & 2{,}863\;(2.33$\times$) \\
Mamba2-Primed-HQwen3 & 16{,}844\;(1.88$\times$) & 9{,}966\;(1.93$\times$) & 5{,}460\;(1.99$\times$) & 2{,}825\;(2.30$\times$) \\
BMOJOF-Primed-HQwen3 & 7{,}854\;(0.88$\times$) & 5{,}597\;(1.08$\times$) & 3{,}573\;(1.30$\times$) & 2{,}153\;(1.75$\times$) \\
\bottomrule
\end{tabular}
\end{table}

\paragraph{B'MOJO-F inference.}
\label{subsec:bmojof-inference}
B'MOJO-F layers (\cref{sec:hybrid-arch}) combine an SSM module with Sliding Window Attention in a single sequence-mixing \emph{Hybrid layer}.  The current inference implementation issues two separate Flash Attention kernel calls, one for the eidetic (SWA) stream and one for the fading (SSM)
stream, and merges their outputs subsequently.  While mathematically equivalent to Attention over the union of both KV sets, this dual-call implementation leaves efficiency on the table: at 16K context, B'MOJO-F is $0.88\times$ thant the Transformer.  The crossover occurs at approximately 32K, beyond which the KV cache savings outweigh the per-layer overhead, reaching $1.75\times$ at 128K (\cref{tab:decode_8b}).  Fusing the two streams into a single Attention kernel call, providing both eidetic and fading stream KV as joint input to one softmax computation, would eliminate the double-call overhead and is expected to shift the crossover for B'MOJO-F to shorter contexts.

\subsection{Time to First Token (TTFT)}
\label{subsec:ttft}

The decode throughput analysis above concerns token generation, which is memory-bandwidth-bound.  Prefill, processing the input prompt before the first generated token, is a different regime: it is primarily \emph{compute-bound}, with Attention's $O(l^2)$ FLOP cost per layer as the dominant term.  Our Primed Hybrid models replace $n/2$ of these quadratic Attention computations with SSM layers whose chunked recurrence scales linearly in~$l$.  As context length grows, this quadratic-vs-linear gap widens, and the TTFT advantage increases.

While the theoretical cost advantage is clear, replacing $O(l^2)$ Attention with $O(l)$ SSM recurrence, two practical factors work against TTFT at short contexts.  First, the Priming procedure (\cref{subsec:stage0-theory}) introduces additional parameters in each SSM layer (gate projections, convolution weights, etc.) that are not present in the source Transformer, increasing the total weight-read cost during prefill.  Second, the superior efficiency of FlashAttention-3 in handling the quadratic Attention compute means that the practical gap between Attention and SSM prefill is smaller than the theoretical complexity suggests, FlashAttention-3 reflects sustained
long-term community investment in saturating GPU hardware specifically for autoregressive (causal) Attention, whereas SSM kernels have not yet received  omparable optimization effort in scale or time. Together, these effects imply that TTFT is comparable to or worse than the Transformer at shorter context lengths, with the advantage emerging only at longer context lengths where the lower compute cost of SSM recurrence eventually outweighs the increased parameter count and FlashAttention-3's efficiency.

All TTFT measurements use the \emph{same} batch size~$b$ for both architectures, set to saturate the Transformer's KV cache, a suitable operating condition for throughput-oriented production deployments.  Our Primed Hybrid models serve this batch with memory to spare but do not increase the batch, isolating the prefill compute benefit from the batching advantage discussed in the decode analysis.

\paragraph{32B models.}
\cref{tab:ttft_32b} reports mean TTFT at 32B scale.  The crossover where our Primed Hybrid models begin to outperform the Transformer on TTFT occurs around 64K. GKA-Primed-HQwen3 reaches $0.90\times$ at 128K with \texttt{num\_iter}=30 ($10\%$ faster), improving to $0.86\times$ at \texttt{num\_iter}=10. GDN-Primed HQwen3 reaches $0.77\times$ ($23\%$ faster) at 128K.

\begin{table}[t]
\centering
\caption{%
  Measured mean TTFT (ms) for our 32B Primed Hybrid models on
  $8\times$~H200 (TP=8).  Batch size set to saturate the Transformer's
  KV cache; ratio vs.\ Qwen3-32B in parentheses (lower is better).
}
\label{tab:ttft_32b}
\small
\begin{tabular}{lcccc}
\toprule
\textbf{Model} & \textbf{16K} & \textbf{32K} & \textbf{64K} & \textbf{128K} \\
\midrule
Source Transformer (Qwen3-32B) & 39{,}421 & 48{,}527 & 65{,}104 & 94{,}479 \\
GKA-Primed-HQwen3 (\texttt{num\_iter}=30) & 52{,}053\;(1.32$\times$) & 58{,}613\;(1.21$\times$) & 68{,}241\;(1.05$\times$) & 84{,}935\;(0.90$\times$) \\
GKA-Primed-HQwen3 (\texttt{num\_iter}=10) & 48{,}560\;(1.23$\times$) & 55{,}039\;(1.13$\times$) & 64{,}766\;(0.99$\times$) & 81{,}410\;(0.86$\times$) \\
GDN-Primed-HQwen3 & 42{,}492\;(1.08$\times$) & 48{,}417\;(1.00$\times$) & 57{,}525\;(0.88$\times$) & 73{,}145\;(0.77$\times$) \\
\bottomrule
\end{tabular}
\end{table}

\paragraph{8B models.}
\cref{tab:ttft_8b} reports the same measurement at 8B scale. At 128K context, GDN-Primed-HQwen3 and Mamba2-Primed-HQwen3 achieve $0.74\times$ the Transformer's TTFT ($26\%$ faster prefill), confirming the benefit of replacing quadratic Attention with linear-cost SSM recurrences at long contexts. GKA Primed-HQwen3 achieves $0.85\times$ at the default \texttt{num\_iter}=30, with the gap to GDN-/Mamba2-based Hybrids attributable to the Chebyshev solver overhead during prefill; reducing to \texttt{num\_iter}=10 improves this to $0.82\times$.  At 16K, the pattern reverses: GKA-Primed-HQwen3 is $1.26\times$ slower than the Transformer, and even GDN-Primed-HQwen3 is only at parity ($1.00\times$), reflecting the short-context penalty discussed above.  B'MOJO-F results are discussed separately in~\cref{subsec:bmojof-inference}.

\begin{table}[t]
\centering
\caption{%
  Measured mean TTFT (ms) for our 8B Primed Hybrid models on $8\times$~H200 (TP=8).  Batch size set to saturate the Transformer's KV cache; ratio vs.\ Qwen3-8B in parentheses (lower is better).
}
\label{tab:ttft_8b}
\small
\begin{tabular}{lcccc}
\toprule
\textbf{Model} & \textbf{16K} & \textbf{32K} & \textbf{64K} & \textbf{128K} \\
\midrule
Source Transformer (Qwen3-8B) & 27{,}736 & 32{,}661 & 42{,}462 & 62{,}922 \\
GKA-Primed-HQwen3 (\texttt{num\_iter}=30) & 35{,}013\;(1.26$\times$) & 38{,}502\;(1.18$\times$) & 44{,}893\;(1.06$\times$) & 53{,}606\;(0.85$\times$) \\
GKA-Primed-HQwen3 (\texttt{num\_iter}=10) & 33{,}008\;(1.19$\times$) & 36{,}334\;(1.11$\times$) & 42{,}076\;(0.99$\times$) & 51{,}404\;(0.82$\times$) \\
GDN-Primed-HQwen3 & 27{,}805\;(1.00$\times$) & 30{,}975\;(0.95$\times$) & 36{,}151\;(0.85$\times$) & 46{,}389\;(0.74$\times$) \\
Mamba2-Primed-HQwen3 & 28{,}668\;(1.03$\times$) & 31{,}405\;(0.96$\times$) & 36{,}666\;(0.86$\times$) & 46{,}618\;(0.74$\times$) \\
BMOJOF-Primed-HQwen3 & 44{,}763\;(1.61$\times$) & 47{,}600\;(1.46$\times$) & 52{,}272\;(1.23$\times$) & 61{,}702\;(0.98$\times$) \\
\bottomrule
\end{tabular}
\end{table}

\paragraph{Limiting factor at long contexts.}
Even at 128K, the TTFT advantage is bounded because the retained $n/2$ Attention layers still dominate prefill time.  Reducing \texttt{num\_iter} has limited additional impact at long contexts for the same reason: the Chebyshev solver in GKA contributes a diminishing fraction of total prefill compute as Attention's quadratic cost grows.

\section{Related Work}
\label{sec:related-work}

Our work builds on and extends several lines of research: distilling Transformers into subquadratic
architectures, and designing hybrid Attention-SSM models. We organize the discussion around these themes and highlight how Priming differs from each prior approach.

\subsection{Attention, Sparse Attention, and State Space Models}

The quadratic cost of Self-Attention~\citep{Vaswani-NeurIPS2017} in sequence length has motivated a large body of work on more efficient alternatives. As models scale to million-token contexts, this cost has become a practical bottleneck for frontier agentic applications in coding, information retrieval, and scientific discovery~\citep{chen2024scienceagentbench, cui2025curie, jimenez2023swe}. Proposed remedies include sparse approximations that restrict each token's Attention to a subset of the sequence, for instance via locality-sensitive hashing~\citep{kitaev2020reformer}, as well as compressed-memory mechanisms that maintain a dynamically updated summary over sliding windows or sequence chunks~\citep{dai2019transformerxl, munkhdalai2024leave, mohtashami2023landmark}. While much of this literature targets sparsity-based reductions~\citep{nunez2024expansion, yuan2025NSA}, in this work we focus on State Space Models for Hybrid architectures.

State Space Models replace the softmax Attention mechanism with constant-size recurrent dynamical systems~\citep{ssd, Yang-ICML2024-gla, beck2024xlstm, yang2024gated, gatedkalmanet}. A wide range of State-Space Model (SSM) variants have been proposed, from designs closely related to Linear Attention~\citep{sun2023retentive} or Linear Time-Invariant systems~\citep{gu2021combining, zancato2022stacked}, to those that introduce adaptive or gated state updates~\citep{Yang-ICML2024-gla, ssd, orvieto2023resurrecting}. Despite their diversity, all SSMs share the working principle of classical state-space models~\citep{Kalman-1960}: the input sequence is processed through a \textit{fixed-size} latent state that serves as a compressed, lossy summary of all tokens seen so far. In finite-precision hardware this state inevitably ``fades away the past'' as more tokens are processed. Efficient SSM implementations exploit hardware-aware primitives such as associative scans~\citep{gu2023mamba, griffin}, chunked computation~\citep{ssd, Yang-ICML2024-gla}, and strategies that avoid materializing the full state in high-bandwidth memory~\citep{gu2023mamba}.

A recurring theme in SSM design is the use of data-dependent gating to modulate the rate at which memory fades, trading off expressivity against scalability. For example, Mamba-2 replaced Mamba's channel-wise gating with head-wise gating for better Tensor Core utilization, and input-dependent gating has been shown to improve training stability more broadly~\citep{arora2023zoology, yang2024gated}, driving the progression from S4~\citep{alber_gu_s4} to Mamba~\citep{gu2023mamba} and from DeltaNet~\citep{yang2024parallelizing} to Gated DeltaNet~\citep{yang2024gated} and Gated KalmaNet \citep{gatedkalmanet}. In particular, Gated KalmaNet shows that such gating arises naturally from solving a weighted least-squares objective, connecting it to the well-studied numerical stability properties of adaptive filtering~\citep{LJUNG_RLS_stability, sayed2003fundamentals, sayed2011adaptive}.

\subsection{Hybrid State-Space Attention Models}
\label{section:appendix_hybrid_intro}

Although scaling the recurrent state of SSM layers has produced strong models, pure SSMs still underperform on tasks that require recalling information from the distant past~\citep{m2h, jelassi2024repeat}. Hybrid State-Space Models address this gap by pairing SSMs' fading memory with Attention layers that provide exact recall~\citep{ssd, griffin, jamba}. Common designs simply interleaved SSM and Attention layers at various ratios~\citep{m2h,gu2023mamba, ssd} or substituted full Attention with Sliding Window Attention~\citep{griffin}. More recent architectures refine this combination~\citep{zamba, bmojo}. B'MOJO~\citep{bmojo} was the first to fuse SSM and SWA \emph{within a single layer}, complementing the SSM's fading state with eidetic memory: within the sliding window, tokens attend to a selected subset of past tokens deemed difficult to predict via an asynchronous causal selection mechanism. Subsequent work~\citep{dong2024hymba, bae2025hybrid} has confirmed that this parallel-fusion design outperforms layer-level stacking at equivalent compute.

Thanks to their reduced memory footprint and favorable scaling over long sequences, Hybrid architectures are increasingly adopted for long-range agentic tasks and have recently been combined with Reinforcement Learning at scale~\citep{chen2025minimax, qwen3codernext2026, nemotron3}. Coupled with system-level optimizations such as prefix caching~\citep{pan2024marconi} and specialized serving engines~\citep{vllm}, Hybrid models can sustain more rollouts per unit time, directly improving end-to-end RL performance.

\subsection{Distilling Transformers into Subquadratic Models}
\label{subsec:rw_distillation}

\paragraph{Mamba in the Llama.}
\citet{mambainllama} demonstrated that pre-trained Transformer models can be distilled into Mamba based models by reusing the linear projection weights from Attention layers to initialize SSM parameters. Their approach uses progressive distillation followed by supervised fine-tuning, requiring ${\sim}20$B tokens.
The resulting model retains a quarter of the original Attention layers and while having a non negligible gap compared to the source Transformer on benchmarks it achieves competitive chat performance.
Priming directly builds on the weight initialization scheme introduced by Mamba in the Llama (our Stage~0, \cref{subsec:stage0-theory}), but extends it in several directions: we generalize the initialization to five SSM layer types beyond Mamba2, introduce a memory-efficient fused architecture for layerwise distillation that scales to 32B parameters, and support multiple Transformer families (Qwen, Llama, Ministral).
Critically, our controlled comparison under identical Priming conditions reveals that the choice of SSM layer matters, GKA consistently outperforms Mamba-2, a finding that was not possible in prior work where each SSM type was evaluated under different conditions.

\paragraph{Mohawk.}
\citet{mohawk} proposed a progressive distillation framework that matches Transformer and SSM representations at three levels of granularity: mixing matrices, hidden states, and end-to-end predictions. Mohawk distilled a Mamba-2 variant from the Phi-1.5 architecture using only 3B tokens.
The key conceptual insight, that both Transformers and SSMs apply different forms of mixing matrices over token sequences, is shared with our realization-theoretic perspective (\cref{sec:why-priming}).
However, Mohawk was demonstrated only at the 1.3B scale with a single SSM type (Mamba-2), whereas Priming scales to 32B and supports five SSM layer types.
Our fused architecture (\cref{subsec:fused-arch}) also provides a more memory-efficient alternative to Mohawk's separate teacher-student forward passes, enabling distillation at scales where maintaining two full model copies would be prohibitive.

\paragraph{RADLADS.}
\citet{radlads} presented a rapid conversion protocol for transforming softmax attention Transformers into purely recurrent SSM-based models, requiring only 350--700M tokens, less than $0.005\%$ of the teacher's pre-training data.
RADLADS introduced two simplified SSM architectures optimized for the conversion setting and released converted Qwen models at 7B, 32B, and 72B scales.
While RADLADS achieves impressive token efficiency for \emph{pure} SSM conversion, Priming targets \emph{hybrid} architectures that retain a fraction of attention layers, yielding stronger absolute quality (particularly on recall-intensive and long-context tasks where pure SSMs struggle). Additionally, Priming's controlled SSM comparison and support for variable-compute layers (GKA) address dimensions not explored by RADLADS.

\paragraph{SUPRA.}
\citet{supra} introduced the idea of \emph{uptraining} a pre-trained Transformer into a recurrent model by replacing softmax Attention with linear Attention and continuing training.
SUPRA demonstrated that this approach can recover competitive performance on standard NLP benchmarks at a fraction of the pre-training cost (${\sim}5\%$ of the original token budget).
However, SUPRA requires ${\sim}100$B tokens of continued training and reports persistent shortfalls in in-context learning and long-context modeling, even at the largest scales.
In contrast, Priming focuses on Hybrid models and achieves better quality with less than $0.5\%$ of the pre-training budget by combining informed initialization, layerwise MSE distillation, and cross-entropy fine-tuning in a structured three-stage pipeline.

\paragraph{Liger.}
\citet{liger} proposed converting Transformers into gated recurrent structures by repurposing the pre-trained key projection weights to construct gating mechanisms, avoiding the introduction of additional parameters.
Combined with LoRA fine-tuning and an intra-layer hybrid Attention mechanism (Liger Attention), this approach recovers ${\sim}93\%$ of the original Transformer's quality at only $0.02\%$ of the pre-training token cost.
Liger's key insight, that gating can be derived from existing weights, is related to our Stage~0 informed initialization, where we map Attention parameters to SSM parameters.
However, Priming goes further by supporting a broader range of SSM architectures (not limited to gated linear recurrences), operating at larger scales (up to 32B), and achieving near-zero quality gaps through the combination of layer-wise alignment and end-to-end fine-tuning.

\paragraph{Layer selection}
The importance of data-driven layer selection is independently recognized by several concurrent efforts. \citet{klguided} find a stark dichotomy: long-context recall is highly sensitive to which layers retain softmax Attention, while short-context reasoning is almost entirely insensitive. \citet{jetnemotron} use a supernet with beam search and find that only two to three layers are critical per task, with the critical set differing across tasks. 
The exact selection mechanisms vary, KL-divergence scoring, supernet search, recall-vs-CSR importance, but all converge on the same conclusion: data driven selection substantially outperforms uniform or random layer placement.
Our importance scoring procedure (\cref{subsec:layer_selection}) confirms this finding.  Unlike the above methods, ours is both training-free, requiring only per-layer SWA mask replacement and evaluation, and motivated by realization theory: SWA directly bounds the effective \ref{eq:max_hankel_rank} of each layer's mixing matrix, connecting the empirical score to a formal measure of how hard the layer is to realize as a fixed-size SSM (\cref{sec:why-priming}).

\subsection{Functional Specialization in Trained Transformers}
\label{subsec:rw_interp}

A central premise of Priming is that trained Transformer layers, despite sharing identical architectures, are functionally heterogeneous: some perform predominantly local mixing while others implement global circuits that are difficult to compress into a fixed-size recurrent state.  We formalize this through realization theory (\cref{subsec. realization theory}), where the effective Hankel rank of each layer's mixing matrix quantifies the difficulty of SSM replacement.
Compatible observations have been made independently in the mechanistic interpretability literature, where trained Transformers are found to develop specialized circuits such as induction heads~\citep{inductionheads} and long-range retrieval patterns~\citep{transformercircuits}.  These findings provide empirical support from a complementary perspective for the layer-level functional heterogeneity that our realization-theoretic analysis predicts and that our importance scoring procedure (\cref{subsec:layer_selection}) exploits.

\section{Limitations}
\label{sec:limitations}

\paragraph{Scope of the SSM hierarchy.}
The full three-way comparison (GKA vs.\ GDN vs.\ Mamba-2) is conducted at 50\% Hybrid ratio, 8B scale, in the Qwen family. At 32B we evaluate only GKA and GDN, since Mamba-2 consistently underperforms at 8B. The GKA advantage over GDN is task-dependent at this scale: the two are competitive on IT benchmarks, while GKA leads on reasoning tasks, consistent with the 8B finding that more expressive state updates matter most on harder tasks requiring long chains of thought. Preliminary experiments suggest the hierarchy generalizes to larger scales and higher Hybrid ratios (up to ${\sim}75\%$ with limited quality regression), but a systematic evaluation across these axes has not been conducted.

\paragraph{Context extension.}
State composition (\cref{subsec:state-composition}) extends native context from 128K to 256K with strong quality retention, and to 512K ($4\times$) with meaningful but degraded performance (e.g., GKA-based 32B Hybrid model drops from 88\% to 61\% accuracy on RULER NIAH between $2\times$ and $4\times$ extension). Whether this degradation can be closed with improved composition strategies (e.g., better chunk-boundary handling, improved merging coefficients, or per-layer adaptive scaling) remains open. Additionally, the evaluation is limited to synthetic retrieval and reasoning tasks (RULER NIAH, BABILong); behavior on more challenging long-document and agentic tasks at these extended lengths is untested.

\paragraph{Inference measurements.}
All throughput numbers are measured on 8$\times$H200. At long context, the Hybrid's speedup over the Transformer is governed by Amdahl's law applied to the Attention fraction of decode time (\cref{sec:inference}): hardware with different compute-to-bandwidth ratios, lower TP configurations where weight-read costs dominate, or short-context workloads where the KV cache is small relative to model weights will shift this fraction and change the observed speedup. In production, servers handle requests of varying lengths concurrently and reuse cached prefixes across turns, both of which change the effective memory pressure and may shift the Hybrid's relative advantage compared to the uniform single-length regime we benchmark.

\paragraph{Quantization.}
All models in this work are served in BF16 with the recurrent states kept in FP32 for numerical stability (e.g.~GKA's state is kept in FP32 to preserve numerical stability of the Chebyshev solver). We do not explore weight quantization (e.g., GPTQ, AWQ) or KV-cache quantization (FP8, INT4). Weight quantization reduces the memory footprint of both Transformers and Hybrids, but may interact differently with SSM layers whose recurrent dynamics are more sensitive to precision than Attention layers. KV-cache quantization directly compresses the Transformer's main memory bottleneck, potentially narrowing the Hybrid's relative advantage at long context. Conversely, lower-precision weights would make decode even more KV-cache memory-bandwidth-bound (smaller weights to read, same KV traffic), which could amplify the Hybrid's edge. Determining whether the GKA's FP32 state can be modified such that it can be relaxed to BF16 or FP8 without degrading the Chebyshev solver's convergence, and how weight and KV-cache quantization interact with Hybrid inference gains, requires further study.

\section{Conclusion}
\label{sec:conclusion}

This work introduces Priming, a method for building Hybrid State-Space models from source pre-trained Transformers at less than 0.5\% of the source model's pre-training token budget. We summarize the main contributions below.

\paragraph{A general method for Hybrid model construction.}
Priming converts the problem of designing Hybrid architectures from one that requires full pre-training into a knowledge-transfer problem. The method is easily adaptable to the source Transformer family of choice (Qwen, Llama, Mistral) and is agnostic to model class (dense or MoE), scale (8B--32B), and target layer type (Mamba-2, GDN, GKA, or the Hybrid B'MOJO-F layer). Its three-stage structure, informed initialization, layerwise alignment, and task adaptation, is grounded in classical realization theory, which identifies when an SSM can recover an Attention layer's input-output behavior and guides the design of each stage.

\paragraph{A controlled comparison of SSM layer types.}
By holding the source model, data, hyperparameters, and Hybrid ratio constant while varying only the SSM layer, Priming enables the first head-to-head evaluation of SSM families at scale under identical conditions. We establish the empirical hierarchy GKA~$>$~GDN~$>$~Mamba-2 on long-context tasks, consistent with the theoretical expressiveness ordering. We further demonstrate that fading-memory SSM layers consistently outperform Sliding Window Attention when paired with full Attention in Hybrid models. We empirically show this holds true even when the sliding window's KV cache is sized to be several times larger than the SSM state, evidence that compressed summaries of the full history are more useful than verbatim access to a bounded recent window.

\paragraph{Hybrid reasoning models at scale.}
We scale GKA-based Primed Hybrids to 32B parameters with 128K native context, achieving reasoning performance that is competitive with its source Transformer post-trained with the same recipe while offering up to $2.3{\times}$ decode throughput and ${\sim}2{\times}$ concurrent sequence capacity on the same hardware. GKA's iterative solver exposes a runtime compute-quality knob that no other SSM offers: a single trained model serves multiple accuracy-efficiency trade-offs by adjusting the Chebyshev iteration count $r$ at deployment time. On AIME~2025, this translates to up to $1.6{\times}$ faster time-to-accuracy under realistic concurrent workloads, a property directly beneficial for reinforcement learning post-training where data generation rate gates training throughput.

\paragraph{Open release.}
We release a model zoo of Primed Hybrid models (8B-32B) for long-context reasoning and instruction following,\footnote{Model zoo: \url{https://huggingface.co/collections/amazon/primed-hybrid-models-collection}} together with the infrastructure that enables them: the Priming pipeline, Sequence Parallelism algorithms for SSM layers, optimized GKA kernels, and vLLM serving plugin, all under Apache~2.0.\footnote{Training and inference code: \url{https://github.com/awslabs/hybrid-model-factory}}

\subsection*{Future Directions}

Several questions remain open:

\begin{itemize}[nosep]
  \item \emph{Scaling beyond 32B.} Whether the Hybrid's advantage persists at 100B+ scale, and whether Priming's token efficiency continues to hold as the source model grows, is untested.
  \item \emph{Higher Hybrid ratios.} Preliminary results suggest Priming scales to ${\sim}75\%$ ratio with limited quality regression. Systematically characterizing how SSM expressiveness, inference efficiency, and downstream quality trade off as the Hybrid ratio increases would inform the design of even more scalable Hybrid architectures.
  \item \emph{Dynamic iteration scheduling.} GKA's iterative solver quality $r$ knob is currently set globally at serving time. Adapting $r$ per-layer or per-token within a generation pass, using fewer iterations early in the sequence when context is short, more as context grows or when more reasoning is necessary, could improve the efficiency-quality frontier further.
  \item \emph{Extreme context extension.} State composition extends context to 256K with strong retention and to 512K with meaningful but degraded performance. Closing the gap at $4\times$ and beyond, through improved merging strategies, adaptive scaling, and evaluation on more challenging long-document and agentic tasks, is a natural next step.
  \item \emph{RL post-training.} Our reasoning models use SFT. Whether the Hybrid architecture's throughput advantage compounds under RL training loops, where more rollouts per unit time translates to faster policy improvement, is an important open practical question.
\end{itemize}

\bibliographystyle{plainnat}
\bibliography{references}

\newpage
\appendix

\section{Priming Implementation Details}

\subsection{Gate Initialization}
\label{sec:priming_gate_init}

SSM layers include an output gate $\mathbf{g}_t = \mathbf{W}_G \mathbf{X}_t$ that modulates the layer's output before the final projection. Prior work~\citep{mambainllama} initializes $\mathbf{W}_G$ randomly. We instead initialize it as a blend of the transposed output projection and the GQA-expanded value projection:
\begin{equation}
\label{eq:gate-init}
\mathbf{W}_G = 0.5 \cdot \bigl(
  \mathbf{W}_O^\top +
  \texttt{repeat\_interleave}(\mathbf{W}_V,\, m)
\bigr),
\end{equation}
where $m = H_Q / H_{\mathrm{KV}}$ is the GQA group size and $\texttt{repeat\_interleave}(\mathbf{W}_V, m) \in \mathbb{R}^{H_Q \cdot d_{\mathrm{head}} \times d_{\mathrm{model}}}$ repeats each of the $H_{\mathrm{KV}}$ value heads $m$ times to match the shape of $\mathbf{W}_O^\top$. The gate controls how much of the recurrent output reaches the output projection, making it dependent on both the output space structure and the input content. $\mathbf{W}_O^\top$ provides alignment with the output space; the expanded $\mathbf{W}_V$ captures the input content structure. Initializing the gate in this manner allows our Primed Hybrid models to effectively use the knowledge stored in the pre-trained weights of the source transformer rather than learning from scratch. We apply this initialization uniformly across GKA, GDN, and Mamba2. 

\begin{figure}[h]
\centering
\includegraphics[width=0.5\linewidth]{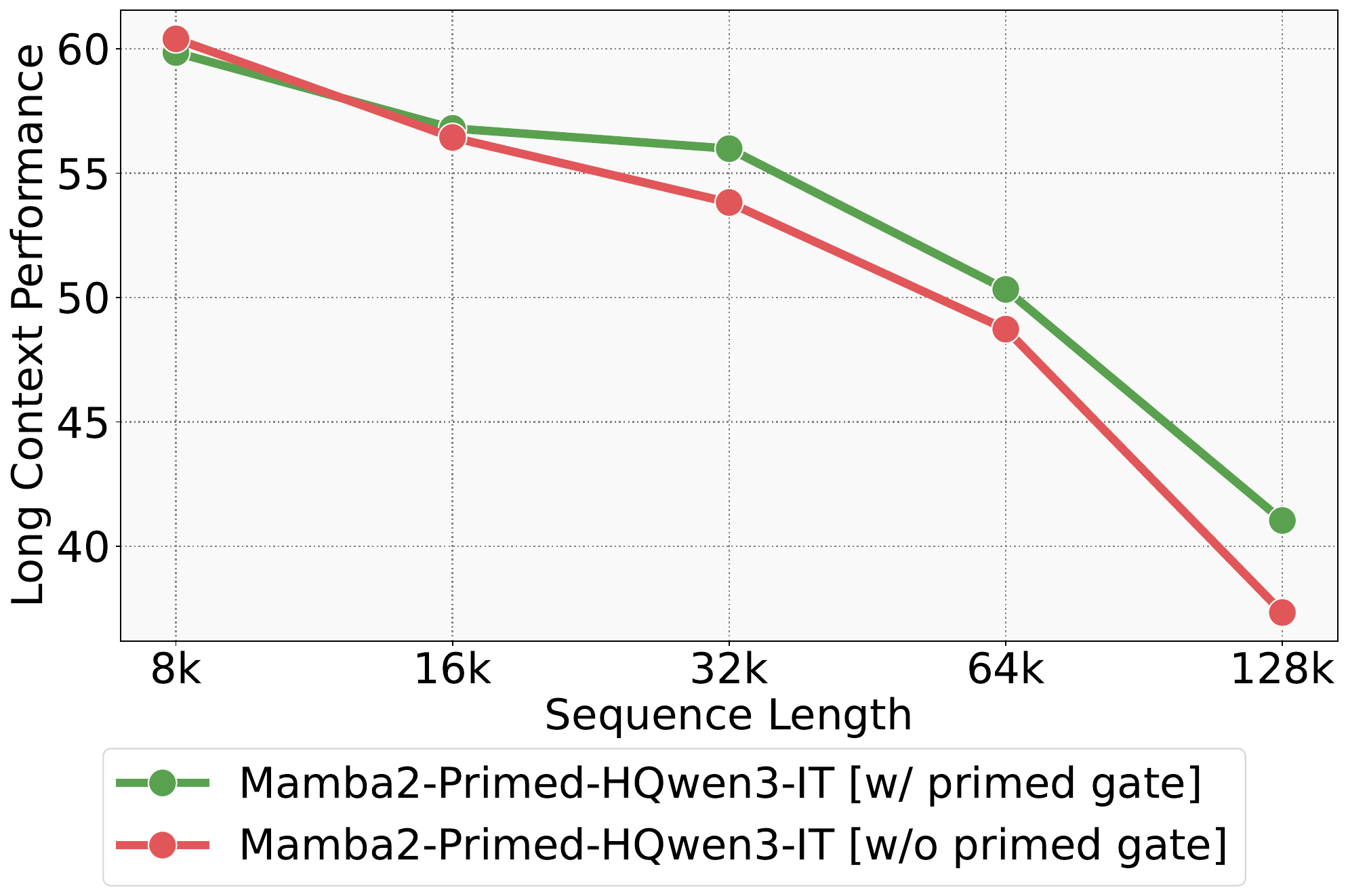}
\caption{\textbf{Effect of the primed gate initialization on long-context performance.} Per-context-length scores of Mamba2-Primed-HQwen3-IT at 8B scale, with (green) and without (red) the gate initialization from \cref{eq:gate-init}. The ``w/o primed gate'' variant follows \citet{mambainllama} and initializes $\mathbf{W}_G$ randomly. Both variants use the same Stage~0/1/2 recipe and differ only in $\mathbf{W}_G$ initialization. Scores are the average over long-context tasks discussed in \cref{subsec: IT_results_long_ctx} reported per context length.}
\label{fig:gate_init_ablation}
\end{figure}

\paragraph{Empirical validation.} \cref{fig:gate_init_ablation} confirms that this effective use of pre-trained knowledge translates into measurable long-context gains on Mamba2-Primed-HQwen3-IT 8B. The two variants are competitive at 8K, and the gap grows with context length, reaching a 9.9\% (relative) improvement of the primed gate over random initialization at 128K. Random initialization forces the gate to learn a good configuration from scratch during Stage~1 and Stage~2, and the Priming token budget ($<$0.5\% of the source Transformer's pre-training budget) is not enough to fully close the gap. We validate this on Mamba2 and apply the same initialization to GKA and GDN, since the mechanism is architecture-agnostic.

\subsection{QK Norm Transfer}
\label{sec:priming_qk_norm}

Some Transformer architectures (e.g., Qwen3) apply normalization to the Q and K projections before computing attention. Standard SSM layers do not natively include these norms. Since Stage~0 seeds SSM parameters from the Attention projections, we add corresponding norm layers to the SSM architecture and initialize them with the source Transformer's QK norm weights. This ensures the SSM layers operate in the same normalized space as the Attention layers from which they were initialized.

\subsection{Adaptive GQA (AGQA)}
\label{sec:priming_agqa}

As discussed in \cref{subsec:stateexp-theory}, standard GQA head repetition is lossy for SSM layers because identical keys and values produce identical state updates across grouped heads. AGQA replaces the fixed repetition with a learned low-rank projection plus a residual connection:
\begin{equation}
\label{eq:agqa}
F_{\mathrm{AGQA}}(\mathbf{X}) =
  \operatorname{\texttt{mat}}\bigl(\mathbf{W}_2 \,\sigma(\mathbf{W}_1
  \operatorname{\texttt{vec}}(\mathbf{X}))\bigr)
  + \mathbf{R}_{\mathrm{AGQA}} \, \mathbf{X},
\end{equation}
where $r$ is the rank of the low-rank projection, $\mathbf{W}_1 \in \mathbb{R}^{r \times H_{\mathrm{KV}} \cdot d_{\mathrm{head}}}$ and $\mathbf{W}_2 \in \mathbb{R}^{H_Q \cdot d_{\mathrm{head}} \times r}$ are learnable parameters, $\sigma$ is SiLU, and
$\mathbf{R}_{\mathrm{AGQA}} \in \mathbb{R}^{H_Q \times H_{\mathrm{KV}}}$ is a residual matrix; $\operatorname{\texttt{vec}}: \mathbb{R}^{H_{\mathrm{KV}} \times d_{\mathrm{head}}} \mapsto \mathbb{R}^{H_{\mathrm{KV}} \cdot d_{\mathrm{head}}}$ flattens to a vector, $\operatorname{\texttt{mat}}: \mathbb{R}^{H_Q \cdot d_{\mathrm{head}}} \mapsto \mathbb{R}^{H_Q \times d_{\mathrm{head}}}$ reshapes to a matrix. We initialize $\mathbf{W}_2$ to zeros and $\mathbf{R}_{\mathrm{AGQA}}$ to the standard GQA replication matrix
$\mathbf{R}_{\mathrm{GQA}} = \mathbf{I}_{H_{\mathrm{KV}}} \otimes \mathbf{1}_{m \times 1}$, so that $F_{\mathrm{AGQA}} = F_{\mathrm{GQA}}$ at the start of training. The model begins with standard GQA and learns to deviate from it only as needed. $\mathbf{R}_{\mathrm{AGQA}}$ can optionally be made learnable, allowing the model to adjust how KV heads are mixed across query groups beyond what the low-rank branch provides.
See \cref{subsubsec:agqa_vs_gqa_ablations} for experimental results.

\subsection{B'MOJO-F Initialization}
\label{sec:priming_bmojof_init}

B'MOJO-F is a dual-stream layer that combines Sliding Window Attention (SWA) with a configurable SSM (GKA, GDN, or Mamba2) within a single decoder layer. During Stage~0, both components are initialized from the same source Attention layer: the SWA component receives a direct copy of $\mathbf{W}_Q$, $\mathbf{W}_K$, $\mathbf{W}_V$, $\mathbf{W}_O$, while the SSM component is initialized using the standard mapping described in \cref{subsec:stage0-theory}. Both components are then trained independently during Stages~1 and~2.

\subsection{Complete Initialization Mapping}
\label{sec:priming_init_mapping}

\Cref{tab:init-mapping} summarizes the full Stage~0 parameter transfer. For models with GQA, key and value heads are expanded via AGQA (\cref{sec:priming_agqa}) or standard $\texttt{repeat\_interleave}$.

\begin{table}[h]
\centering
\caption{\textbf{Stage~0 initialization mapping from Attention projections to SSM parameters.} For models with GQA, key and value heads are expanded via AGQA or standard head repetition. All remaining layer-specific parameters $\boldsymbol{\theta}$ (e.g., convolutions, transition matrices) are initialized randomly.}
\label{tab:init-mapping}
\begin{tabular}{ll}
\toprule
\textbf{SSM Parameter} & \textbf{Initialized From} \\
\midrule
$\mathbf{W}_Q$ & Attention $\mathbf{W}_Q$ \\
$\mathbf{W}_K$ (+ AGQA expansion) & Attention $\mathbf{W}_K$ \\
$\mathbf{W}_V$ (+ AGQA expansion) & Attention $\mathbf{W}_V$ \\
$\mathbf{W}_O$ & Attention $\mathbf{W}_O$ \\
$\mathbf{W}_G$ & $0.5 \cdot (\mathbf{W}_O^\top + \texttt{repeat\_interleave}(\mathbf{W}_V, m))$ \\
QK Norms & Attention QK norms \\
$\boldsymbol{\theta}$ (layer-specific) & Random \\
\bottomrule
\end{tabular}
\end{table}

\section{Realization Theory}
\label{app:realization-theory}

\subsection{Mixing Matrix and Finite-Horizon Input-Output Matrix}
\label{sec: mixing matrix and finite horizon input-output matrix}

All SSM layers considered in this work (Mamba-2, GDN, GKA) share the following general linear time-varying state-space form. Let $\bm{v}_t\in\mathbb R^{d_v}$ denote the value input, $\bm{y}_t\in\mathbb R^{d_v}$ the output at time~$t$, and $\mathbf S_t\in\mathbb R^{d_v\times n}$ the matrix-valued latent state of state dimension~$n$. The dynamics equations are
\[
\mathbf S_{t}
  = \mathbf S_{t-1}\,\mathbf A_t(\mathbf X_{\le t})
    + \bm{v}_t\,\mathbf B_t(\mathbf X_{\le t}),
\qquad
\bm{y}_t = \mathbf S_t\,\mathbf C_t(\mathbf X_{\le t})
           + \bm{v}_t\,D_t(\mathbf X_{\le t}),
\]
with $\mathbf S_0=\mathbf 0$, where the system matrices
$\mathbf A_t(\mathbf X_{\le t})\in\mathbb R^{n\times n}$,
$\mathbf B_t(\mathbf X_{\le t})\in\mathbb R^{1\times n}$,
$\mathbf C_t(\mathbf X_{\le t})\in\mathbb R^{n\times 1}$, and
$D_t(\mathbf X_{\le t})\in\mathbb R$
are all functions of the input sequence up to time~$t$ in their most general form.

\paragraph{Row-wise decomposition and Single-Input Single-Output (SISO) reduction.}
The dynamics above simplify because the matrices $\mathbf A_t$, $\mathbf B_t$, $\mathbf C_t$, $D_t$ are shared across all $d_v$ rows of $\mathbf S_t$. Let $\mathbf S_t^{(p)}\in\mathbb R^{1\times n}$ denote the $p$-th row of $\mathbf S_t$ and $(v_t)_p,(y_t)_p\in\mathbb R$ the $p$-th coordinates of $\bm{v}_t$ and $\bm{y}_t$. In the state update, the term $\bm{v}_t\,\mathbf B_t$ is the outer product $\bm{v}_t\in\mathbb R^{d_v\times 1}$ times $\mathbf B_t\in\mathbb R^{1\times n}$, so its $p$-th row is simply $(v_t)_p\,\mathbf B_t$.
In the output, $\bm{y}_t=\mathbf S_t\,\mathbf C_t+\bm{v}_t\,D_t$ gives $(y_t)_p=\mathbf S_t^{(p)}\,\mathbf C_t+(v_t)_p\,D_t$. Row by row, the dynamics therefore decouple into $d_v$ independent scalar (SISO) systems, each driven by a different coordinate of $\bm{v}_t$ but governed by the same matrices:
\[
\mathbf S_{t}^{(p)}
  =\mathbf S_{t-1}^{(p)}\,\mathbf A_t+(v_t)_p\,\mathbf B_t,
\qquad
(y_t)_p=\mathbf S_t^{(p)}\,\mathbf C_t+(v_t)_p\,D_t,
\qquad p=1,\dots,d_v.
\]
Because the finite horizon input-output matrix $\mathbf T$ is identical for every coordinate~$p$, it suffices to analyse a single SISO channel. Writing $v_t,y_t\in\mathbb R$ for that coordinate and $s_t\in\mathbb R^{1\times n}$ for the corresponding row of $\mathbf S_t$, the dynamics become
\[
s_{t} = s_{t-1}\,\mathbf A_t + v_t\,\mathbf B_t,
\qquad
y_t = s_t\,\mathbf C_t + v_t\,D_t.
\]
Unrolling the recurrence with $s_0=\bm 0$ gives
\[
y_i = \sum_{j=1}^{i-1} v_j\;\mathbf B_j\,\mathbf A_{j+1}\cdots
      \mathbf A_i\,\mathbf C_i
      \;+\; v_i\bigl(\mathbf B_i\,\mathbf C_i + D_i\bigr).
\]

The input-output relationship is $\mathbf y=\mathbf T\,\mathbf v$, where $\mathbf y=(y_1,\dots,y_T)^\top$, $\mathbf v=(v_1,\dots,v_T)^\top$, and $\mathbf T\in\mathbb R^{T\times T}$ is the \emph{finite horizon input-output matrix}:
\begin{equation}\label{eq:transfer_matrix_v2}
T_{ij}=
\begin{cases}
  \mathbf B_j\,\mathbf A_{j+1}\cdots\mathbf A_i\,\mathbf C_i
    & \text{if } i>j,\\[6pt]
  \mathbf B_i\,\mathbf C_i + D_i
    & \text{if } i=j,\\[4pt]
  0 & \text{if } i<j.
\end{cases}
\end{equation}
The matrix $\mathbf T$ is lower-triangular by causality. Moreover, $\mathbf T$ is a \emph{semi-separable matrix of order~$n$}: every submatrix lying strictly below the diagonal has rank at most~$n$. To see why, note that the $(i,j)$ entry for $i>j$ factors as $\mathbf B_j\cdot(\text{state-transition from }j\text{ to }i)\cdot \mathbf C_i$, where the middle factor is $n\times n$; any subblock of such entries therefore has rank bounded by the state dimension~$n$. All SSMs considered in this work, Mamba-2, GDN, GKA, with state dimension~$n$ produce a finite horizon input-output matrix of this form, and the semi-separable order is the fundamental bottleneck that limits how faithfully an SSM can reproduce a given mixing matrix.

\subsection{Proof of Theorem~\ref{thm:realization_of_mixing_matrices}}

The following result adapts the classical time-varying realization theorem \citep[Theorem~3.7]{DewildeVanderVeen1998} to the causal Attention setting.

\begin{theorem}[Finite-Horizon Per-Instance Realization of the Attention Map]
\label{thm:realization_of_mixing_matrices_appendix}
Fix a sequence $\mathbf X\in\mathbb R^{T\times d}$, let $\mathbf V(\mathbf X)\in\mathbb R^{T\times d_v}$ be the sequence of values, and let $\mathbf M(\mathbf X)\in\mathbb R^{T\times T}$ be the causal Attention mixing matrix of a single head.
Let $\mathbf X_{\le t}$ denote the input sequence up to time~$t$. Then there exists a linear time-varying state-space realization
\[
\mathbf S_{t}
  = \mathbf S_{t-1}\,\mathbf A_t(\mathbf X_{\le t})
    + \bm{v}_t\,\mathbf B_t(\mathbf X_{\le t}),
\qquad
\bm{y}_t = \mathbf S_t\,\mathbf C_t(\mathbf X_{\le t})
           + \bm{v}_t\,D_t(\mathbf X_{\le t}),
\]
with $\bm{v}_t\in\mathbb R^{d_v}$, $\bm{y}_t\in\mathbb R^{d_v}$, zero initial state, and state dimension $n_{\min}(\mathbf X)$, whose \textit{finite-horizon input-output matrix} $\mathbf T(\mathbf X)$ is exactly the \textit{mixing matrix} $\mathbf M(\mathbf X)$. Moreover, no realization of this finite-horizon input-output map with smaller state dimension exists.
\end{theorem}

\begin{proof}
Fix $\mathbf X$ and write $\mathbf M=\mathbf M(\mathbf X)$. As noted above, the SISO reduction applies row-wise, so it suffices to work with the scalar system
\[
s_{t}=s_{t-1}\,\mathbf A_t+v_t\,\mathbf B_t,
\qquad
y_t=s_t\,\mathbf C_t+v_t\,D_t,
\]
with $s_0=\bm 0$. Since $\mathbf M$ is causal, it is lower triangular.

\medskip\noindent\textbf{Finite horizon input-output matrix.}\;
The induced input-output matrix $\mathbf T$ is lower triangular with
\[
T_{ii}=\mathbf B_i\,\mathbf C_i+D_i,
\qquad
T_{ij}=\mathbf B_j\,\mathbf A_{j+1}\cdots\mathbf A_i\,\mathbf C_i,
\quad i>j.
\]

\medskip\noindent\textbf{Lower bound.}\;
For every cut $k$, the past-to-future block satisfies
\[
\mathbf T_{k+1:T,\;1:k}
=\bigl(\mathcal R_k\,\mathcal O_k\bigr)^\top,
\]
where
\[
\mathcal R_k
=\begin{bmatrix}
\mathbf B_1\,\mathbf A_2\cdots\mathbf A_k\\[2pt]
\mathbf B_2\,\mathbf A_3\cdots\mathbf A_k\\
\vdots\\
\mathbf B_{k-1}\,\mathbf A_k\\[2pt]
\mathbf B_k
\end{bmatrix}
\in\mathbb R^{k\times n}
\]
is the reachability map (its $j$-th row is
$\mathbf B_j\prod_{\ell=j+1}^{k}\mathbf A_\ell$, with the empty product
equal to $\mathbf I_n$ when $j=k$), and
\[
\mathcal O_k
=\begin{bmatrix}
\mathbf A_{k+1}\,\mathbf C_{k+1}
& \mathbf A_{k+1}\mathbf A_{k+2}\,\mathbf C_{k+2}
& \cdots
& \mathbf A_{k+1}\cdots\mathbf A_T\,\mathbf C_T
\end{bmatrix}
\in\mathbb R^{n\times(T-k)}
\]
is the observability map.
The $(j,i{-}k)$ entry of $\mathcal R_k\,\mathcal O_k$ is $\mathbf B_j\,\mathbf A_{j+1}\cdots\mathbf A_i\,\mathbf C_i=T_{ij}$, which is the $(i{-}k,j)$ entry of $\mathbf T_{k+1:T,\;1:k}$, confirming the transpose relationship. Note that the feedthrough terms $D_t$ do not appear in the past-to-future block (which is strictly below the diagonal), so the factorization depends only on $\mathbf A_t$, $\mathbf B_t$, $\mathbf C_t$. Since $\mathcal R_k$ has at most $n$ columns and $\mathcal O_k$ has at most $n$ rows,
\[
\operatorname{rank}\!\bigl(\mathbf T_{k+1:T,\;1:k}\bigr)
=\operatorname{rank}\!\bigl(\mathcal R_k\,\mathcal O_k\bigr)\le n
\]
for every~$k$.  Hence any realization of $\mathbf M$ must satisfy
\[
n\ge\max_{1\le k<T}
     \operatorname{rank}\!\bigl(\mathbf M_{k+1:T,\;1:k}\bigr)
    =n_{\min}(\mathbf X).
\]

\medskip\noindent\textbf{Upper bound.}\;
Let
\[
r=n_{\min}(\mathbf X)
 =\max_{1\le k<T}
  \operatorname{rank}\!\bigl(\mathbf M_{k+1:T,\;1:k}\bigr).
\]
Because every past-to-future block of $\mathbf M$ has rank at most~$r$, $\mathbf M$ is a lower-triangular semiseparable matrix of order~$r$. By Lemma~\ref{lem:semisep_realization_v2} below, the strictly lower-triangular entries of $\mathbf M$ admit a generator representation
\[
M_{ij}=\mathbf B_j\,\mathbf A_{j+1}\cdots\mathbf A_i\,\mathbf C_i,
\qquad i>j,
\]
with state dimension~$r$. The diagonal entries are realized independently by choosing
\[
D_t = M_{tt} - \mathbf B_t\,\mathbf C_t.
\]
With these choices, the induced finite horizon input-output matrix $\mathbf T$ satisfies
\[
T_{ii}=\mathbf B_i\,\mathbf C_i + D_i
      =\mathbf B_i\,\mathbf C_i + M_{ii} - \mathbf B_i\,\mathbf C_i
      =M_{ii},
\]
and, for $i>j$,
\[
T_{ij}=\mathbf B_j\,\mathbf A_{j+1}\cdots\mathbf A_i\,\mathbf C_i=M_{ij}.
\]
Since both $\mathbf T$ and $\mathbf M$ are lower triangular, this gives $\mathbf T=\mathbf M$.

Thus there exists a finite-horizon linear time-varying realization with state dimension $r=n_{\min}(\mathbf X)$, and the lower-bound argument shows that no realization with smaller state dimension can realize the same finite-horizon input-output map.
\end{proof}

\begin{lemma}[Finite-Horizon Semiseparable Realization]
\label{lem:semisep_realization_v2}
Let $\mathbf M\in\mathbb R^{T\times T}$ be lower triangular and define
\[
r=\max_{1\le k<T}
  \operatorname{rank}\!\bigl(\mathbf M_{k+1:T,\;1:k}\bigr).
\]
Then the strictly lower-triangular part of $\mathbf M$ admits a linear time-varying realization of state dimension~$r$ in the transposed convention (i.e. the state if a row vector multiplied by matrices on the right): there exist matrices
\[
\mathbf A_t\in\mathbb R^{r\times r},\qquad
\mathbf B_t\in\mathbb R^{1\times r},\qquad
\mathbf C_t\in\mathbb R^{r\times 1},
\]
such that, for all $i>j$,
\[
M_{ij}=\mathbf B_j\,\mathbf A_{j+1}\cdots\mathbf A_i\,\mathbf C_i.
\]
\end{lemma}

The idea behind the proof is the following. At every time cut~$k$, the past inputs $v_1,\dots,v_k$ influence the future outputs $y_{k+1},\dots,y_T$ only through the \emph{future output tail} $z_{k+1}\in\mathbb R^{T-k}$, which is the vector of all future outputs
attributable to those past inputs. The set of all achievable future tails, $\mathcal X_k=\operatorname{im}(H_k)$, has dimension at most~$r$ by assumption. We pick an $r$-dimensional coordinate system $Q_k$ for this space at each cut, and then express the one-step evolution of the future tail (drop the output just emitted, add the new input's contribution) as a linear map in those coordinates. This yields the state-space matrices $\mathbf A_t$, $\mathbf B_t$, $\mathbf C_t$ for the strictly lower-triangular part directly. The diagonal is handled separately by the feedthrough term $D_t$ in the theorem.

\begin{proof}
For each cut $k=1,\dots,T-1$, define the past-to-future block
\[
H_k:=\mathbf M_{k+1:T,\;1:k}\in\mathbb R^{(T-k)\times k}.
\]
Thus $H_k$ maps a past input vector $(v_1,\dots,v_k)^\top$ to the future output tail $(y_{k+1},\dots,y_T)^\top$ generated by those past inputs.

Let
\[
\mathcal X_k:=\operatorname{im}(H_k)\subseteq\mathbb R^{T-k}.
\]
By definition, $\dim\mathcal X_k=\operatorname{rank}(H_k)\le r$. The space $\mathcal X_k$ is the reachable future-tail space at time~$k$: two past input histories are equivalent if they produce the same future output tail.

For each~$k$, choose a matrix
\[
Q_k\in\mathbb R^{(T-k)\times r}
\]
whose columns form a basis for $\mathcal X_k$. If $\dim\mathcal X_k<r$, extend to a set of $r$ columns that still has full column rank (e.g.\ by appending arbitrary vectors that complete the columns to a linearly independent set in $\mathbb R^{T-k}$). In particular, $Q_k$ has full column rank~$r$ for every~$k$, and
\[
\operatorname{im}(Q_k)\supseteq\mathcal X_k.
\]

\paragraph{Evolution of the future tail.}
Suppose the past inputs up to time~$k$ have produced the future tail
\[
z_{k+1}
=H_k
\begin{bmatrix}v_1\\\vdots\\v_k\end{bmatrix}
\in\mathcal X_k.
\]
Its first coordinate is the output $y_{k+1}$ caused by past inputs; the remaining coordinates are the outputs $y_{k+2},\dots,y_T$ caused by those same past inputs.

Let $P_k:\mathbb R^{T-k}\to\mathbb R^{T-k-1}$ delete the first coordinate. After advancing one step, the old past contribution to the remaining future tail is $P_k\,z_{k+1}$.

The new input $v_{k+1}$ contributes to future outputs $y_{k+2},\dots,y_T$ through the column tail
\[
b_{k+1}:=\mathbf M_{k+2:T,\;k+1}\in\mathbb R^{T-k-1}.
\]
Hence the new future tail is
\[
z_{k+2}=P_k\,z_{k+1}+b_{k+1}\,v_{k+1}.
\]
By construction this vector lies in $\mathcal X_{k+1}=\operatorname{im}(H_{k+1})$, because
\[
H_{k+1}
\begin{bmatrix}v_1\\\vdots\\v_k\\v_{k+1}\end{bmatrix}
=P_k\,H_k
\begin{bmatrix}v_1\\\vdots\\v_k\end{bmatrix}
+\mathbf M_{k+2:T,\;k+1}\,v_{k+1}.
\]
Therefore
\[
\operatorname{im}(P_k\,Q_k)\subseteq\operatorname{im}(Q_{k+1}),
\qquad
b_{k+1}\in\operatorname{im}(Q_{k+1}).
\]

\paragraph{State-space matrices.}
To express the future-tail recurrence in coordinates, substitute
$z_{k+1}=Q_k\,\sigma_{k+1}^\top$ and
$z_{k+2}=Q_{k+1}\,\sigma_{k+2}^\top$ into
$z_{k+2}=P_k\,z_{k+1}+b_{k+1}\,v_{k+1}$, giving
\[
Q_{k+1}\,\sigma_{k+2}^\top
=P_k\,Q_k\,\sigma_{k+1}^\top+b_{k+1}\,v_{k+1}.
\]
We are looking to find the recurrent relationship between the coordinate representation of the future tails $\sigma_{k+2}$, $\sigma_{k+1}$ and inputs $v_{k+1}$. For this to take the state-space form $\sigma_{k+2}=\sigma_{k+1}\,\mathbf A_{k+1}+v_{k+1}\,\mathbf B_{k+1}$ for all $\sigma_{k+1}$ and $v_{k+1}$, we transpose the desired update to column form,
$\sigma_{k+2}^\top=\mathbf A_{k+1}^\top\,\sigma_{k+1}^\top
+\mathbf B_{k+1}^\top\,v_{k+1}$,
and substitute into the left-hand side:
\[
Q_{k+1}\bigl(\mathbf A_{k+1}^\top\,\sigma_{k+1}^\top
+\mathbf B_{k+1}^\top\,v_{k+1}\bigr)
=P_k\,Q_k\,\sigma_{k+1}^\top+b_{k+1}\,v_{k+1}.
\]
Expanding and matching the coefficients of $\sigma_{k+1}^\top$ and $v_{k+1}$ independently (since the identity must hold for every input history and every new input), we obtain
\[
Q_{k+1}\,\mathbf A_{k+1}^\top=P_k\,Q_k,
\qquad
Q_{k+1}\,\mathbf B_{k+1}^\top=b_{k+1}.
\]
The first equation says that $\mathbf A_{k+1}^\top$ is the change of basis from the old coordinate system~$Q_k$ to the new one~$Q_{k+1}$ after the first output has been emitted (i.e.\ after applying~$P_k$). The second equation says that $\mathbf B_{k+1}^\top$ encodes the new input's contribution in the new coordinates.

Since $Q_{k+1}$ has full column rank~$r$ (and hence admits a left inverse), both equations have unique solutions:
\[
\mathbf A_{k+1}^\top=Q_{k+1}^+\,P_k\,Q_k,
\qquad
\mathbf B_{k+1}^\top=Q_{k+1}^+\,b_{k+1},
\]
where $Q_{k+1}^+=(Q_{k+1}^\top Q_{k+1})^{-1}Q_{k+1}^\top$ is the left inverse.
Transposing gives $\mathbf A_{k+1}$ and $\mathbf B_{k+1}$ in the row-vector convention directly.

The output at time $k+1$ caused by the past is the first coordinate of $z_{k+1}$.  Since $z_{k+1}=Q_k\,\sigma_{k+1}^\top$, this coordinate is
\[
y_{k+1}
=e_1^\top\,z_{k+1}
=e_1^\top\,Q_k\,\sigma_{k+1}^\top
=\sigma_{k+1}\,Q_k^\top\,e_1,
\]
where the last step transposes a scalar. Comparing with the desired output equation
$y_{k+1}=\sigma_{k+1}\,\mathbf C_{k+1}$ gives
\[
\mathbf C_{k+1}:=Q_k^\top\,e_1\in\mathbb R^{r\times 1},
\]
where $e_1$ is the first standard basis vector in $\mathbb R^{T-k}$.

\paragraph{Coordinate state.}
Let $\sigma_{k+1}\in\mathbb R^{1\times r}$ (a row vector) satisfy
$z_{k+1}=Q_k\,\sigma_{k+1}^\top$, i.e.\
$\sigma_{k+1}=(Q_k^+\,z_{k+1})^\top$.
The dynamics become
\[
\sigma_{k+2}=\sigma_{k+1}\,\mathbf A_{k+1}+v_{k+1}\,\mathbf B_{k+1},
\qquad
y_{k+1}^{\text{past}}=\sigma_{k+1}\,\mathbf C_{k+1},
\]
where $y_{k+1}^{\text{past}}$ denotes the contribution of past inputs $v_1,\dots,v_k$ to the output at time $k+1$. The contribution of the current input $v_{k+1}$ to $y_{k+1}$ is not captured by the future-tail construction; it is handled by the feedthrough term $D_{k+1}$ in the theorem.

\paragraph{Verification.}
Chaining the one-step updates, the contribution of input $v_j$ to output
$y_i$ for $i>j$ is
\[
v_j\;\mathbf B_j\,\mathbf A_{j+1}\cdots\mathbf A_i\,\mathbf C_i
=v_j\;M_{ij}.
\]
Hence $T_{ij}=M_{ij}$ for all $i>j$, and $T_{ij}=0$ for $i<j$ by causality. The diagonal entries are not determined by this construction; they are
realized by the feedthrough term $D_t$ in the theorem.

Thus the strictly lower-triangular part of $\mathbf M$ admits a linear time-varying realization of state dimension~$r$ in the transposed convention.
\end{proof}

\begin{remark}[Choice of basis $Q_k$]
\label{rem:choice_of_Q}
The proof requires only that $Q_k$ has full column rank~$r$ and spans $\mathcal X_k=\operatorname{im}(H_k)$; different choices yield different realizations $(\mathbf A_t,\mathbf B_t,\mathbf C_t)$ related by time-varying similarity transforms, all producing the same finite horizon input output matrix~$\mathbf T$. Several concrete choices are natural:
\begin{itemize}
\item \textbf{SVD basis.}
  Take $Q_k$ from the thin SVD of $H_k=U_k\Sigma_k V_k^\top$: use the first $\rho_k=\operatorname{rank}(H_k)$ left singular vectors as columns and pad with $r-\rho_k$ orthonormal vectors from the orthogonal complement. The resulting $Q_k$ has orthonormal columns, so the left inverse reduces to $Q_k^+=Q_k^\top$, giving $\mathbf A_{k+1}^\top=Q_{k+1}^\top P_k\,Q_k$ and $\mathbf B_{k+1}^\top=Q_{k+1}^\top b_{k+1}$.

\item \textbf{Balanced realization.}
  Choose $Q_k$ so that the reachability and observability Gramians of the resulting realization are equal and diagonal. Balanced realizations are the basis of balanced truncation~\citep{BTSurvery,TVBTSurvery}, which provides near-optimal model reduction with $H_\infty$-norm error bounds for time-invariant systems.
  For the time-varying finite-horizon setting considered here, balanced truncation remains applicable but the optimality guarantees are weaker~\citep{TVBTSurvery}.

\end{itemize}
When the goal is exact realization (state dimension equals $r$), all choices are equivalent up to similarity. The distinction matters when the state dimension is reduced below~$r$ (approximate realization): the truncated SVD basis minimizes the per-cut Frobenius error for an individual Hankel block, whereas balanced truncation uses reachability and observability information to produce a system-aware approximation of the entire input-output map.
\end{remark}

\section{Layer Selection Procedure}
\label{app:layer-selection}

Algorithm~\ref{alg:layer_selection} details the importance scoring procedure used to select which Attention layers to replace with SSM layers during Priming.  The procedure is training-free: for each layer, we substitute its full-context Attention with Sliding Window Attention (SWA) and measure the resulting degradation on a suite of long-context benchmarks.  Layers that tolerate SWA well are safe candidates for SSM instantiation, while layers that degrade significantly are kept as Attention.  As discussed in \cref{subsec:layer_selection}, SWA serves as an empirical proxy for the effective Hankel rank of each layer's mixing matrix, layers with high Hankel rank rely on long-range structure that a fixed-size recurrent state cannot compress.

\begin{algorithm}[t]
\caption{Layer Selection via Importance Scoring}\label{alg:layer_selection}
\KwIn{Source Transformer $\mathcal{M}$ with $L$ Attention layers; number of SSM layers $M$; SWA window size $w$; evaluation benchmark suite $\mathcal{B}$}
\KwOut{Set of layers $\mathcal{S}_{\mathrm{convert}}$ to instantiate as SSM layers}
\For{each Attention layer $i = 0, \ldots, L{-}1$}{
  Create variant $\mathcal{M}^{(i)}$: replace only layer $i$ of $\mathcal{M}$ with SWA (window size $w$)\;
  Evaluate $\mathcal{M}^{(i)}$ on $\mathcal{B}$ to obtain scores $s_t^{(i)}$ for each task $t$\;
}
Compute mean retained performance per layer: $\bar{s}^{(i)} = \frac{1}{|\mathcal{B}|}\sum_{t} s_t^{(i)}$\;
Rank layers by $\bar{s}^{(i)}$ in descending order (highest = least degradation)\;
$\mathcal{S}_{\mathrm{convert}} \leftarrow$ top-$M$ layers by rank\;
\Return{$\mathcal{S}_{\mathrm{convert}}$}
\end{algorithm}

\section{Additional Reasoning Results}
\label{app:additional-reasoning-results}

\subsection{SciCode Problem vs.~Subtask Accuracy}
\label{app:scicode-subtask}

SciCode~\citep{tian2024scicode} decomposes each research-coding problem into several \emph{subtasks}: a solution is counted as correct at the problem level only if \emph{every} subtask is solved, making problem-level accuracy much lower than subtask-level accuracy and highly sensitive to a single failing step.  In the main paper (\cref{tab:reasoning_8b,tab:reasoning_32b}) we report problem-level pass\,@1, which often yields single-digit absolute numbers for models at 8B.  Many publicly reported SciCode numbers instead use subtask-level pass\,@1, which runs roughly three to four times higher in absolute terms.

In \cref{tab:scicode_subtask} we report both metrics side by side for the models in our reasoning evaluation suite.  Relative rankings are identical
under the two metrics.  At 8B: Qwen3-8B $>$ GKA-Primed $>$ Qwen3-8B\,[Reasoner-SFT] $>$ GDN-Primed, on both problem and subtask accuracy. At 32B: Qwen3-32B $>$ Qwen3-32B\,[Reasoner-SFT] $>$ GKA-Primed, again on both metrics.  The qualitative conclusions of the main paper (both on ranking and on post-training regressions relative to the off-the-shelf checkpoint) are therefore unaffected by the choice of metric.

\begin{table}[ht]
\centering
\caption{\textbf{SciCode problem vs.~subtask accuracy.} Problem-level numbers
reproduce \cref{tab:reasoning_8b,tab:reasoning_32b} from the main paper.
Subtask-level numbers are the standard metric reported in the
SciCode~\citep{tian2024scicode} paper and by third parties. Relative rankings
are preserved under both metrics.}
\label{tab:scicode_subtask}
\small
\begin{tabular}{llcc}
\toprule
\textbf{Scale} & \textbf{Model} & \textbf{Problem acc.} & \textbf{Subtask acc.} \\
\midrule
\multirow{4}{*}{8B}
 & Qwen3-8B                        & 10.6 & 27.9 \\
 & Qwen3-8B\,[Reasoner-SFT]        &  4.1 & 20.1 \\
 & GKA-Primed-HQwen3-Reasoner   &  6.4 & 27.3 \\
 & GDN-Primed-HQwen3-Reasoner   &  2.5 & 11.0 \\
\midrule
\multirow{3}{*}{32B}
 & Qwen3-32B                       & 15.9 & 38.9 \\
 & Qwen3-32B\,[Reasoner-SFT]       & 13.1 & 37.1 \\
 & GKA-Primed-HQwen3-Reasoner   & 12.3 & 30.6 \\
\bottomrule
\end{tabular}
\end{table}

\subsection{Choosing Inference Parameters for Test-Time Scaling}
\label{app:inference-params-tts}

In our work we use vLLM~\citep{vllm} to scale inference deployment of our models. vLLM serves requests using continuous batching: the scheduler maintains waiting and running sets of requests and, at each engine step, selects a batch of prefill and/or decode work subject to token and KV-cache constraints. The running set is bounded by the KV-cache budget: when GPU memory reserved for KV caches is fully committed, the engine is \emph{KV-cache saturated} and no new requests can be admitted until a running one completes. If the running set later exceeds what the budget can sustain (e.g., admitted sequences grew longer than anticipated at admission time), the scheduler \emph{preempts} some requests and returns them to the waiting queue; when a preempted request is re-admitted, much of its prefill must be re-run, which results in some throughput overhead.

Our test-time scaling experiment (\cref{subsec:tts}) triggers exactly this preemption regime: AIME~2025 rollouts mix easy and hard problems, so short correct traces finish early while long traces from hard problems accumulate and saturate the KV budget, forcing the scheduler to preempt the remaining running requests. One way to limit this is to cap how many sequences the scheduler may admit concurrently, a lower cap trades early-run throughput on short-trace rollouts for less late-run preemption once long traces pile up. The vLLM argument \texttt{max\_\allowbreak num\_\allowbreak seqs} is exactly this cap on the running queue. The right setting is architecture-dependent: the Hybrid's smaller per-sequence KV footprint supports more concurrent sequences before saturation than the Transformer's, so their optima differ. We sweep it per model and report the fastest configuration.

\paragraph{Protocol.} For each (model, \texttt{max\_\allowbreak num\_\allowbreak seqs}) pair we run the full AIME~2025 set (30 problems) at five rollout counts $n_b \in \{480, 720, 960, 1200, 1440\}$ on 8$\times$H200 (TP$=$8) with identical sampling parameters. From the per-request server logs we take each run's wall-clock as $\text{wc}_b = \max(\text{end}) - \min(\text{start})$, i.e., end-to-end elapsed time from first dispatch to last completion. We summarize each pair with a rollout-weighted mean, $\text{sec/rollout} = \sum_b \text{wc}_b \,/\, \sum_b n_b$, where $b$ indexes the five rollout counts.

\paragraph{Observations.} GKA-Primed-HQwen3-Reasoner's throughput degrades monotonically above \texttt{max\_\allowbreak num\_\allowbreak seqs}${=}256$ (4.6$\to$5.0$\to$5.3 sec/rollout from 256 to 768): once the ceiling exceeds the KV-saturation concurrency, additional admissions trigger preemption, and the re-prefill cost dominates any batching gain. The Qwen3-32B\,[Reasoner-SFT] baseline is comparatively insensitive to this knob over the same range (5.9$\to$5.8 sec/rollout from 256 to 1024). At their respective optima (GKA-Primed: 256; Qwen3-32B\,[Reasoner-SFT]: 512), GKA-Primed-HQwen3-Reasoner serves AIME~2025 at 21\% lower sec/rollout (4.6 vs.\ 5.8) at comparable accuracy. Dropping Chebyshev iterations to $r{=}10$ (\cref{subsec:chebyshev}) at the same \texttt{max\_\allowbreak num\_\allowbreak seqs}${=}256$ reduces sec/rollout further to 4.3 with a $\sim$1\% accuracy cost.

\section{Hybrid State Composition Theoretical Grounding}
\label{subsubsec:theory}
\label{app:state-comp-theory}

For each SSM family, the recurrence structure admits a principled decomposition that justifies combining independently computed states (at least in the single layer setting).  We describe the CASO/PICASO argument for Mamba2 (where it was originally derived), show how it extends to GDN and GKA with different computational tradeoffs, and then show that GKA additionally admits a simpler additive composition in the ungated limit via the information filter.

\paragraph{Mamba2: CASO and PICASO.}
Recall from \cref{subsec:ssm-zoo} that Mamba2 updates its state as
$\mathbf{S}_t = \alpha_t\,\mathbf{S}_{t-1}
  + \mathbf{v}_t\,\mathbf{k}_t^\top$.
For a single-layer SSM, the state obtained by processing the concatenation of $K$ chunks $u_1, \ldots, u_K$ can be decomposed exactly~\citep{picaso}:
\begin{equation}\label{eq:caso}
  \mathbf{S}(u_1 \cdot \ldots \cdot u_K)
  = \mathbf{S}^{(K)}
    + \sum_{c=1}^{K-1}
      \Bigl(\prod_{j=c+1}^{K} \mathbf{A}^{(j)}\Bigr)\,
      \mathbf{S}^{(c)},
\end{equation}
where $\mathbf{S}^{(c)}$ is the state from chunk~$c$ alone and $\mathbf{A}^{(c)} = \prod_{t \in \text{chunk}\,c} \mathrm{diag}(\alpha_t)$ is the accumulated decay matrix for that chunk.  Because Mamba2 uses diagonal transitions, $\mathbf{A}^{(c)}$ is diagonal and the products reduce to element-wise arithmetic---no forward passes are needed.

This is the CASO (Compositional Aggregation of States as Observations) formula.  It is exact for a single layer but order-dependent: the result changes with the permutation of chunks.  PICASO~\citep{picaso} removes this dependence by averaging over permutations.  The cyclic variant (PICASO-R) averages over $K$ cyclic shifts and is computable in $O(K)$ time, and is what we implement as part of state composition for this layer.

\paragraph{GDN: PICASO with dense transition matrices.}
GDN's recurrence (\cref{subsec:ssm-zoo}) is
$\mathbf{S}_t = \alpha_t(\mathbf{I} - \beta_t\,\mathbf{k}_t\mathbf{k}_t^\top)\,
  \mathbf{S}_{t-1} + \beta_t\,\mathbf{v}_t\mathbf{k}_t^\top$. 
This has the same $\mathbf{S}_t = \mathbf{A}_t\,\mathbf{S}_{t-1} + \mathbf{B}_t$ structure as Mamba2, so the CASO decomposition
(Eq.~\ref{eq:caso}) applies directly with PICASO averaging restoring approximate permutation invariance.  The key difference is computational: Mamba2's diagonal $\mathbf{A}_t$ yields a diagonal accumulated product $\mathbf{A}^{(c)}$, so CASO reduces to element-wise operations in $O(d_k)$.  For GDN, each per-step transition $\mathbf{A}_t = \alpha_t(\mathbf{I} - \beta_t\,\mathbf{k}_t\mathbf{k}_t^\top)$ is a rank-1 perturbation of a scaled identity, and their product $\mathbf{A}^{(c)} = \prod_{t \in \text{chunk}\,c} \mathbf{A}_t$ is a product of Householder-like matrices.  In our current implementation we materialize the accumulated product $\mathbf{A}^{(c)}$ as a dense $d_k \times d_k$ matrix; however, the rank-1 structure of each factor could be exploited for a more compact representation, which we leave to future work.

\paragraph{GKA: additive composition via the information filter.}
GKA maintains a gated information matrix $\mathbf{H}_t$ and information state $\mathbf{U}_t$ as its internal states (see \cref{sec:gka_main_ideas} for the full parameterization, including the adaptive regularization $\lambda_t$):
\begin{align}
  \mathbf{H}_t &= \gamma_t\,\mathbf{H}_{t-1}
    + \beta_t \mathbf{k}_t\,\mathbf{k}_t^\top, \\
  \mathbf{U}_t &= \gamma_t\,\mathbf{U}_{t-1}
    + \beta_t  \mathbf{v}_t\,\mathbf{k}_t^\top,
\end{align}
where $\gamma_t \in [0,1]$ is a learned decay gate~\citep{gatedkalmanet}. Note that while the output state is a non-linear combination of these terms,
$\mathbf{S}_t = \mathbf{U}_t\,(\mathbf{H}_t + \lambda_t\mathbf{I})^{-1}$, the state updates themselves are affine.
Since $\mathbf{A}_t = \gamma_t\,\mathbf{I}$ is diagonal (a scalar times identity), the CASO decomposition (Eq.~\ref{eq:caso}) applies with the same element-wise efficiency as Mamba2, and PICASO averaging restores approximate permutation invariance.

However, GKA also admits a simpler composition.  In the ungated limit ($\gamma_t = 1$), both $\mathbf{H}_t$ and $\mathbf{U}_t$ accumulate purely additively, so states from independently processed chunks compose by summation:
\begin{equation}\label{eq:gka-compose}
  \mathbf{H}_{\mathrm{merged}} = \sum_{c=1}^{K} \mathbf{H}^{(c)},
  \qquad
  \mathbf{U}_{\mathrm{merged}} = \sum_{c=1}^{K} \mathbf{U}^{(c)}.
\end{equation}
This is the multi-sensor fusion identity from classical Kalman filtering~\citep{mutambara1998}: if $K$ sensors observe the same latent state with shared dynamics, their information matrices and states can be summed to obtain the fused posterior.  When $\gamma_t < 1$, the additive composition is no longer exact; however, for sufficiently long chunks the gated decay suppresses the contribution of early tokens, making the additive approximation effective in practice.

\section{Symmetric Tiled Triton Kernel: Implementation Details}
\label{app:tiled-kernel-details}

\noindent \textbf{Updating the Information Matrix $\mathbf{H}_t$ (State-update loop).} Let $\mathbf{H}
_t[i,j]$ denote the $(i,j)$-th tile. We iterate over lower-triangular tiles ($i \geq j$), and for each: \begin{enumerate} 
\item Load $\mathbf{H}_{t-1}[i,j]$, $\mathbf{k}_t[i]$, $\mathbf{k}_t[j]$ from HBM. 
\item Compute $\mathbf{H}_t[i,j] \leftarrow \gamma_t \mathbf{H}_{t-1}[i,j] + \beta_t \mathbf{k}_t[i]\mathbf{k}_t[j]^\top$. 
\item Accumulate the tile's contribution to $\|\mathbf{H}_t\|_F^2$, doubling the contribution for off-diagonal tiles ($i > j$) to account for the mirrored upper-triangular half. 
\item Store $\mathbf{H}_t[i,j]$ back to HBM. 
\end{enumerate} A final square root over the accumulator gives $\|\mathbf{H}_t\|_F$, which sets the regularization strength $\lambda_t = \alpha \|\mathbf{H}_t\|_F$ used by the CH. The Information State $\mathbf{U}_t$ is updated with the same tile-streaming pattern, but since it is not symmetric every tile is read, updated, and written back.

\noindent \textbf{CH iterations.} Each CH iteration requires the product $\mathbf{m}_t^{(k)} = \mathbf{H}_t \mathbf{x}_{t}^{(k-1)}$, where $\mathbf{x}_{t}^{(k-1)}$ is the CH solution obtained after $(k-1)$ iterations. This computation can be decomposed over tiles. Let $\mathbf{x}_t^{(k-1)}[i]$ denote the $i$-th tile of $\mathbf{x}_t^{(k-1)}$ and $\mathbf{m}_t^{(k)}[j]$ the $j$-th tile of the output. We loop over all tile pairs $(i,j)$: \begin{itemize} 
\item For $i \geq j$, reload $\mathbf{H}_t[i,j]$ from HBM, for \texttt{tiled\_large\_batch} variant. For \texttt{tiled\_small\_batch} variant, just use previously-computed $\mathbf{H}_t[i,j]$ tile which is already in registers.
\item For $i < j$, recompute $\mathbf{H}_t[i,j] = \mathbf{H}_t[j,i]^\top$ on the fly from its lower-triangular mirror; we do not persist this tile to HBM, as re-transposing in each CH iteration is cheaper on GPUs (like H200s) than an HBM write-back and reload. 
\item Accumulate $\mathbf{m}_t^{(k)}[i] \mathrel{+}= \mathbf{H}_t[i,j] , \mathbf{x}_t^{(k-1)}[j]$. 
\end{itemize} 

Once every tile pair has been processed, $\mathbf{m}_t^{(k)}$ is complete. The remaining per-iteration CH update~\citep[Algorithm 1]{gatedkalmanet} involves only pointwise operations on the full vectors, which do not benefit from tiling, so we join the tiles into contiguous vectors, perform the update, and split the resulting $\mathbf{x}_t^{(k)}$ back into tiles for the next iteration. \\

\end{document}